\documentclass{article}

\usepackage[final,nonatbib]{neurips_2021}

\usepackage[utf8]{inputenc} %
\usepackage[T1]{fontenc}    %
\usepackage[colorlinks=true,citecolor=blue]{hyperref}       %
\usepackage{url}            %
\usepackage{booktabs}       %
\usepackage{amsfonts}       %
\usepackage{nicefrac}       %
\usepackage{microtype}      %
\usepackage{bm}
\usepackage{bbm}
\usepackage{graphicx}
\usepackage{todonotes}
\usepackage{upgreek}
\usepackage{times}
\usepackage{epsfig}
\usepackage{graphicx}
\usepackage{caption}
\usepackage{amssymb}
\usepackage{amsmath}
\usepackage{multirow}
\usepackage{algorithm}
\usepackage[noend]{algpseudocode}
\usepackage{mathtools}

\usepackage{amsthm}

\usepackage{breqn}

\usepackage{soul}
\usepackage[caption=false]{subfig}
\usepackage{thmtools, thm-restate}

\declaretheorem[name=Lemma]{lemma}
\declaretheorem[name=Definition]{definition}
\declaretheorem[name=Corollary]{corollary}

\newcommand\numberthis{\addtocounter{equation}{1}\tag{\theequation}}

\DeclarePairedDelimiterX{\norm}[1]{\lVert}{\rVert}{#1}

\DeclareMathOperator*{\argmax}{\arg\!\max}

\title{Towards Robust Bisimulation Metric Learning}

\author{%
  Mete Kemertas\thanks{Equal contribution.} \\
  Department of Computer Science\\
  University of Toronto\\
  \texttt{kemertas@cs.toronto.edu} \\
  \And
  Tristan Aumentado-Armstrong$^*$ \\
  Department of Computer Science\\
  University of Toronto\\
  \texttt{taumen@cs.toronto.edu} \\
}

\begin{document}

\maketitle

\begin{abstract}
  Learned representations in deep reinforcement learning (DRL) have to extract task-relevant information from complex observations, balancing between robustness to distraction and informativeness to the policy. Such stable and rich representations, often learned via modern function approximation techniques, can enable practical application of the policy improvement theorem, even in high-dimensional continuous state-action spaces. Bisimulation metrics offer one solution to this representation learning problem, by collapsing functionally similar states together in representation space, which promotes invariance to noise and distractors. In this work, we generalize value function approximation bounds for on-policy bisimulation metrics to non-optimal policies and approximate environment dynamics. Our theoretical results help us identify embedding pathologies that may occur in practical use. In particular, we find that these issues stem from an underconstrained dynamics model and an unstable dependence of the embedding norm on the reward signal in environments with sparse rewards. Further, we propose a set of practical remedies: (i) a norm constraint on the representation space, and (ii) an extension of prior approaches with intrinsic rewards and latent space regularization. Finally, we provide evidence that the resulting method is not only more robust to sparse reward functions, but also able to solve challenging continuous control tasks with observational distractions, where prior methods fail.
\end{abstract}

\section{Introduction}
Complex reinforcement learning (RL) problems require the agent to infer a useful representation of the world state from observations.
The utility of this representation can be measured by how readily it can be used to learn and enact a policy that solves a specific task.
As an example, consider a robot using visual perception to pour a cup of coffee.
The clutter on the counter and the colour of the walls have little effect on the correct action;
even more pertinent aspects, such as two mugs with slightly different patterns and shapes,  should be treated nearly the same by the policy.
In other words, many states are  \textit{equivalent} for a given task, with their differences being task-irrelevant distractors.
Thus, a natural approach to generalization across environmental changes is constructing a representation that is invariant to such nuisance variables -- effectively ``grouping'' such equivalent states together.

One method of obtaining such a representation
uses the notion of a \textit{bisimulation metric} (BSM) \cite{Ferns2004Metrics,Ferns2011Bisimulation}.
The goal is to abstract the states into a new metric space, which groups ``behaviourally'' similar states in a task-dependent manner.
In particular, bisimilar states (i.e., in this case, those close under the BSM) should yield similar stochastic reward sequences, given the same action sequence.
In a recursive sense, this requires bisimilar states to have similar
(i) immediate rewards and (ii) transition distributions in the BSM space (e.g., see Eq.\ \ref{eq:ferns-bisim-metric}).
Thus, ideally, the resulting representation space contains the information necessary to help the policy maximize return and little else.

Recently, the Deep Bisimulation for Control (DBC) algorithm \cite{zhang2021invariant} showed how to map observations into a learned space that follows a policy-dependent (or on-policy) BSM (PBSM) for a given task \cite{castro2020scalable}, resulting in a powerful ability to ignore distractors within high-dimensional observation spaces.
Using a learned dynamics model, DBC trains an encoder to abstract states to follow the estimated PBSM, using the aforementioned recursive formulation for metric learning.

However, this method can suffer from issues of robustness under certain circumstances.
Conceptually, differentiating state encodings requires trajectories with different rewards.
In other words, two states that only differ in rewards far in the future should still not be bisimilar.
Yet, in practice, the encoder and policy are learned simultaneously as the agent explores;
hence, in the case of \textit{uninformative} (e.g., sparse or near-constant) rewards, it may incorrectly surmise that two states are bisimilar, as most trajectories will look very similar.
This leads to a premature \textit{collapse} of the embedding space, grouping states that should not be grouped.
On the other hand, we can show that the formulation of the metric learning loss is susceptible to embedding \textit{explosion} if the representation space is left unconstrained\footnote{Embedding collapse and explosion issues also appear in the computer vision literature \cite{hermans2017defense, schroff2015facenet, wu2017sampling}.}.
In our work, we build upon the DBC model in an attempt to tackle both problems:
(1) we address embedding explosion by stabilizing the state representation space via a norm constraint
and
(2) we prevent embedding collapse by altering the encoder training method.

\textbf{Contributions}\hspace{0.25em}
We
(i) generalize theoretical value function approximation bounds to a broader family of BSMs, with learned dynamics; (ii) analyze the BSM loss formulation and identify sources of potential embedding pathologies (explosion and collapse);
(iii) show how to constrain representation space to obtain better 
 embedding quality guarantees;
(iv) devise theoretically-motivated training modifications based on intrinsic motivation and inverse dynamics regularization;
and
(v) show on a variety of tasks that our method is more robust to both sparsity and distractors in the environment.

\section{Related Work}
Our work builds on two major aspects of DRL:
representation learning, 
    particularly task-specific encodings following the BSM,
and 
methods for sparse reward environments, 
    (e.g., intrinsic motivation).

\textbf{Representation Learning for Control}\hspace{0.25em}
In the context of RL, there is a spectrum of representations 
	from task-agnostic to task-specialized.
One task-agnostic approach to learning latent representation spaces 
	uses a reconstruction loss.
In particular, some methods utilize this approach for model-based planning 
	\cite{kurutach2018learning,zhang2019solar,hafner2019learning}, 
	while others do so for model-free RL 
	\cite{nair2018visual,gelada2019deepmdp,lee2019stochastic}.  
Such approaches often utilize generative models, 
	which have shown use for learning via simulation, 
	in addition to representation learning \cite{ha2018world,watter2015embed}.
Instead of fully reconstructing an input, other approaches use self-supervised representation learning methods, especially contrastive learning algorithms 
	\cite{okada2020dreaming,laskin2020curl,stooke2020decoupling}.

While task-agnostic reconstruction is reusable and transferable, it can store unnecessary or distracting details. In contrast, \textit{state abstraction methods}  \cite{Li06towardsa} can obtain compact state representations that ignore irrelevant information, by finding (approximately) equivalent aggregated MDPs that have (nearly) identical optimal policies. To learn such MDP homomorphisms \cite{ravindran2004thesis, ravindran2004approximate}, previous efforts have exploited various structural regularities (e.g., equal rewards and transition probabilities \cite{givan2003equivalence, taylor2008bounding}, state-action symmetries and equivariances \cite{taylor2008bounding, van2020plannable, vanderpol2020symmetries}) and assumptions \cite{biza2019online}. 
\textit{State similarity metric learning} \cite{le2021metrics} is closely related, as one can aggregate $\epsilon$-close states to produce a new MDP. 
For example, inter-state distances can be assigned based on the difference between policy distributions conditioned on respective states \cite{agarwal2021pse, ghosh2018learning}. 
Inverse dynamics prediction can also help encourage holding only information relevant to what the agent can control \cite{agrawal2016learning,pathak2017curiosity}. 
For model-based RL, prior work constructed equivalence classes of \textit{models} \cite{grimm2020value}, as well as other abstractions \cite{corneil2018efficient,franccois2019combined}, to save computation and memory.

In our work, we build on the use of task-specific \textit{bisimulation relations} \cite{givan2003equivalence},
which define equivalence classes of states based on rewards and transition probabilities.
However, obtaining equivalence classes is brittle, as our estimates of these quantities may be inexact.
Hence, Ferns et al. \cite{Ferns2004Metrics,Ferns2011Bisimulation, ferns2014bisimulation} instead turn to a bisimulation \textit{metric} between states, 
    which can vary smoothly as the rewards and transition probabilities change.
More recently, Zhang et al.\ \cite{zhang2021invariant} devised Deep Bisimulation for Control (DBC), which applies a metric learning approach to enforce its representation to approximately follow bisimulation-derived state aggregation.
While DeepMDP \cite{gelada2019deepmdp} proves that its representation provides an upper bound on the bisimulation distance, DBC \cite{zhang2021invariant} directly imposes the bisimulation metric structure on the latent space.
Herein, 
    we extend prior analysis and methods for representation learning 
    using PBSMs,
    with the goal of improving robustness and general performance.
In particular, 
    we focus on preventing embedding pathologies,
    in the context of uninformative reward signals and distraction.
    
\textbf{Reinforcement Learning under Uninformative Rewards}\hspace{0.25em}
Many real-world tasks suffer from uninformative rewards, 
	meaning that the agent receives sparse (mostly zero) or static (largely constant) rewards throughout its trajectories. 
For example, games like Montezuma's Revenge only provide reward signals after a large number of steps -- and the agent is likely not to receive one at all, in most cases.
Thus, while common, this situation significantly increases the difficulty of the RL task, as the agent receives little signal about its progress.
To deal with such situations, a common tactic is to encourage the agent to explore in the absence of an extrinsic reward signal. 
This is often done by providing \textit{intrinsic motivation} via a self-derived reward, resulting in ``curiosity-driven'' behaviour \cite{burda2018large,aubret2019survey}. 
Such approaches include
surprise \cite{achiam2017surprise,bellemare2016unifying} 
	(where experiencing unexpected dynamics is rewarded),
and 
empowerment \cite{salge2014empowerment,gregor2016variational} 
	(where the agent prefers states in which it has more control).

In this work, we specifically target the case of uninformative rewards, 
	showing that the DBC model is especially susceptible to instability or collapse for such tasks, and consider ways to ameliorate this issue.
In particular, we utilize the forward prediction error as an intrinsic reward \cite{pathak2017curiosity,burda2018large,stadie2015incentivizing}, thus augmenting the sparse extrinsic rewards and encouraging exploration.
We also regularize the latent space by learning an inverse dynamics model on top of it, which does not rely on the extrinsic reward signal and has previously been used for distraction-robust representation learning in DRL \cite{agrawal2016learning,pathak2017curiosity}.

\section{Technical Approach}
In this section, we first describe our notation and problem setting, and review prior definitions of bisimulation metrics and relevant metric learning objectives. Then, we consider theoretical connections to value functions and convergence properties, and show that these formulations may be susceptible to (i) instabilities in optimization due to embedding explosion (i.e., large norms), and (ii) convergence to trivial solutions where all states are represented as a single point (embedding collapse). Based on our analysis, we propose an extension of deep bisimulation metric learning, 
with theoretically-motivated constraints on the optimization objective, including alterations to the forward dynamics model, intrinsic motivation, and inverse dynamics regularization.
\subsection{Preliminaries}
\label{sec:prelim} 
In this work, we consider a discounted Markov Decision Process (MDP) given by a tuple, $\langle \mathcal{S}, \mathcal{A}, \mathcal{P}, R, \rho_0 \rangle$. At the beginning of each episode, an initial state, $\mathbf{s}_0 \in \mathcal{S}$, is sampled from the initial-state distribution $\rho_0$ over the state space $\mathcal{S}$. Then, at each discrete time-step $t\geq0$, an agent takes an action, $\mathbf{a}_t \in \mathcal{A}$, according to a policy $\pi(\mathbf{a}_t|\mathbf{s}_t)$. As a result, the MDP transitions to the next state according to a transition distribution $\mathcal{P}(\mathbf{s}_{t+1}|\mathbf{s}_{t}, \mathbf{a}_{t})$. The agent collects a scalar reward, $r_t=R(\mathbf{s}_t, \mathbf{a}_t)$, from the environment according to a bounded reward function, $R: \mathcal{S} \times \mathcal{A} \rightarrow [R_{\mathrm{min}}, R_{\mathrm{max}}]$. In infinite- and long-horizon settings, a discount factor, $\gamma \in [0, 1)$, is used to calculate the agent's discounted return in a given episode, $G = \sum_{t \geq 0}\gamma^tr_t$. RL algorithms aim to find an optimal policy, $\pi^* \coloneqq \argmax_{\pi \in \Pi}\mathbb{E}[G]$, for a class of stationary policies $\Pi$. In high-dimensional, continuous state (or observation) spaces, this learning problem is rendered tractable via a state encoder (e.g., a neural network), $\phi: \mathcal{S} \rightarrow \mathbb{R}^n$, 
which is used to
learn a policy of the form $\pi(\mathbf{a}|\phi(\mathbf{s}))$. 

The following (pseudo\footnote{Pseudo-metrics are a generalization of metrics that allow distinct points to have zero distance. For simplicity, we broadly use the term "metric" in the remainder of this paper at the cost of imprecision.})-metric, based on the Wasserstein metric (see Appendix \ref{sec:wasserstein} for a review), is of particular relevance to this work. Distances are assigned to state pairs, $(\mathbf{s}_i, \mathbf{s}_j) \in \mathcal{S} \times \mathcal{S}$, according to a pessimistic measure \cite{castro2010using} of how much the rewards collected in each state and the respective transition distributions differ. A distance of zero for a pair implies state aggregation, or \textit{bisimilarity}.
\begin{definition}[Bisimulation metric for continuous MDPs, Thm. 3.12 of \cite{Ferns2011Bisimulation}]
\label{def:bisim2}
The following metric exists and is unique, given $R: \mathcal{S} \times \mathcal{A} \rightarrow [0, 1]$ and $c \in (0, 1)$ for continuous MDPs:
\begin{align}
\label{eq:ferns-bisim-metric}
    d(\mathbf{s}_i, \mathbf{s}_j) = \max_{\mathbf{a} \in \mathcal{A}} (1-c)|R(\mathbf{s}_i, \mathbf{a}) - R(\mathbf{s}_j, \mathbf{a})| ~+~ c W_1(d)(\mathcal{P}(\cdot|\mathbf{s}_i, \mathbf{a}), \mathcal{P}(\cdot|\mathbf{s}_j, \mathbf{a})).
\end{align}
\end{definition}
An earlier version of this metric for finite MDPs used separate weighting constants $c_R, c_T \geq 0$ for the first and second terms respectively, and required that $c_R + c_T \leq 1$ \cite{Ferns2004Metrics}. Here, when the weighting constant $c_T$ of the $W_1(d)$ term is in $[0, 1)$, the RHS is a contraction mapping, $\mathcal{F}(d): \mathfrak{met} \rightarrow \mathfrak{met}$, in the space of metrics. Then, the Banach fixed-point theorem can be applied to ensure the existence of a unique metric, which also ensures convergence via fixed-point iteration for finite MDPs. For more details, we refer the reader to \cite{Ferns2004Metrics, Ferns2011Bisimulation} and the proof of Remark \ref{remark:rem1} in Appendix \ref{sec:proofs}. 
Notice that $c$ (or $c_T$) determines a timescale for the bisimulation metric, weighting the importance of current versus future rewards, analogously to the discount factor $\gamma$. 
More recently, an \textit{on-policy} bisimulation metric (also called $\pi$-bisimulation) was proposed to circumvent the intractibility introduced by taking the $\max$ operation over high-dimensional action spaces (e.g., continuous control), as well as the inherent pessimism of the policy-independent form \cite{castro2020scalable}.
\begin{definition}[On-policy bisimulation metric \cite{castro2020scalable}]
Given a fixed policy $\pi$, the following on-policy bisimulation metric exists and is unique:
\label{def:on-policy}
\begin{align} 
\label{eq:on-policy-bisim-metric}
d_{\pi}(\mathbf{s}_i, \mathbf{s}_j) \coloneqq |r^{\pi}_i - r^{\pi}_j| ~+~ \gamma W_1(d_\pi)(\mathcal{P}^\pi(\cdot| \mathbf{s}_i), \mathcal{P}^\pi(\cdot| \mathbf{s}_j)),
\end{align}
where $r_i^\pi \coloneqq \mathbb{E}_{\mathbf{a} \sim \pi}[R(\mathbf{s}_i, \mathbf{a})]$ and $\mathcal{P}^\pi(\cdot|\mathbf{s}_i) \coloneqq \mathbb{E}_{\mathbf{a} \sim \pi}[\mathcal{P}(\cdot|\mathbf{s}_i, \mathbf{a})]$.
\end{definition}
Zhang et al. \cite{zhang2021invariant} proposed to learn a similar on-policy bisimulation metric directly in the embedding space via an MSE objective. They proposed an algorithm for jointly learning a policy $\pi(\mathbf{a}|\phi(\mathbf{s}))$ with an on-policy bisimulation metric. Below, we define a generalized variant of their objective:
\begin{align}
J(\phi) &\coloneqq \frac{1}{2}\mathbb{E}\left[\left(\widehat{d}_{\pi, \phi}(\mathbf{s}_i, \mathbf{s}_j) - |r^{\pi}_i - r^{\pi}_j| ~-~ \gamma
W_2(\norm{\cdot}_{q_1})(\widehat{\mathcal{P}}^\pi(\cdot| \phi(\mathbf{s}_i)), \widehat{\mathcal{P}}^\pi(\cdot| \phi(\mathbf{s}_j)))\right)^2\right], \label{eq:general-bisim-loss}
\end{align}
where $\widehat{d}_{\pi, \phi}(\mathbf{s}_i, \mathbf{s}_j) \coloneqq \norm{\phi(\mathbf{s}_i)-\phi(\mathbf{s}_j)}_{q_2}$ and they used $(q_1=2, q_2=1)$. Notice that the recursion induced by this objective is different from prior metrics in three ways; (i) a $W_2$ metric was used instead of a $W_1$ metric since $W_2$ has a convenient closed form for Gaussian distributions when $q_1=2$, (ii) the distance function used for the bisimulation metric and the Wasserstein metric are different ($L_1$ and $L_2$ respectively), and (iii) a forward dynamics model is used instead of the ground truth dynamics. While they may introduce practical benefits, these differences violate the conditions under which the existence of a unique bisimulation metric has been proven \cite{castro2020scalable, Ferns2004Metrics, Ferns2011Bisimulation, ferns2014bisimulation}. Thus, we will (i) assume a unique metric exists for all Wasserstein metrics $W_p$ when necessary (see Assumption \ref{assum:p-wass}), (ii) study losses that use a matching metric, i.e., $q_1=q_2=q$, %
and (iii) introduce a constraint on forward models in Sec. \ref{sec:lipschitz-forward}. 

Next, we note that they recommend the use of stop gradients for the $W_2$ term. The resulting gradient updates may be considered as approximate fixed-point iteration in the space of metrics:
\begin{align}
    \widehat{\mathcal{F}}(\widehat{d}_{\pi}(\phi_n, \widehat{\mathcal{P}}))(\mathbf{s}_i, \mathbf{s}_j)  \coloneqq 
    \widehat{d}_{\pi}(\mathbf{s}_i, \mathbf{s}_j; \phi_n + \alpha_n \nabla_{\phi}J(\phi_n), \widehat{\mathcal{P}}),
\end{align}
where $\alpha_n$ is a learning rate. However, in practice, $\pi$ and $\widehat{\mathcal{P}}$ may also be  updated in training, which is of particular relevance when they have form $\pi(\mathbf{a}|\phi(\mathbf{s}))$ and $\widehat{\mathcal{P}}(\phi(\mathbf{s}^\prime)|\phi(\mathbf{s}), \mathbf{a})$ as in \cite{zhang2021invariant}. In the next section, we will discuss conditions under which joint updates to a policy $\pi$ and a metric defined by state encodings $\phi$ may or may not converge in practical settings:
\begin{align}
    \lim_{n \rightarrow \infty}\mathcal{F}^{(n)}(\pi, \widehat{d}_{\pi, \phi}) \stackrel{?}{=} d_{\pi^*},
\end{align}
where $d_{\pi^*}$ is the on-policy bisimulation metric for the optimal policy $\pi^*$.

\subsection{Theoretical Analysis}
\label{sec:theoretical}
In this section, we generalize prior value function bounds for bisimulation metrics to cases where arbitrary weighting constants $c_R, c_T$ are used, and further derive a value function approximation (VFA) bound for $V^\pi$ rather than $V^*$ unlike prior work \cite{Ferns2004Metrics, Ferns2011Bisimulation, zhang2021invariant}. We then describe constraints on forward dynamics models for convergence to a unique metric in joint training, and derive VFA bounds as a function of forward dynamics modelling error. Then, we discuss potential pitfalls in on-policy bisimulation metric learning, which lead us to connecting its weaknesses with sparse rewards. Our findings motivate our proposed remedies to the issues we identify in on-policy bisimulation; we recommend a particular norm constraint on the latent space, and motivate the use of inverse dynamics and intrinsic motivation outlined in Sec. \ref{sec:intrinsic}. All proofs  are provided in Appendix \ref{sec:proofs}.
\subsubsection{Value Function Bounds}
\label{sec:vfa}
An important feature of bisimulation metrics is their relation to value functions since a provably tight connection implies guarantees in VFA. Similarly to previous bounds for policy-independent bisimulation metrics \cite{Ferns2004Metrics, Ferns2011Bisimulation}, Castro \cite{castro2020scalable} showed that given any two states $(\mathbf{s}_i, \mathbf{s}_j)$, the following bound holds for on-policy bisimulation metrics (see Definition \ref{def:on-policy}): $|V^\pi(\mathbf{s}_i) - V^\pi(\mathbf{s}_j)| \leq d_\pi(\mathbf{s}_i, \mathbf{s}_j)$. As discussed earlier, differently from previous approaches, Zhang et al. \cite{zhang2021invariant} used a 2-Wasserstein metric due to practical reasons. 
Here, in order to generalize previous value function bounds, we assume the existence of a unique bisimulation metric for $p$-Wasserstein metrics. %

\begin{restatable}[A1, $p$-Wasserstein bisimulation metric]{assumption}{existence}
\label{assum:p-wass}
For a given $c_T \in [0, 1)$, $c_R \in [0, \infty)$ and $p \geq 1$, the following bisimulation metric exists and is unique:
\begin{align}
\label{eq:p-Wass-bisim-metric}
    d_{\pi}(\mathbf{s}_i, \mathbf{s}_j) &\coloneqq c_R|r^{\pi}_i - r^{\pi}_j| ~+~ c_T
W_{p}(d_{\pi})(\mathcal{P}^\pi(\cdot| \mathbf{s}_i), \mathcal{P}^\pi(\cdot| \mathbf{s}_j)).
\end{align}
\end{restatable}
\begin{restatable}{remark}{existenceremark}
\label{remark:rem1}
If $p=1$, or both the environment and policy are deterministic, \hyperref[assum:p-wass]{A1} holds.
\end{restatable}
\begin{restatable}[Generalized value difference bound]{theorem}{generalizedvdbound}
\label{thm:generalized-ct}
Let the reward function be bounded as $R \in [0, 1]$. For an on-policy bisimulation metric given by Eq. \eqref{eq:p-Wass-bisim-metric}, for any $c_T \in [0, 1)$ and $p \geq 1$, define $\overline{\gamma} = \min(c_T, \gamma)$. Given \hyperref[assum:p-wass]{A1}, the bisimulation distance between a pair of states upper-bounds the difference in their values:
\begin{align}
    c_R|V^\pi(\mathbf{s}_i) - V^\pi(\mathbf{s}_j)| ~\leq~ d_\pi(\mathbf{s}_i, \mathbf{s}_j) +   \frac{2c_R(\gamma - \overline{\gamma})}{(1-\gamma)(1-c_T)}, ~\forall(\mathbf{s}_i, \mathbf{s}_j) \in \mathcal{S} \times \mathcal{S}.
\end{align}
\end{restatable}
Note that Thm. 3 of \cite{castro2020scalable} is a special case with $c_R=p=1$ and $c_T=\gamma$. 
Suppose $c_T \geq \gamma$; we observe that for a degenerate metric $d_\pi=0$, the corresponding value function is a constant $V^\pi(\mathbf{s})= c$. We will discuss the dangers posed by this relation in Section \ref{sec:dangers}.
\begin{restatable}[Generalized VFA bound]{theorem}{generalizedvfabound}
\label{thm:generalized-ct-vfa}
Let rewards be bounded as $R \in [0, 1]$ and $\Phi: \mathcal{S} \rightarrow \widetilde{\mathcal{S}}$ be a function mapping states to a finite partitioning $\widetilde{\mathcal{S}}$ such that $\Phi(\mathbf{s}_i) = \Phi(\mathbf{s}_j) \Rightarrow d_\pi(\mathbf{s}_i, \mathbf{s}_j) \leq 2\epsilon$, which produces an aggregated MDP $\langle \widetilde{\mathcal{S}}, \mathcal{A}, \widetilde{\mathcal{P}}, \widetilde{R}, \widetilde{\rho_0} \rangle$. For any $c_T \in [0, 1)$, let $\overline{\gamma} = \min(c_T, \gamma)$. Given \hyperref[assum:p-wass]{A1},
\begin{align}
    |V^\pi(\mathbf{s}) - \widetilde{V}^\pi(\Phi(\mathbf{s}))| ~\leq~ \frac{2\epsilon}{c_R(1-\overline{\gamma})} + \frac{2(\gamma - \overline{\gamma})}{(1-\gamma)(1-c_T)}, ~\forall\mathbf{s} \in \mathcal{S}.
\end{align}
\end{restatable}
This result generalizes previous performance bounds for $V^*$ to non-optimal policies. Previous bounds also assumed $c_T \geq \gamma$, while we generalize to arbitrary $c_T \in [0, 1)$. Here, the second term of the upper bound characterizes the penalty paid for ``myopic" bisimulation (i.e., $c_T < \gamma$) in VFA error guarantees.
We further show in the next section that $c_T < \gamma$ may be desirable when approximate forward models are used (see Thm. \ref{theorem:valboundmodelerror}). Indeed,  in Appendix C, we not only connect large $c_T$ to high variance in embedding norms, but also find surprisingly strong results when empirically investigating the use of $c_T < \gamma$. We speculate that with sufficient modelling capacity and a critic $V(\phi)$ being trained based on $\gamma$, the metric space may hold information about the environment regarding both timescales $\gamma$ and $c_T$.

\subsubsection{Bisimulation Metrics with Approximate Forward Dynamics}
\label{sec:lipschitz-forward}
\begin{figure}
\centering
\begin{minipage}{.42\textwidth}
  \centering
    \includegraphics[width=1.\linewidth]{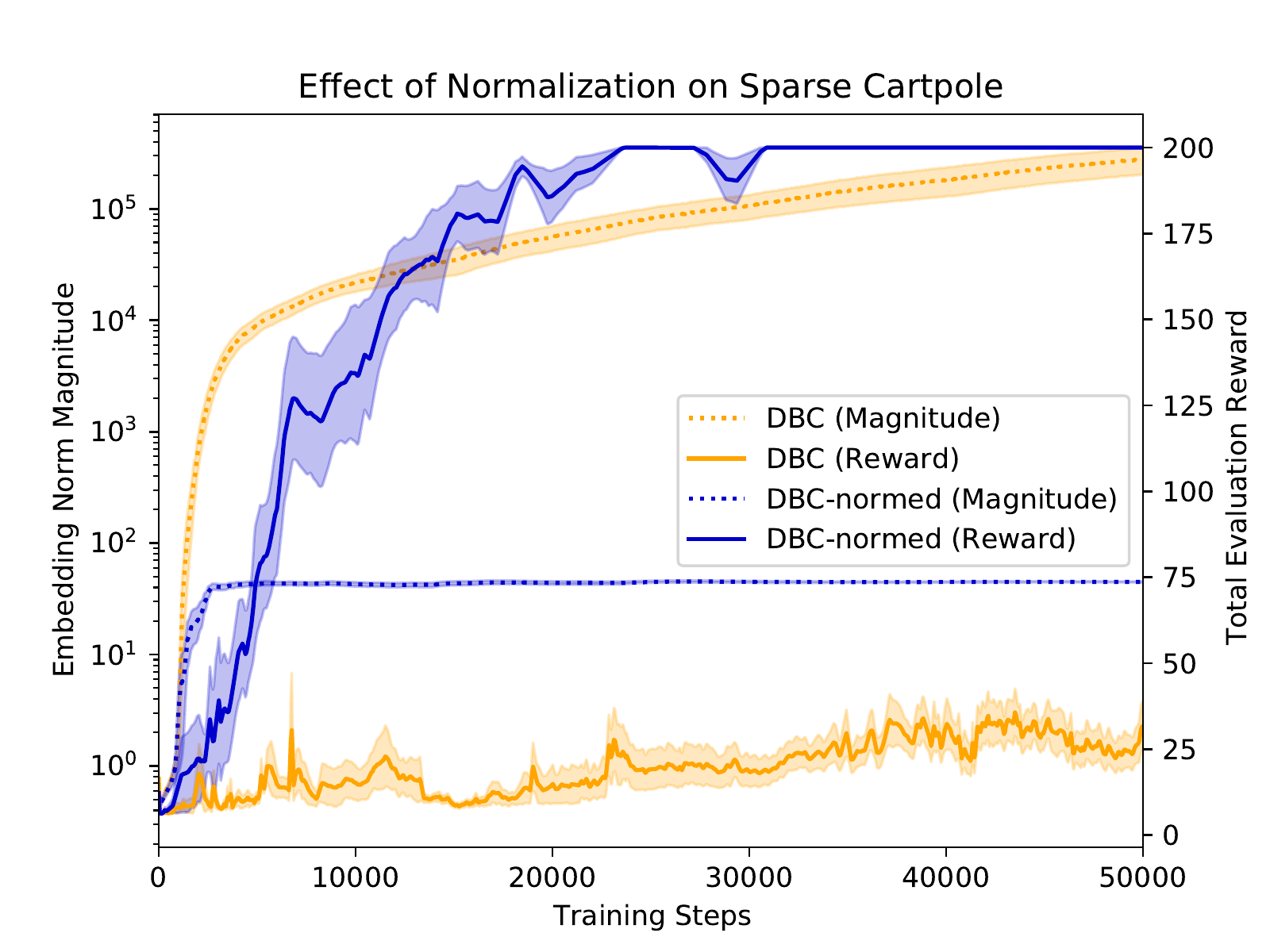}
\end{minipage}
\vspace{-10pt} %
\begin{minipage}{.42
\textwidth}
  \centering
  \includegraphics[width=1.\linewidth]{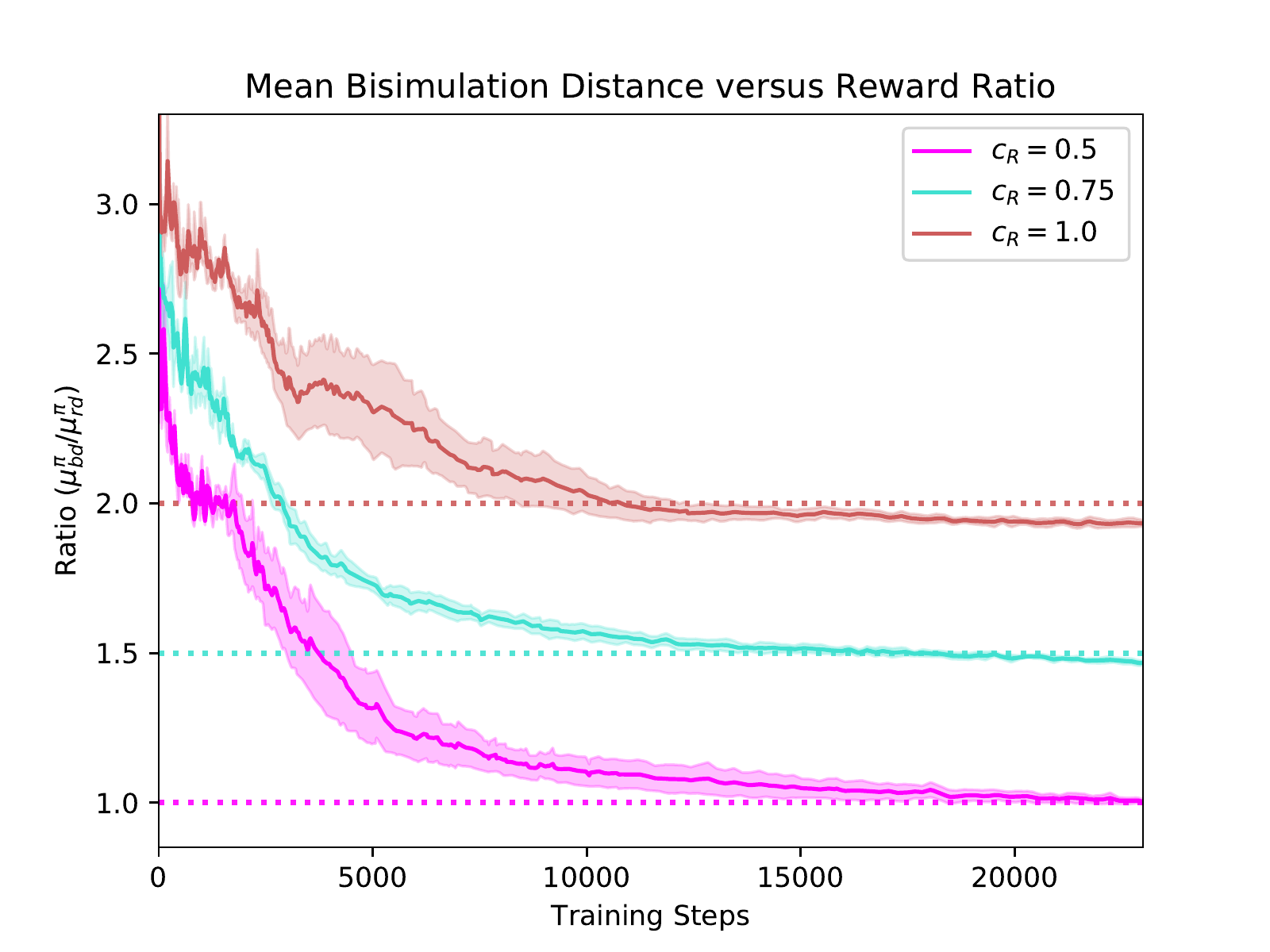}
\end{minipage}
\caption{Theoretical properties under the SparseCartpole task. (\textbf{Left}) Embedding explosion occurs for DBC \cite{zhang2021invariant} for the sparse 2D-Cartpole task, while our constraints discussed in Sec. \ref{sec:lipschitz-forward} keep the model stable. Our approach achieves maximum returns at evaluation, while DBC diverges. (\textbf{Right}) Given $c_T=0.5$, Eq. \eqref{eq:reward-dependence2} correctly predicts the ratio $ \mu_{bd}^\pi / \mu_{rd}^\pi $ between bisimulation distance and difference of rewards at convergence. Dashed lines indicate analytically calculated targets, while solid lines correspond to mini-batch estimates of the ratio during training.}
\label{fig:theoretical}
\end{figure}

We next examine theoretical properties of the on-policy bisimulation distance, which turn out to constrain properties of the approximate transition model. This motivates an additional architectural constraint, resulting in empirical improvements over prior work.

Metrics prior to \cite{zhang2021invariant} were defined based on ground-truth dynamics, which enabled the application of the Banach fixed-point theorem to show the existence of a unique metric. However, the Banach fixed-point theorem requires that a contraction mapping is applied on complete metric spaces. In this case, the metric space on which the Banach fixed-point theorem is applied is itself a space of metrics $\mathfrak{met}$ over states. For that reason, when bisimulation metrics were generalized from finite to continuous MDPs, a compact\footnote{A metric space is compact if and only if it is totally bounded and complete \cite{Ferns2011Bisimulation}.} state space was assumed to ensure completeness of $\mathfrak{met}$ \cite{Ferns2011Bisimulation}. In practice, when using an approximate forward dynamics model $\widehat{\mathcal{P}}: \mathcal{S} \times \mathcal{A} \rightarrow \mathcal{M}(\mathcal{S}^\prime)$\footnote{$\mathcal{M}(\mathcal{X})$ denotes the space of all probability distributions over $\mathcal{X}$.}, if compactness is not guaranteed, the same convergence guarantees do not apply. Thus, ideally, $\mathcal{S}^\prime$ is a compact subset of $\mathcal{S}$. Here, we formalize these restrictions on latent representations and the forward dynamics model to guarantee convergence to a unique metric. The following states a constraint on the BSM solution space given approximate dynamics.
\begin{restatable}[Diameter of $\mathcal{S}$ is bounded]{lemma}{maxdiam}
\label{lemma:diam}
Let $d: \mathcal{S} \times \mathcal{S} \rightarrow [0, \infty)$ be any bisimulation metric:
\begin{align}
\mathrm{diam}(\mathcal{S}; d) \coloneqq \sup_{\mathbf{s}_i, \mathbf{s}_j \in \mathcal{S} \times \mathcal{S}}d(\mathbf{s}_i, \mathbf{s}_j) \leq \frac{c_R}{1 - c_T}(R_{\mathrm{max}} - R_{\mathrm{min}}).
\end{align}
\end{restatable}
The above lemma holds for all bisimulation metrics (on-policy or otherwise) that are based on exact dynamics, if said metrics exist.
Let us also consider an on-policy BSM with imperfect dynamics, such that it may not necessarily satisfy Lemma 1. 
\begin{restatable}[On-policy bisimulation metric with approximate dynamics]{definition}{approximatemetric}
\label{dfn:approximate-bisim}
Given an approximate dynamics model $\widehat{\mathcal{P}}: \mathcal{S} \times \mathcal{A} \rightarrow \mathcal{M}(\mathcal{S}^\prime)$, $c_T \in [0, 1)$, and $c_R \in [0, \infty)$:
\begin{align}
\label{eq:approx-bisim-metric}
    \widehat{d}_\pi(\mathbf{s}_i, \mathbf{s}_j) &\coloneqq c_R|r^{\pi}_i - r^{\pi}_j| ~+~ c_T
W_{1}(\widehat{d}_\pi)(\widehat{\mathcal{P}}^\pi(\cdot| \mathbf{s}_i), \widehat{\mathcal{P}}^\pi(\cdot| \mathbf{s}_j)).
\end{align}
\end{restatable}

Next, we provide a sufficient condition for the existence of a unique metric $\widehat{d}_\pi$ based on approximate dynamics, which also satisfies the upper-bound of Lemma \ref{lemma:diam}. 
Meeting this condition 
shrinks the solution space of metrics being searched to a set known to contain $d_\pi$.

\begin{restatable}[Boundedness condition for convergence]{theorem}{boundedness}
\label{thm:boundedness}
Assume $\mathcal{S}$ is compact. If the support of an approximate dynamics model $\widehat{\mathcal{P}}$, i.e., $\mathcal{S}^\prime = \mathrm{supp}(\widehat{\mathcal{P}})$, is a closed subset of $\mathcal{S}$, then there exists a unique on-policy bisimulation metric $\widehat{d}_\pi$ of the form Eq. \eqref{eq:approx-bisim-metric}, and this metric is bounded:
\begin{align}
\label{eq:boundedness}
    \mathrm{supp}(\widehat{\mathcal{P}}) \subseteq \mathcal{S} \Rightarrow \mathrm{diam}(\mathcal{S}; \widehat{d}_\pi) \leq \frac{c_R}{1 - c_T}(R_{\mathrm{max}} - R_{\mathrm{min}}).
\end{align}
\end{restatable}
If this condition is not satisfied at any point during training, the system may diverge. Indeed, we confirm empirically in Sec. \ref{sec:experiments} that the absence of a norm constraint on the forward dynamics model may result in embedding explosion with practical consequences 
(see Fig. \ref{fig:theoretical}). 
Such explosions are due to compactness violations of the approximate dynamics (e.g., if the predictions always increase the embedding space diameter, the recursive nature of the metric can lead to runaway expansion).
Luckily, the condition can be satisfied with ease, e.g., by projecting larger vectors onto the surface of a closed ball, $\mathbb{B}_c$, with diameter given in Lemma \ref{lemma:diam}, such that the following are true:
\begin{align}
    \label{eq:constraint}
    \phi(\mathbf{s}) \in \mathbb{B}_c = 
    \{\mathbf{x} \in \mathbb{R}^n ~&|~ \norm{\mathbf{x}}_q \leq \frac{c_R(R_{\mathrm{max}} - R_{\mathrm{min}})}{2(1 - c_T)} \}, ~\forall{\mathbf{s} \in \mathcal{S}}\\
    \widehat{\mathcal{P}}(\cdot|\phi(\mathbf{s}), \mathbf{a}) &\in \mathcal{M}(\mathbb{B}_c), ~\forall{(\mathbf{s}, \mathbf{a}) \in \mathcal{S} \times \mathcal{A}}.
\end{align}
We find that the constraint applied here is mild yet effective, as evidenced by significantly improved performance and stability when the constraints are active (see Fig.\ \ref{fig:theoretical}).

Clearly, a necessary condition for $\widehat{d}_\pi \rightarrow d_\pi$ is an error-free dynamics model. A natural question that arises from this view concerns the degree to which modelling errors affect VFA bounds. 
\begin{restatable}[VFA bound in terms of model error]{theorem}{valuemodelerror}
\label{theorem:valboundmodelerror}
Consider the same conditions as in Theorem \ref{thm:generalized-ct-vfa}, except that $c_T\in [\gamma, 1)$, $p=1$, and $ \Phi(\mathbf{s}_i) = \Phi(\mathbf{s}_j) \Rightarrow \widehat{d}_{\pi, \phi}(\mathbf{s}_i, \mathbf{s}_j) = \norm{\phi(\mathbf{s}_i)-\phi(\mathbf{s}_j)}_q \leq 2\,\widehat{\epsilon}$. 
Then: 
\begin{equation}
| V^\pi(\mathbf{s}) - \widetilde{V}^\pi( \Phi( \mathbf{s} ) ) |
\leq 
\frac{1}{c_R(1 - \gamma)} \left( 
    2 \,\widehat{\epsilon} + 
    \mathcal{E}_\phi + 
    \frac{2c_R}{1 - c_T}\mathcal{E}_r + 
    \frac{2c_T}{1 - c_T}\mathcal{E}_{\mathcal{P}} \right), \forall \mathbf{s} \in \mathcal{S}.
\end{equation}
where 
$ \mathcal{E}_\phi := \norm{\widehat{d}_{\pi, \phi} - \widehat{d}_\pi}_\infty $ 
    is the metric learning error,
$ \mathcal{E}_r := \norm{\widehat{r}^\pi - {r}^\pi}_\infty $ 
    is the reward approximation error,
and 
    $ \mathcal{E}_{\mathcal{P}} := \sup_{\mathbf{s} \in \mathcal{S}} 
        W_1(d_\pi)( {\mathcal{P}}^\pi(\cdot|\mathbf{s}), 
             \widehat{\mathcal{P}}^\pi(\cdot|\mathbf{s}) 
            ) $
    is the state transition model error. 
\end{restatable}
See Appendix \ref{sec:vfa-model-error} for details, including a generalized version for $c_T \in [0, 1)$, listed as Corollary \ref{crl:vfamodelerrorct}.
Consider an ideal case where ground-truth rewards $r^\pi$ are available and the metric is learned perfectly (i.e., $\mathcal{E}_\phi = \mathcal{E}_r = 0$).
Then, we observe that the  choice of $c_T$ defines a trade-off between VFA error due to (a) forward model error $\mathcal{E}_{\mathcal{P}}$, and (b) ``myopic'' bisimulation (see Thm. \ref{thm:generalized-ct-vfa}, $c_T < \gamma$),
where the BSM timescale is shorter than that of the discounted return.

\subsubsection{On the Dangers of On-policy Bisimulation}
\label{sec:dangers}
Suppose that at training time, $|r^{\pi}_i - r^{\pi}_j| = 0$ for all pairs of states, e.g., during early training in a sparse reward setting. Then, unlike the policy-independent bisimulation metric, the on-policy formulation has a degenerate solution at $\mathrm{diam}(\mathcal{S}; d_\pi)=0$, regardless of the structure of the underlying MDP.
\begin{restatable}[A reason for caution in on-policy bisimulation]{lemma}{onpolicydiam}
\label{lemma:dist-range}
On-policy bisimulation metrics of the form Eq. \eqref{eq:p-Wass-bisim-metric} have an upper bound determined by their policy:
\begin{align}
    \mathrm{diam}(\mathcal{S}; d_\pi) \leq \frac{c_R}{1 - c_T}\sup_{i, j}|r_i^\pi - r_j^\pi|.
\end{align}
\end{restatable}

Due to policy-dependence, the target metric $d_\pi$ changes with each policy update during joint training. As a result of this difficulty, convergence to a unique fixed point (i.e., a unique metric) for a learning algorithm was previously guaranteed with a strong assumption:  ``a policy that is continuously improving'', as in Thm. 1 of \cite{zhang2021invariant}. Informally, if the policy is assumed to be continuously improving, in the worst case, the metric will have a fixed point $d_{\pi^*}$ after the policy itself reaches a fixed point, namely, the optimal policy $\pi^*$. %
However, this assumption may be too strong to guarantee convergence especially for continuous MDPs, as the policy learning process depends heavily on the encoder.

Specifically, to prove their Thm. 1, Zhang et al. \cite{zhang2021invariant} relied on the policy improvement theorem. For continuous MDPs, in practice, a policy $\pi(\mathbf{a}|\phi(\mathbf{s}))$ is learned via non-convex optimization (e.g., policy gradients, VFA), rather than the vanilla policy improvement algorithm. Thus, if the bisimulation metric is degenerate, a continuously improving policy cannot be guaranteed. As we will show, on-policy bisimulation metrics can in fact obstruct policy search in some cases (e.g., sparse rewards, low dispersion\footnote{Statistical dispersion is an umbrella term used to describe measures of variability or diversity (e.g., variance).} rewards) by yielding a collapsed or exploded metric space, which is unable to approximate the value function. 
Specifically, we can define and relate measures of (i) collapse in the embedding space, and (ii) statistical dispersion of rewards under the current policy.
\begin{definition}[Measuring collapse and sparse rewards] 
Let $\rho^\pi$ denote the stationary distribution over states, and $\nu^\pi$ the distribution over pairs of states, $(\mathbf{s}_i, \mathbf{s}_j)$ sampled independently from $\rho^\pi$. Then;
\begin{align}
    \mu_{bd}^\pi &\coloneqq \mathbb{E}_{(\mathbf{s}_i, \mathbf{s}_j) \sim \nu^\pi}[d_\pi(\mathbf{s}_i, \mathbf{s}_j)] &&
    \mu_{rd}^\pi \coloneqq \mathbb{E}_{(\mathbf{s}_i, \mathbf{s}_j) \sim \nu^\pi}[|r_i^\pi - r_j^\pi|].
\end{align}
\end{definition}
In the low dispersion case, $\mu_{rd}^\pi \approx 0$, meaning the current reward signal under $\pi$ is uninformative. However, this can have dire consequences for our embedding, as shown in the lemma below.

\begin{restatable}[Relating collapse and low-dispersion rewards]{lemma}{meanupperbound}
\label{lemma:mean-lemma}
Assume deterministic transitions and the existence of a stationary distribution $\rho_\pi$ over states. Given a bisimulation metric of the form Eq.\ \eqref{eq:p-Wass-bisim-metric}:
\begin{align} 
\label{eq:reward-dependence2}
\mu^{\pi}_{bd} = \frac{c_R}{1-c_T} \mu^{\pi}_{rd}.
\end{align}
\end{restatable}
Clearly, a collapsed state encoder, $\phi^*(\mathbf{s}_i)=\phi_0$, 
loses all information about the state. 
These observations motivate us to extend the method with inverse dynamics-based regularization and intrinsic rewards based on forward prediction errors (see Sec \ref{sec:intrinsic}), since they promote $\mu_{bd}^\pi, \mu_{rd}^\pi > 0$ respectively.
For empirical evidence of the relation in Lemma \ref{lemma:mean-lemma} at training-time, see Figure \ref{fig:theoretical}. Relationships between variances are also discussed in Appendix \ref{sec:reward_scale}. %

\subsection{Intrinsic Rewards and Inverse Dynamics}
\label{sec:intrinsic}

The core principle of DBC is to construct a representation such that a given $\phi(\mathbf{s})$ relates to other latent states via the PBSM, thus providing robustness to distractors, but also ensuring the latent state holds sufficient information to maximize return.
As already discussed, however, uninformative rewards can induce metric learning issues for DBC, causing the embedding to hold insufficient information to solve the task.
We consider two approaches to improving upon this representational deficiency:
(1) using intrinsic rewards (IR), 
and
(2) regularizing the latent space with inverse dynamics (ID) learning.

\newcommand\szz{0.329}
\begin{figure} 
\centering
  \includegraphics[width=\szz\textwidth]{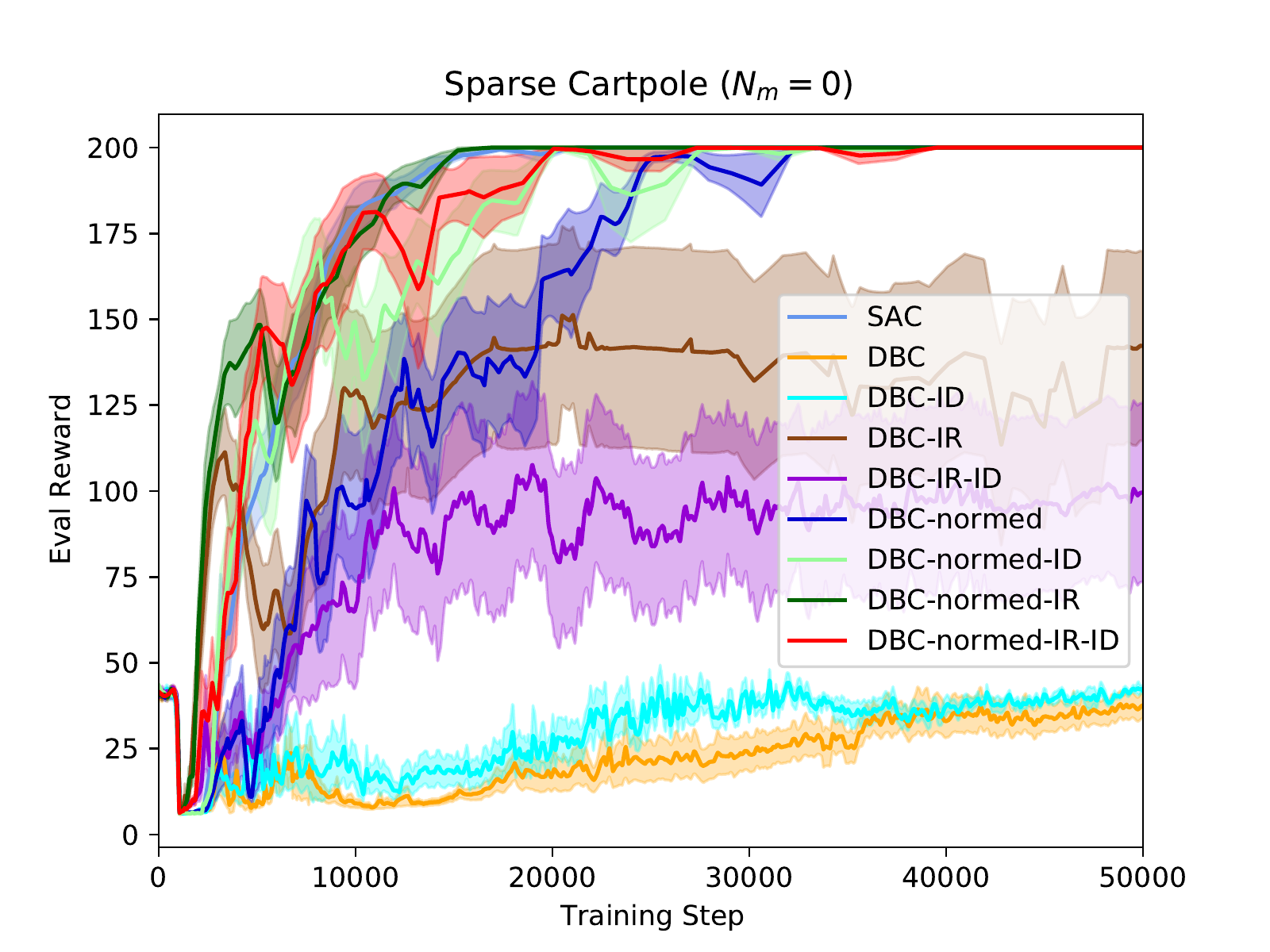}%
  \includegraphics[width=\szz\linewidth]{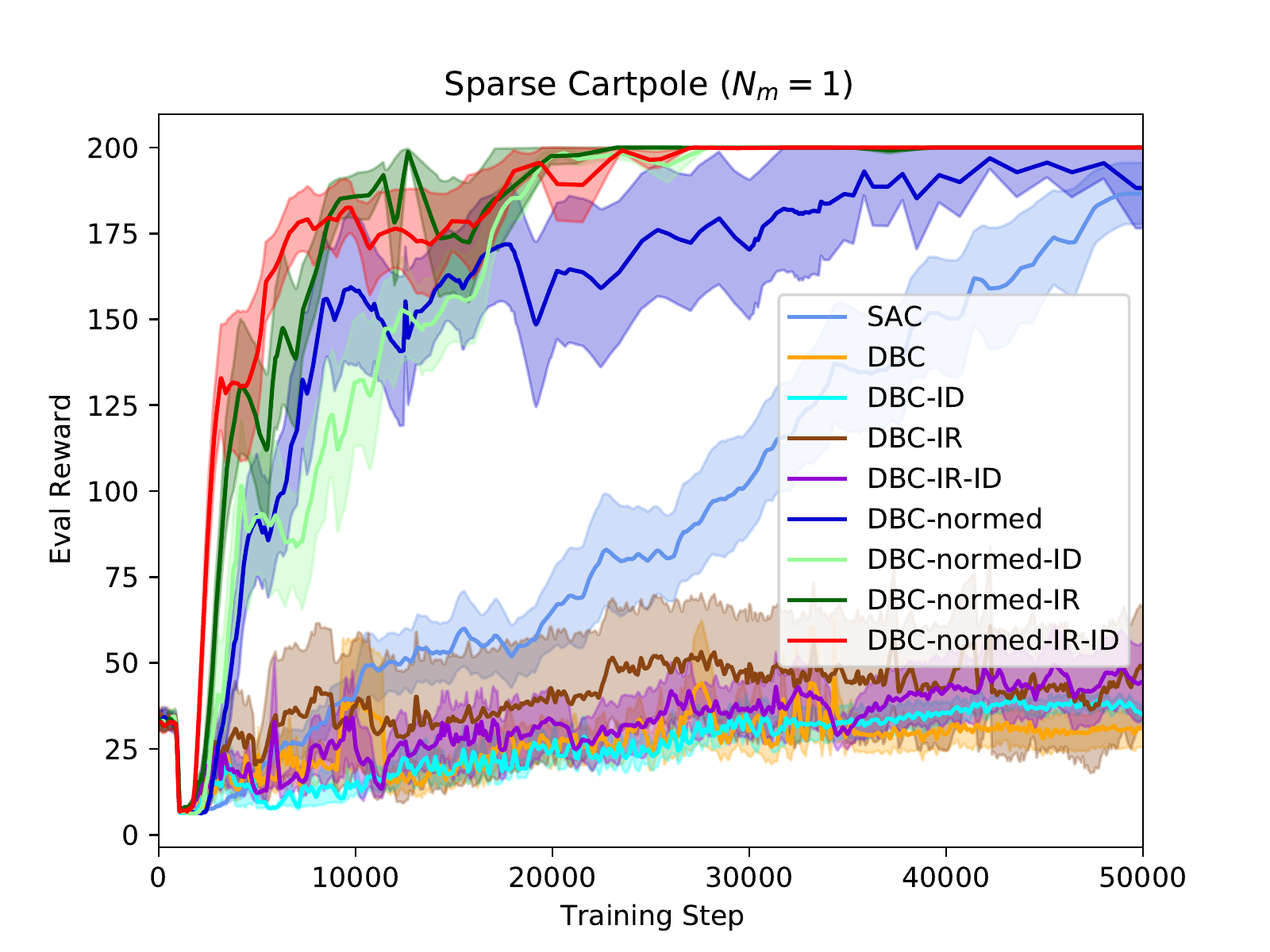}%
  \includegraphics[width=\szz\linewidth]{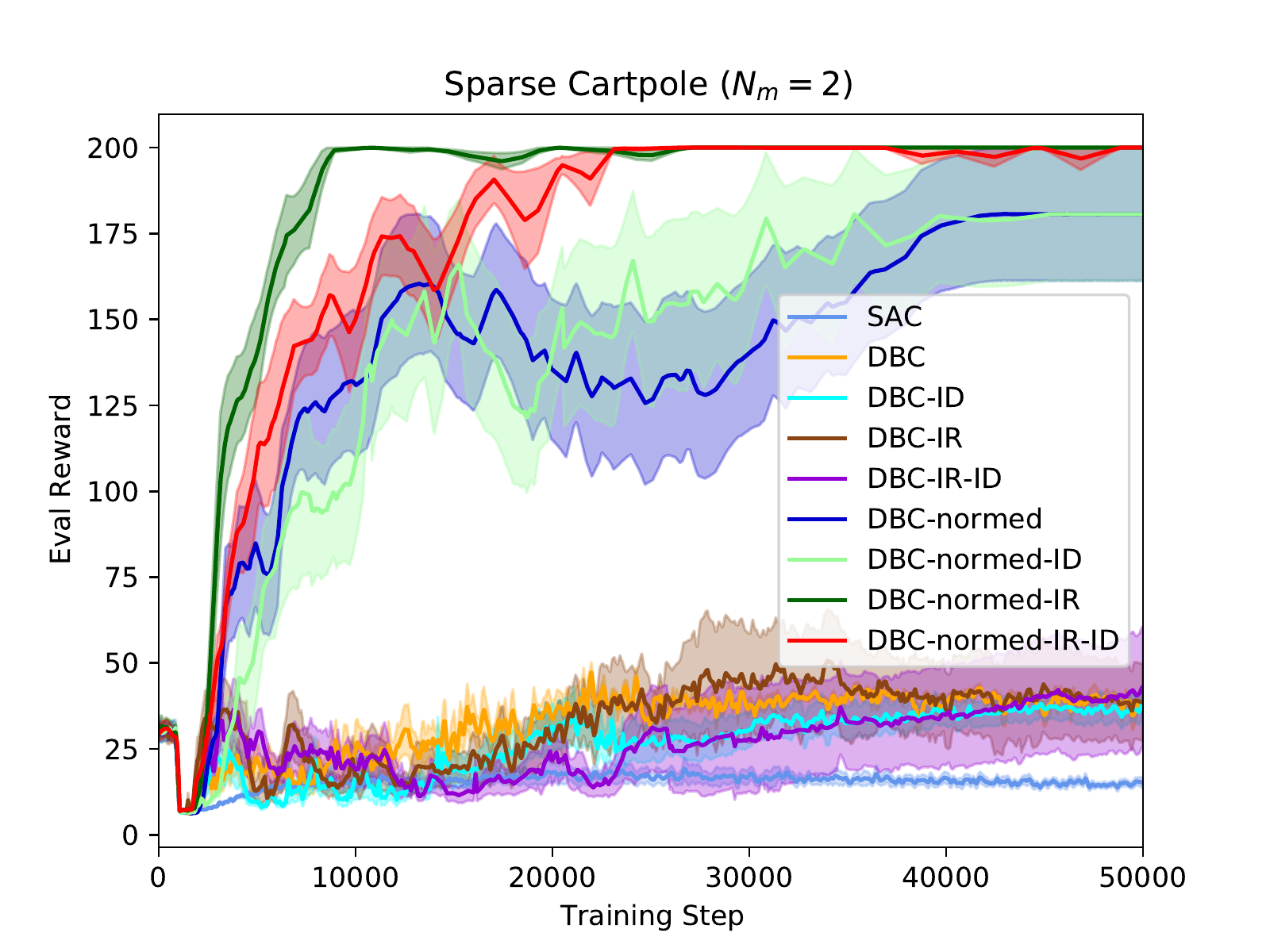}\\%
  \includegraphics[width=\szz\textwidth]{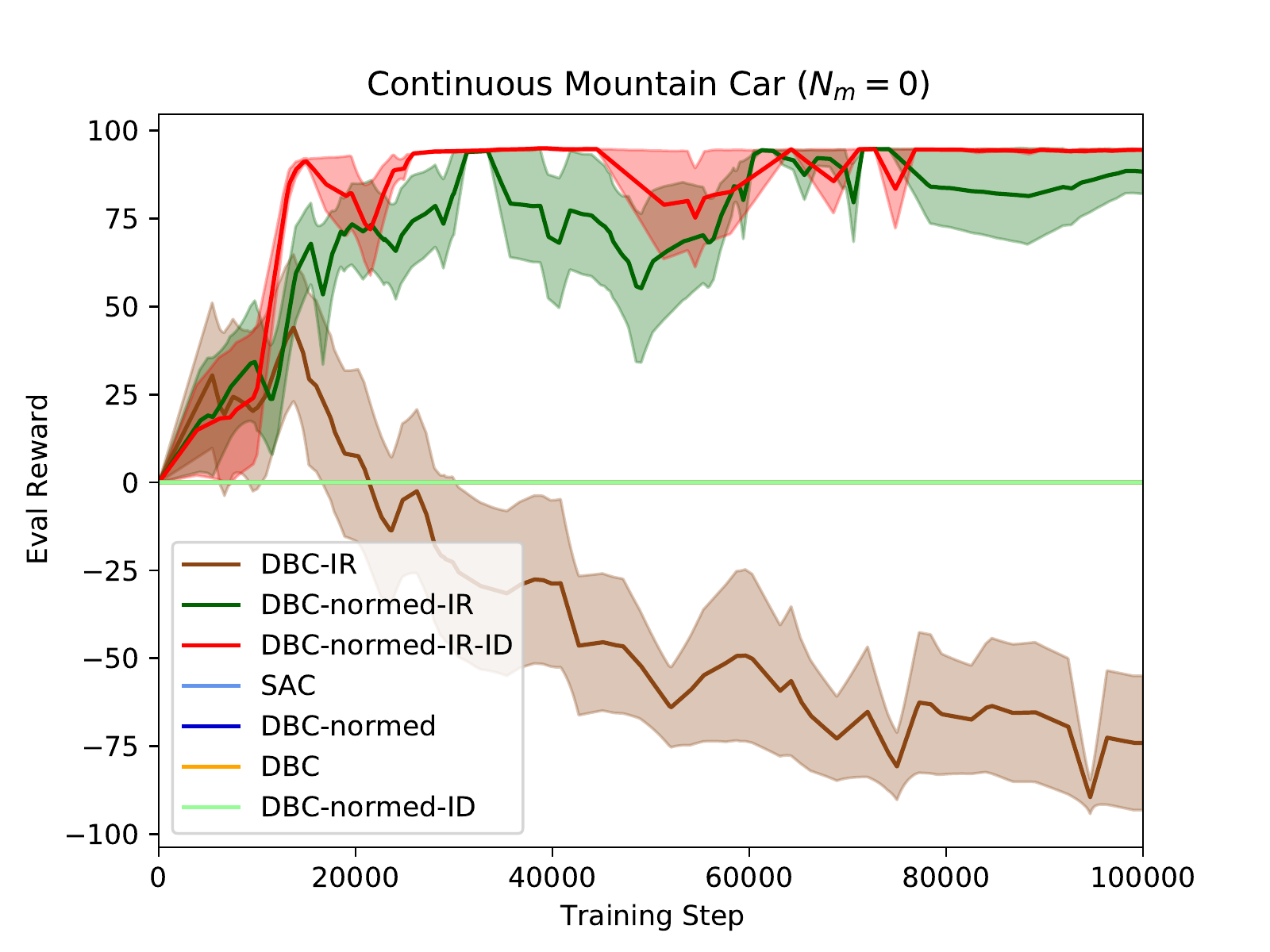}%
  \includegraphics[width=\szz\linewidth]{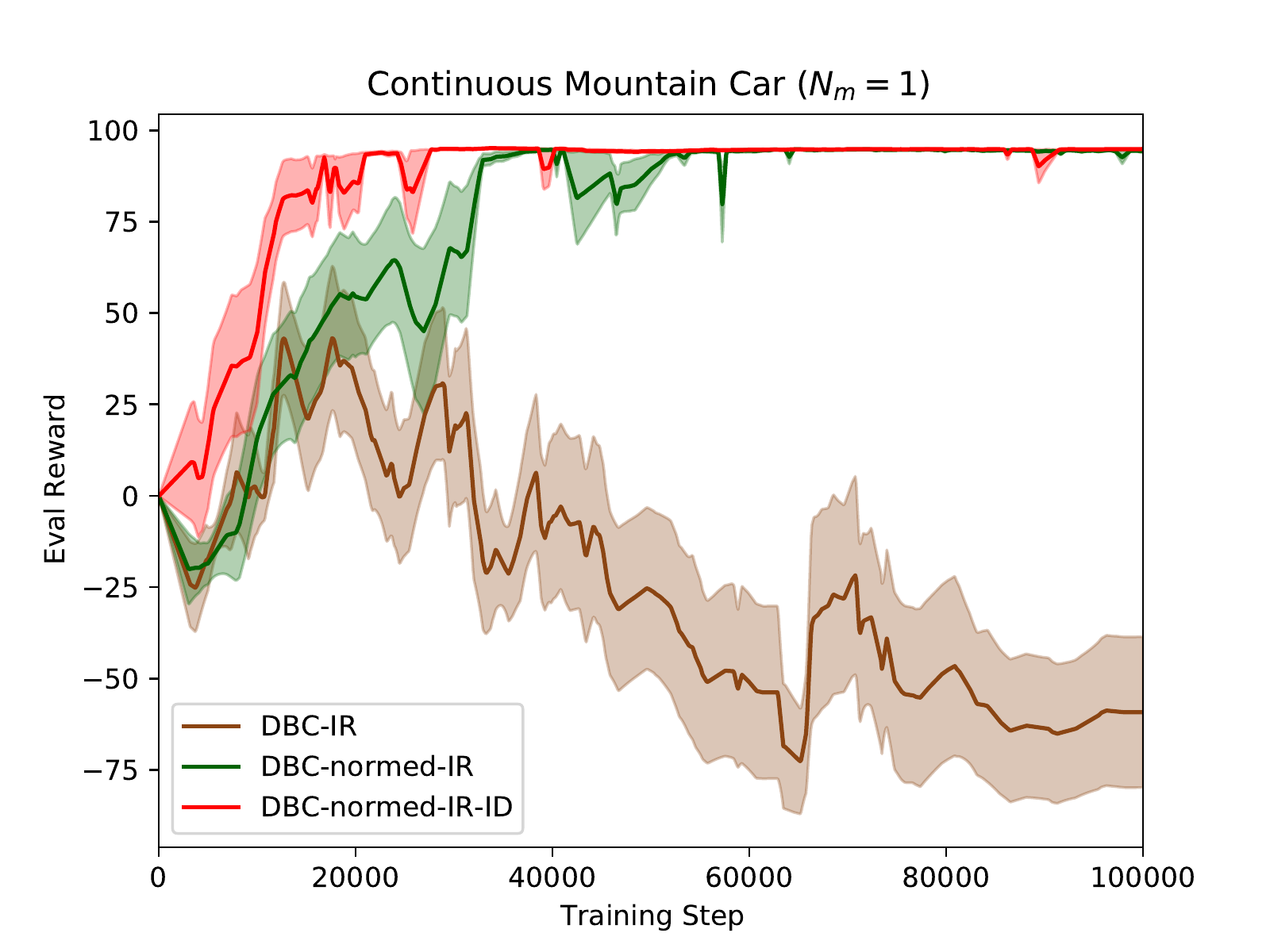}%
  \includegraphics[width=\szz\linewidth]{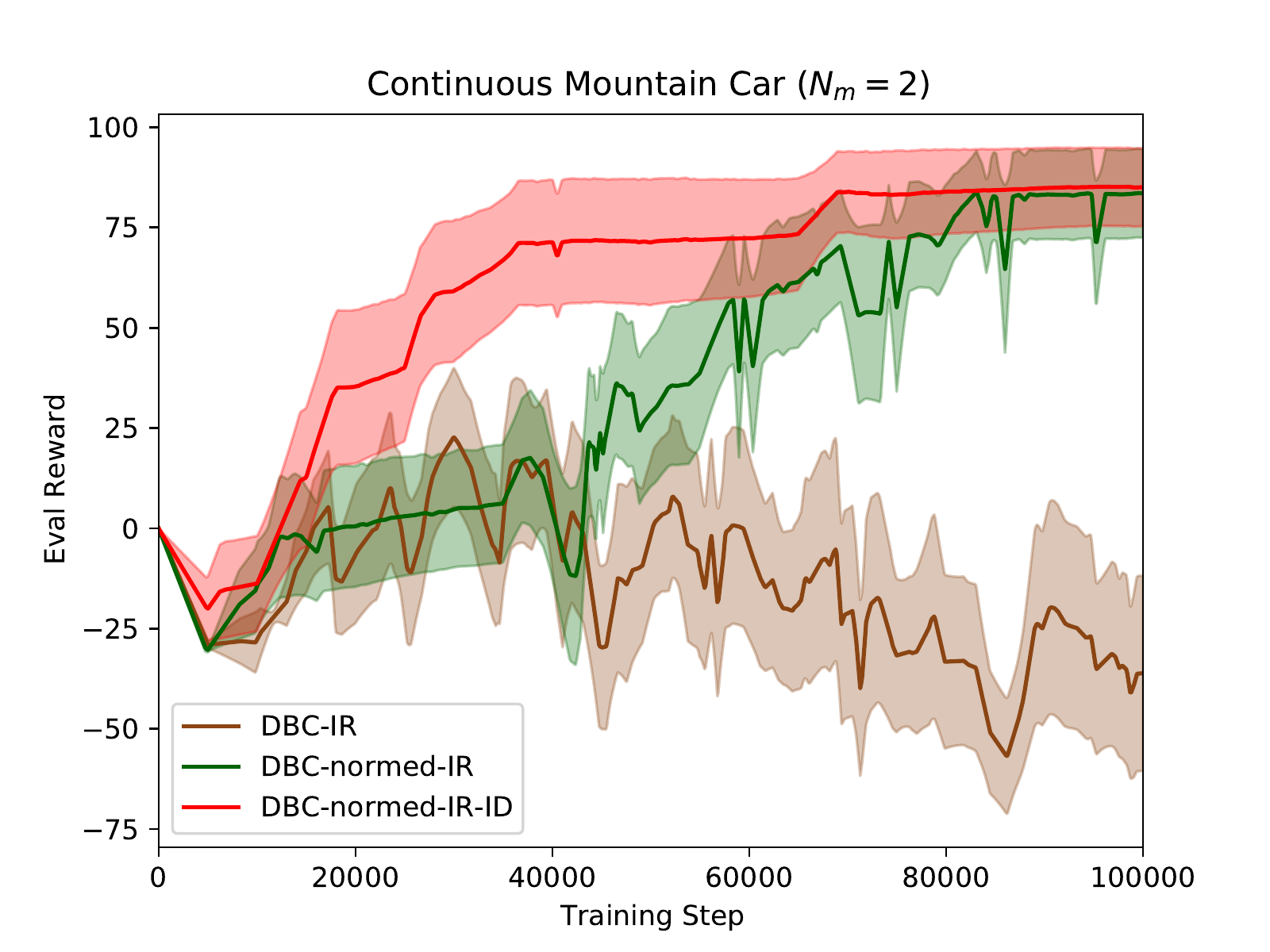}%
\caption{
Results on the \textit{modified} Gym tasks: Cartpole (\textbf{top row}) and  Mountain Car (\textbf{bottom row}).
The left, middle, and right columns have $0$, $\mathrm{dim}(\mathcal{S})$, and $2\mathrm{dim}(\mathcal{S})$
distractor noise dimensions respectively.
We show the 
Soft Actor-Critic (\textit{SAC}) and 
Deep Bisimulation for Control (\textit{DBC}) baselines,
along with our modifications using
embedding normalization (\textit{normed}),
intrinsic rewards (\textit{IR}),
and inverse dynamics (\textit{ID}).
DBC struggles under all conditions of sparsity;
while SAC handles it better (top-left), it cannot deal with high distraction (top-right) or more extreme sparsity (bottom row).
However, latent normalization immediately improves performance (top row).
Further, combining it with IR or IR+ID improves performance on all tasks.
Shading shows standard error over 10 seeds.
}
\label{fig:gym}
\end{figure}

\textbf{Curiosity-Driven Intrinsic Rewards}\hspace{0.25em}
One technique for encouraging exploration in sparse-reward environments uses intrinsic motivation via self-generated rewards, based on a notion of curiosity
\cite{burda2018large,aubret2019survey}.
In such cases, the agent rewards itself for entering unpredictable areas of the environment:
in particular, at every time step $t$, 
we redefine the reward signal to be 
$r_t = r_{I,t} + r_{E,t}$,
where $r_{E,t}$ is the extrinsic reward 
(i.e., the original environmental reward) and
$r_{I,t}$ is the IR. 
Following prior work \cite{burda2018large, pathak2017curiosity,stadie2015incentivizing}, 
we use the forward model error in the latent space to compute intrinsic rewards: 
$%
    \widetilde{r}_{I,t} \coloneqq \eta_r || \widehat{\phi}_\mu(\mathbf{s}_t, \mathbf{a}_t) - 
                                  \phi(\mathbf{s}_{t+1}) ||_2^2 / (2n),
$%
~where %
$ \widehat{\phi}_\mu(\mathbf{s}_t, \mathbf{a}_t) 
= \mathbb{E}_{ \phi(\mathbf{s}_{t+1}) \sim \widehat{\mathcal{P}}(\cdot|\phi(\mathbf{s}_t), \mathbf{a}_t) }[\phi(\mathbf{s}_{t+1})] $ 
is the mean of the predicted distribution from the forward dynamics model
and $\eta_r > 0$ is a hyper-parameter controlling the intrinsic reward scale.
We finally clamp to a fixed maximum $R_{\mathrm{max}, I}$, 
so that %
$r_{I,t} \coloneqq \min( R_{\mathrm{max}, I}, \widetilde{r}_{I,t} )$.
We already train $\widehat{\mathcal{P}}$ 
for the DBC loss,
so the additional computational cost is limited.
Finally, note that IRs cause the %
reward signal to become non-stationary, resulting in the target of the BSM learning changing as well;
however, as the agent explores and can better predict its environment over time, the IR (and thus its influence on the metric learning process) should fade.

\textbf{Latent Inverse Dynamics Prior}\hspace{0.25em}
We also consider regularizing the learned latent state space, by having it learn to predict the inverse dynamics of the world,
following prior work \cite{agrawal2016learning,pathak2017curiosity}. 
Let
$g_I : \mathbb{R}^n \times \mathbb{R}^n \rightarrow \mathcal{A}$
be an inverse model, 
which predicts the action $\mathbf{a}_t \in \mathcal{A}$ 
that changed $\mathbf{s}_t$
to $\mathbf{s}_{t+1}$ using the learned latent space: %
$ %
    \widehat{\mathbf{a}}_t \coloneqq g_I(\phi(\mathbf{s}_t), \phi(\mathbf{s}_{t+1})).
$ %
This can be trained from observed transitions
$(\mathbf{s}_t, \mathbf{a}_t, \mathbf{s}_{t+1})$, 
via
$J_{d}(\phi, g_I;\eta_d) \coloneqq \eta_d || \mathbf{a}_t - \widehat{\mathbf{a}}_t ||_1/n_a,$
where 
$\eta_d > 0$ 
weighs the ID loss importance,
and
$\mathcal{A} = [-1,1]^{n_a}$ for continuous control tasks.
This loss 
prevents removal of information
that pertains to the agent's ability to control the environment, 
but it is balanced against the loss driving $\phi$ to mimic the BSM.
In this sense, we are placing a regularizing prior on $\phi$ via an auxiliary task, 
such that it adheres to the BSM requirement when possible, 
but still transmits useful information in the sparse reward case instead of collapsing.
Yet, %
    since only \textit{actions} need be predicted,
    distractor aspects of the observation 
    that the agent cannot influence are naturally ignored.
Hence,
    we expect less disruption 
    to the BSM encoding than other auxiliary tasks 
    (e.g., reconstruction).

\textbf{Summary}\hspace{0.25em}
In this section, we 
(1) provided generalized bounds connecting the value function to the on-policy BSM (Sec.\ \ref{sec:vfa}); 
(2) showed that the BSM satisfies a fixed latent diameter, and devised a normalization for enforcing this property 
during learning
(Sec.\ \ref{sec:lipschitz-forward}), which not only prevents embedding explosion but improves performance (Fig.\ \ref{fig:theoretical});
and 
(3) found the PBSM is susceptible to embedding collapse with sparse rewards, 
and suggested IR and ID to mitigate it 
(Sec.\ \ref{sec:dangers}, \ref{sec:intrinsic}).

\section{Experiments}
\label{sec:experiments}
In this section, we seek answers to the following questions concerning our main hypotheses:
\vspace{-0.5em}
\begin{enumerate}
\itemsep-0.1em 
    \item Do the embedding collapse and explosion issues predicted theoretically occur in practice?
    \item Do our contributions address these problems?
    \item Does our proposed approach preserve the noise-invariance property of bisimulation?
    \item How do our proposed improvements interact with each other?
    \item How does our method perform compared to prior work, particularly with sparse rewards?
\end{enumerate}
\vspace{-0.5em}
To that end, we experiment on several altered classic control tasks from OpenAI Gym \cite{brockman2016openai} by (i) sparsifying the reward signal, and (ii) augmenting the environment state with noisy dimensions, to simulate distractions. 
We also perform larger scale experiments on two challenging vision-based 3D robotics benchmarks from the DeepMind Control Suite \cite{tassa2018deepmind}. 

\begin{figure}
\centering

\begin{minipage}{.3\textwidth}%
   \centering
   \vspace{-0.3cm}
    \includegraphics[width=0.75\linewidth]{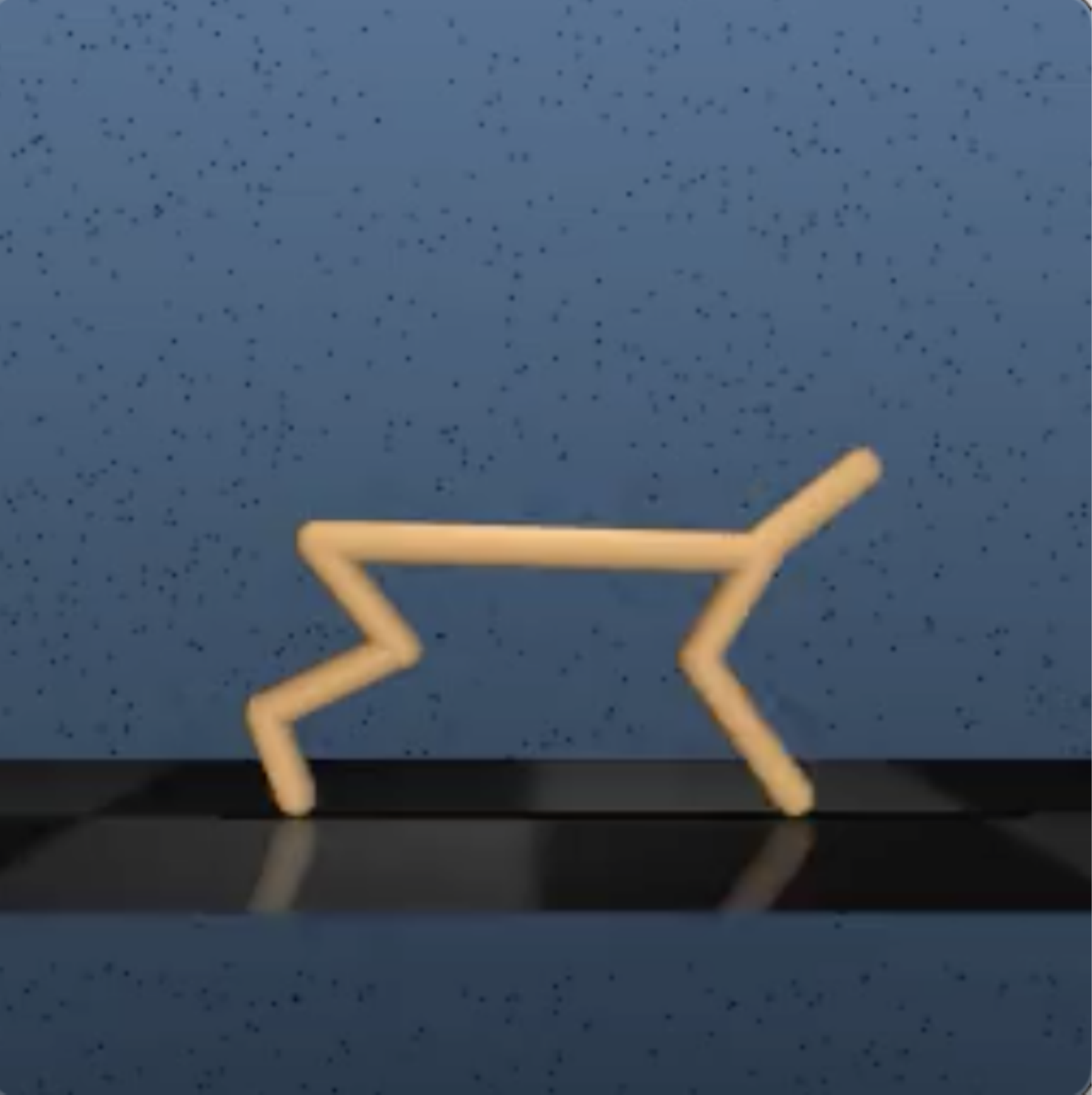}
    \par \vspace{0.8cm}
    \includegraphics[width=0.75\linewidth]{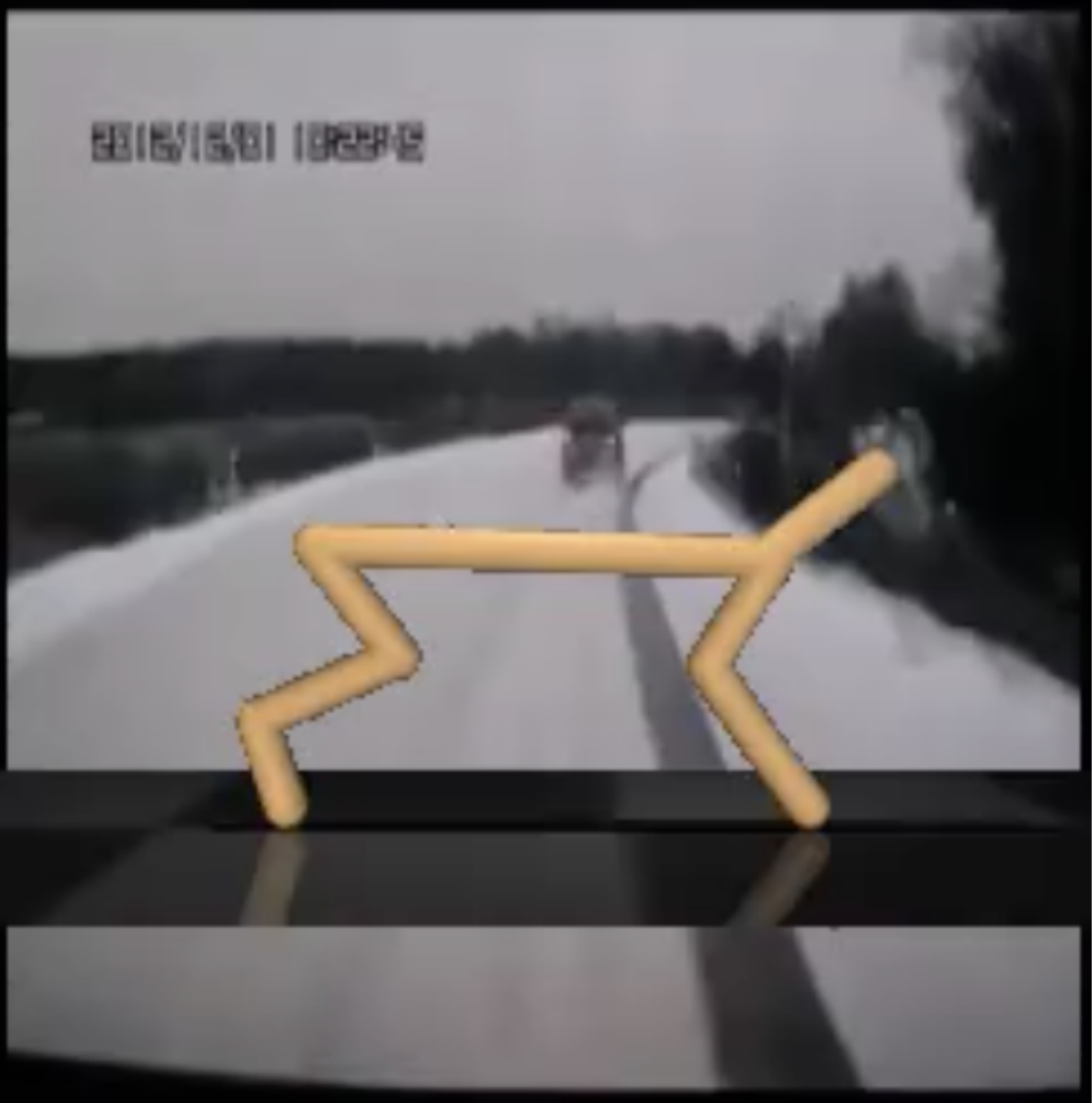}
\end{minipage}
\hspace{-0.65cm}
\begin{minipage}{.38\textwidth}%
   \centering
    \includegraphics[width=1.\linewidth]{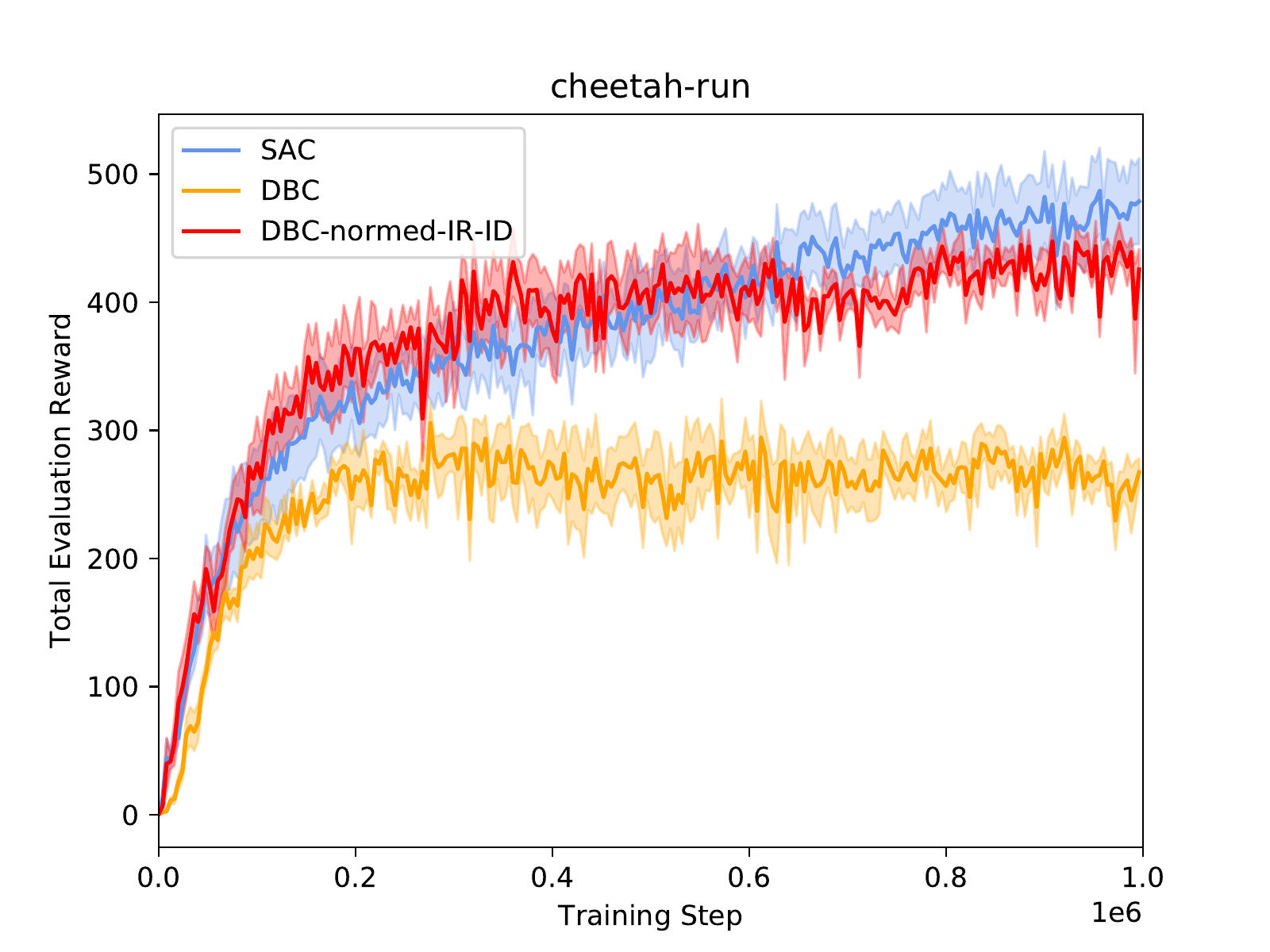}
    \includegraphics[width=1.\linewidth]{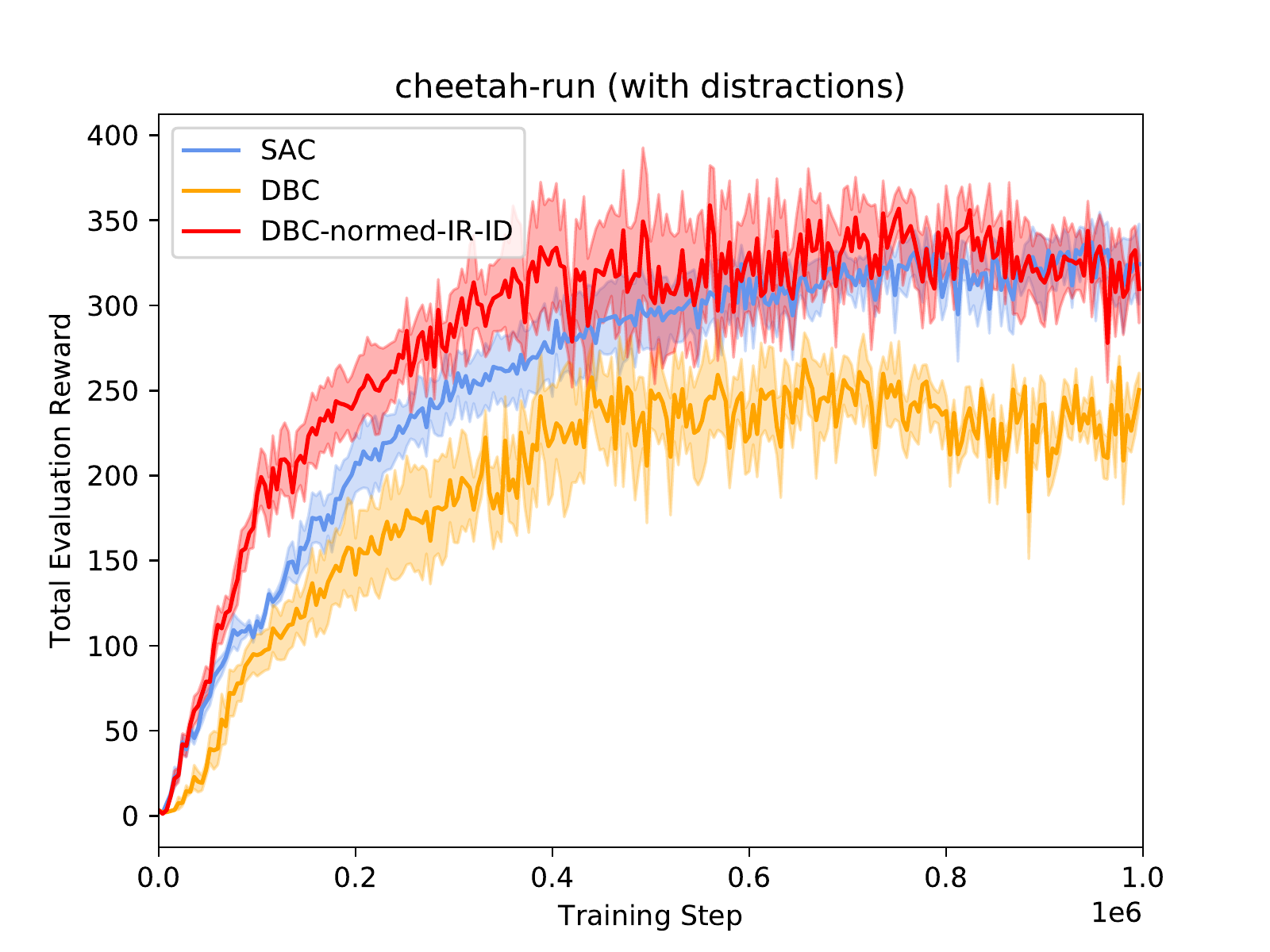}
\end{minipage}
\vspace{-10pt}
\hspace{-15pt}
\begin{minipage}{.38\textwidth} %
  \centering
  \includegraphics[width=1.\linewidth]{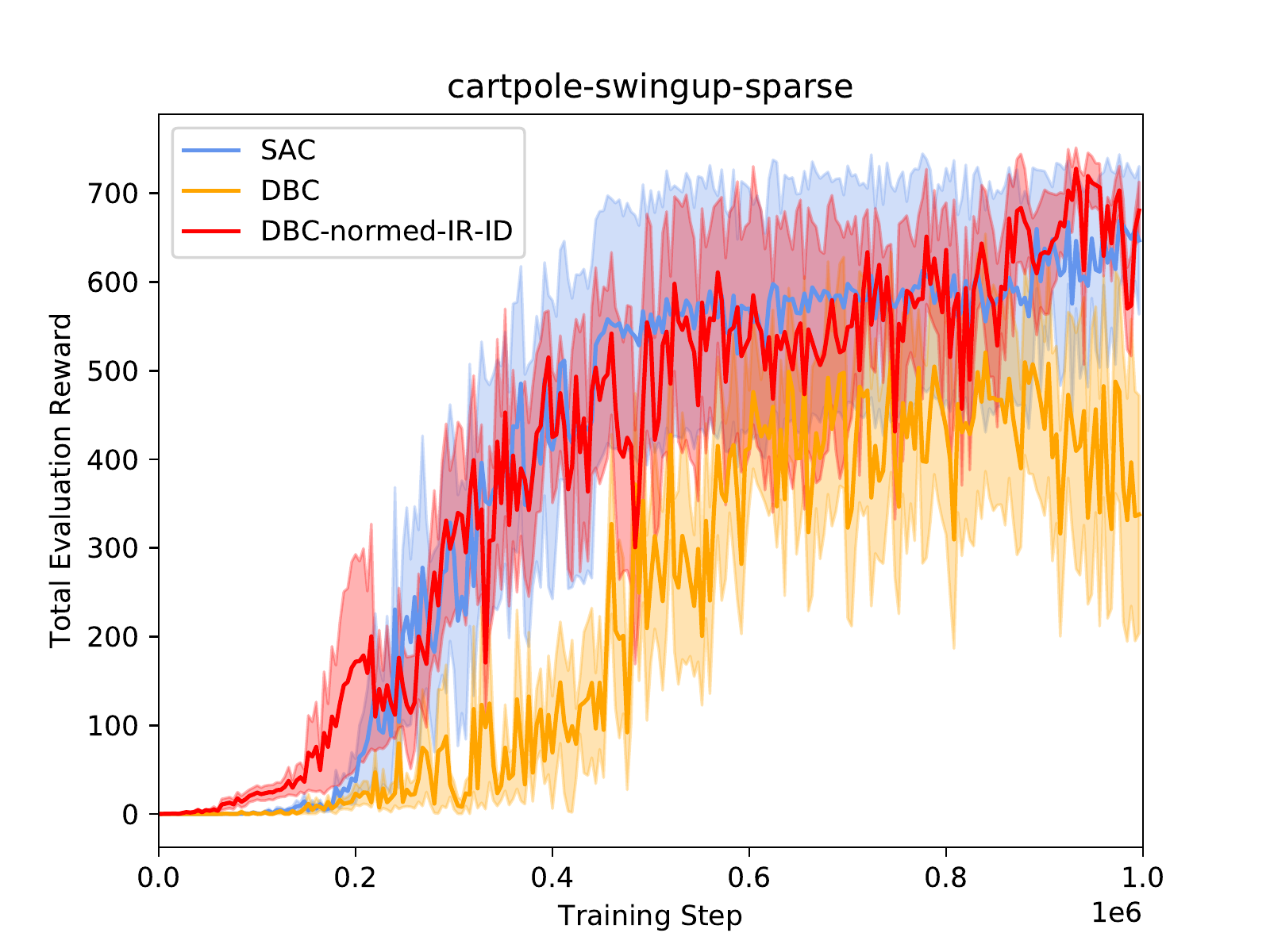}
  \includegraphics[width=1.\linewidth]{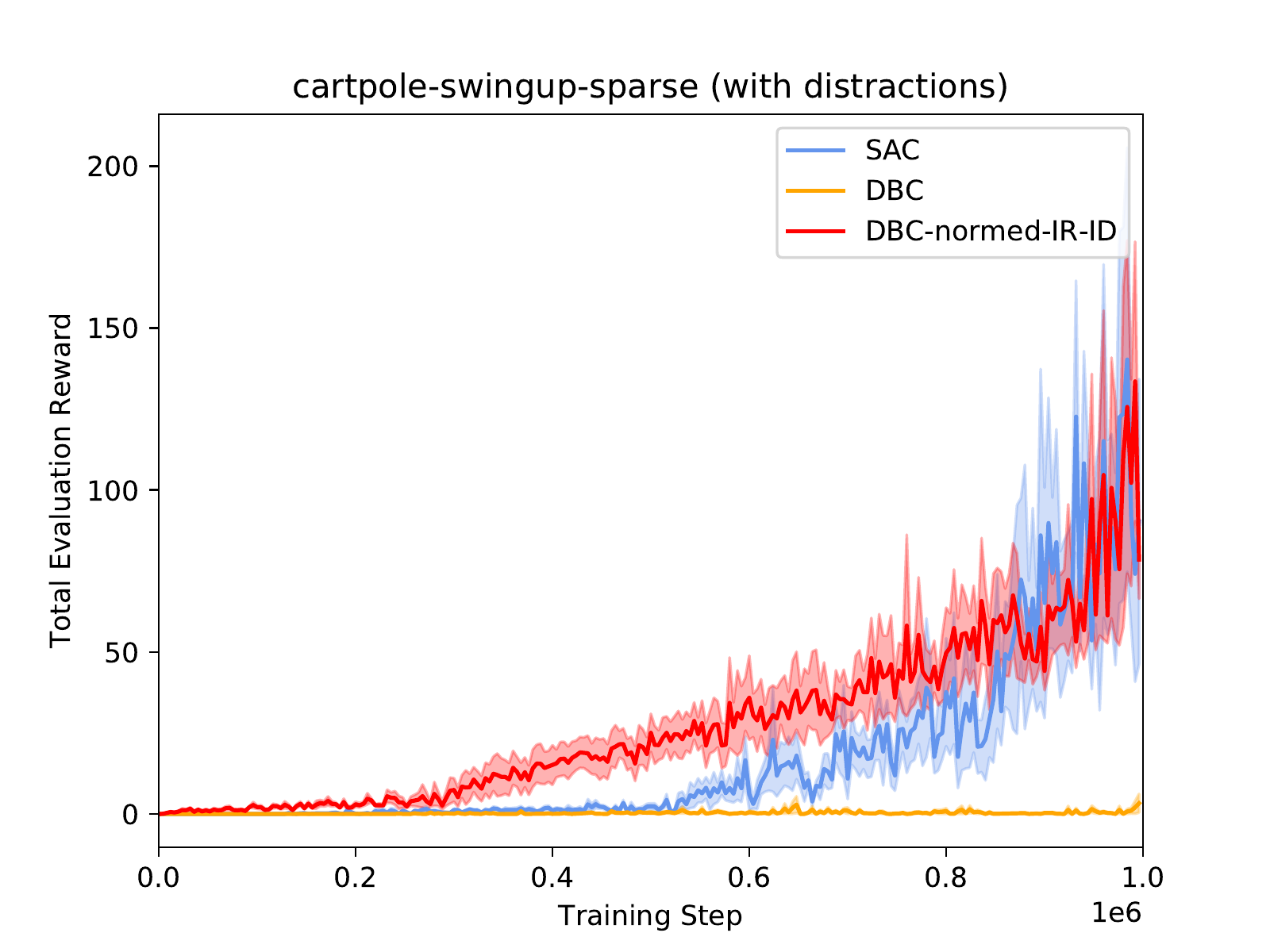}
\end{minipage}
\caption{Performance on 3D robotics tasks from DMC \cite{tassa2018deepmind} with 5 random seeds. \textbf{Leftmost column} illustrates the training setup for the \texttt{cheetah-run} task without (top) and with (bottom) natural video distractions. %
(\textbf{Top-left}) Our method significantly improves upon DBC \cite{zhang2021invariant} and is slightly below SAC \cite{haarnoja2018soft}. (\textbf{Top-right}) DBC struggles under sparse rewards, while our approach remains robust to sparsity. (\textbf{Bottom-left}) Our approach significantly outperforms DBC and is more sample-efficient than SAC. (\textbf{Bottom-right}) DBC makes no progress under both sparsity and distractions, and SAC begins learning later due to sparsity.}
\label{fig:dmc}
\end{figure}

\textbf{Sparse Noisy Cartpole}\hspace{0.25em}
First, we modified the \texttt{Cartpole-v0} task,
in which an agent tries to keep a pole upright.
To increase sparsity, we shrank the angular extent in which the pole must be to earn rewards to 1\% of its standard value.
Then, to mimic distractors, we concatenate an $N_m \mathrm{dim}(\mathcal{S})$-dimensional vector sampled from an isotropic Gaussian to the state vector. 
(See Appendix \ref{sec:additional-results:cp} for more details.)
The performance of the respective models, as well as other baselines are shown in Fig.\ \ref{fig:gym}, top row. 
While DBC struggles on this task due to the sparsity, the Soft Actor-Critic (SAC, \cite{haarnoja2018soft}) baseline, which does not use BSMs, is able to solve the problem (top-left inset). Our proposed modifications address the embedding explosion problem and perform on-par with the SAC baseline, both when combined and separately. Furthermore, when distractors are introduced
(top-middle and top-right insets), 
SAC also fails, while our approach combining intrinsic rewards (IR) and robust metric learning still solves the task.
See also Appendix \ref{sec:additional-results:nm3} for results with $N_m = 3$.

\textbf{Noisy Mountain Car}\hspace{0.25em}
We also consider the classic \texttt{MountainCarContinuous-v0} task \cite{moore1990efficient}, in its continuous control form (see Appendix \ref{sec:additional-results:mc} for details).
Since the task was already highly sparse, we modified only the distraction aspect, as in the Cartpole scenario. 
As shown in Fig.\ \ref{fig:gym},
all methods without IR are unable to solve the task and are stuck with rewards close to zero.
However, DBC with IR is unstable and unable to solve the task either.
Only using the normalization allows the agent to succeed; 
furthermore, the inverse dynamics (ID) regularization improves convergence speed and stability.
This trend continues at even higher distraction levels as well
(see Appendix \ref{sec:additional-results:nm3}).

\textbf{Sparse Noisy Pendulum}\hspace{0.25em}
We also modified the continuous \texttt{Pendulum-v0} task 
to have a higher degree of sparsity and distraction. 
Due to space constraints, we relegate details to Appendix \ref{sec:additional-results:pd}.
We find that our method performs on par with DBC, 
but outperforms SAC in the presence of distractions.

\textbf{DeepMind Control Suite}\hspace{0.25em}
The DeepMind Control Suite (DMC) \cite{tassa2018deepmind} contains a set of 3D robotics tasks based on the MuJoCo physics simulator \cite{todorov2012mujoco}. We include results from two DMC tasks, namely, \texttt{cheetah-run} ($\mathrm{dim}(\mathcal{S}){=}18, \mathrm{dim}(\mathcal{A}){=}6$) and \texttt{cartpole-swingup-sparse} ($\mathrm{dim}(\mathcal{S}){=}4, \mathrm{dim}(\mathcal{A}){=}1$), shown in Figure \ref{fig:dmc}. The former is a common benchmark where Zhang et al. \cite{zhang2021invariant} showed that their method performed sub-par in the \textit{absence} of distractors. The latter is a task we select due to reward sparsity. Similarly to \cite{zhang2021invariant}, we train for each task with and without natural video distractions, as illustrated in Figure \ref{fig:dmc}. In all cases, our approach performs significantly better than the DBC baseline \cite{zhang2021invariant}. Without distractors, our method performs competitively with SAC \cite{haarnoja2018soft}, but when distractions are introduced, our approach is slightly more sample-efficient. With these experiments, we verify that our improvements carry over to larger-scale tasks where learning is over raw-pixels.

\section{Conclusion}
\textbf{Limitations and Impact}\hspace{0.25em}
One shortcoming of our approach is the lack of a principled way to set hyper-parameters for IR and ID, which was done empirically. 
Indeed, our use of forward model error and ID regularization itself may not be optimal.
Sometimes, we observe that the embeddings all attain norms close to the maximum allowed value, suggesting alternative approaches to normalization may be helpful.
Finally, IR and ID incur an additional computational cost to training.
In terms of broader impact, 
while our work is foundational RL,
it focuses on removing distractions from internal representations.
However, this culling can remove information that the agent deems unrelated to the task, potentially leading to unsafe decisions in unseen scenarios (e.g., for mission critical systems).

\textbf{Discussion}\hspace{0.25em}
In this work, we investigated the behavior of on-policy deep bisimulation metric learning approaches, which construct efficient neural representations that are invariant to noise and distractors,  without reconstructing the raw input \cite{castro2020scalable, zhang2021invariant}. 
We identified embedding collapse and explosion as potential failure modes of this method via theoretical analysis, and highlighted that it may be especially susceptible to failure in sparse reward settings. 
Our experiments confirmed the dangers of these failure modes and showed that our proposed remedies address the issue. 
In particular, we enforced a norm constraint on state representations, 
and incorporated intrinsic motivation and latent regularization in our technique.
The resulting approach preserves the noise-invariance property of bisimulation metrics while comparing favorably against strong baselines \cite{haarnoja2018soft, zhang2021invariant} on altered versions of classic control tasks, 
which we rendered more challenging 
by sparsifying the rewards and synthesizing distractors, 
as well as harder tasks with visual observations.
Future work includes investigating alternative ways to improve embedding regularization and reward informativeness, as well as more realistic 3D robotics benchmarks, with sparse reward structures and heavy distraction.

\textbf{Acknowledgments}~~We thank Sven Dickinson, Allan Jepson, and Amir-massoud Farahmand for helpful discussions. The support of NSERC (CGSD3-534955-2019) and Samsung AI Research is gratefully acknowledged.

\bibliographystyle{plain}
\bibliography{references}

\newpage
\appendix
\section{Wasserstein Distances}
\label{sec:wasserstein}
\begin{definition}[Wasserstein metric \cite{villani2008optimal}]
\label{def:wass-primal}
Let $d: X \times X \rightarrow [0, \infty)$ be a distance function and $\Omega$ the set of all joint distributions with marginals $\mu$ and $\lambda$ over the space $X$;
\begin{align}
\label{eq:wasserstein}
    W_p(d)(\mu, \lambda) = \left(\inf_{\omega \in \Omega}\mathbb{E}_{(x_1, x_2) \sim \omega}[d(x_1, x_2)^p]  \right)^{\frac{1}{p}}.
\end{align}
\end{definition}
\begin{definition}[Dual formulation of the Wasserstein metric \cite{villani2008optimal}]
Let $d: X \times X \rightarrow [0, \infty)$ be a distance function, and $\mu$ and $\lambda$ marginals over the space $X$;
\begin{align}
\label{eq:wasserstein-dual}
    W_p(d)(\mu, \lambda) = \left(\sup_{\phi \oplus \psi \leq d^p}\mathbb{E}_{x_1 \sim \mu}[\phi(x_1)] + \mathbb{E}_{x_2 \sim \lambda}[\psi(x_2)] \right)^{\frac{1}{p}},
\end{align}
where $\phi \oplus \psi \leq d^p \iff \phi(x) + \psi(y) \leq d{(x, y)}^p, ~\forall(x, y) \in X \times X$.
\end{definition}
This dual formulation takes a simple form for $p=1$:
\begin{align}
\label{eq:wass-dual}
    W_1(d)(\mu, \lambda) = \sup_{f \in \mathrm{Lip}_{1,d}(X)}\mathbb{E}_{x_1 \sim \mu}[f(x_1)] - \mathbb{E}_{x_2 \sim \lambda}[f(x_2)],
\end{align}
where $\mathrm{Lip}_{1,d}(X)$ denotes $1$-Lipschitz functions $f: X \rightarrow \mathbb{R}$ such that $|f(x_1) - f(x_2)| \leq d(x_1, x_2)$.
Note that the 2-Wasserstein metric $W_2(\norm{\cdot}_2)$ (or simply $W_2$) has a closed-form for Gaussians \cite{olkin1982distance}:
\begin{align}
\label{eq: 2-Wasserstein}
    W_2(\mathcal{N}(\mu_i, \Sigma_i), \mathcal{N}(\mu_j, \Sigma_j))^2 = \norm{\mu_i - \mu_j}_2^2 + \norm{\Sigma_i - \Sigma_j}_{\mathcal{F}}^2,
\end{align}
where $\norm{\cdot}_{\mathcal{F}}$ denotes the Frobenius norm. We can observe that for point masses (i.e., $\Sigma_i, \Sigma_j \rightarrow 0$), the 2-Wasserstein metric is equivalent to the Euclidean distance between the two points.
\begin{lemma}[$p$-Wasserstein Inequality \cite{villani2008optimal}]
\label{lemma:wass-lemma}
For any two distributions $\mu, \lambda$, if $p \leq q$:
\begin{align}
    W_p(\mu, \lambda) \leq  W_q(\mu, \lambda).
\end{align}
\end{lemma}

\begin{lemma}[Bounds on Wasserstein distances \cite{santambrogio2015optimal}]
\label{lemma:wass-diam}
For any two distributions $\mu$, $\lambda$ over a space $X$, for all $p \geq 1$:
\begin{align}
W_1(\mu, \lambda) \leq W_p(\mu, \lambda) \leq \mathrm{diam}(X)^{\frac{p-1}{p}}W_1(\mu, \lambda)^{\frac{1}{p}}.
\end{align}
\end{lemma}

\section{Proofs}
\label{sec:proofs}

\existenceremark*
\begin{proof}
The existence proof is virtually identical to the proof of Thm. 3.12 of \cite{Ferns2011Bisimulation}, except it discards $\max_{\mathbf{a} \in \mathcal{A}}$ operations in favor of expectations under a policy $\pi$. We need to show that the following fixed-point update is a contraction:
\begin{align*}
    \mathcal{F}(d_\pi)(\mathbf{s}_i, \mathbf{s}_j) &\coloneqq c_R|r^{\pi}_i - r^{\pi}_j| ~+~ c_T
W_{p}(d_{\pi})(\mathcal{P}^\pi(\cdot| \mathbf{s}_i), \mathcal{P}^\pi(\cdot| \mathbf{s}_j)),
\end{align*}
and invoke the Banach fixed-point theorem to show the existence of a unique metric.

First, consider the case where $p=1$:
\begin{align*}
    & \mathcal{F}(d_\pi)(\mathbf{s}_i, \mathbf{s}_j) - \mathcal{F}(d_\pi^\prime)(\mathbf{s}_i, \mathbf{s}_j) \\
    &= c_T \left(W_1(d_\pi)(\mathcal{P}^\pi(\cdot | \mathbf{s}_i), \mathcal{P}^\pi(\cdot | \mathbf{s}_j)) -  W_1(d_\pi^\prime)(\mathcal{P}^\pi(\cdot | \mathbf{s}_i), \mathcal{P}^\pi(\cdot | \mathbf{s}_j)) \right)\\
    & = c_T \left(W_1(d_\pi-d_\pi^\prime+d_\pi^\prime)(\mathcal{P}^\pi(\cdot | \mathbf{s}_i), \mathcal{P}^\pi(\cdot | \mathbf{s}_j)) -  W_1(d_\pi^\prime)(\mathcal{P}^\pi(\cdot | \mathbf{s}_i), \mathcal{P}^\pi(\cdot | \mathbf{s}_j)) \right)\\
    & \leq c_T \left(W_1(\norm{d_\pi-d_\pi^\prime}_\infty+d_\pi^\prime)(\mathcal{P}^\pi(\cdot | \mathbf{s}_i), \mathcal{P}^\pi(\cdot | \mathbf{s}_j)) -  W_1(d_\pi^\prime)(\mathcal{P}^\pi(\cdot | \mathbf{s}_i), \mathcal{P}^\pi(\cdot | \mathbf{s}_j)) \right)\\
    & \leq c_T \left(\norm{d_\pi-d_\pi^\prime}_\infty + W_1(d_\pi^\prime)(\mathcal{P}^\pi(\cdot | \mathbf{s}_i), \mathcal{P}^\pi(\cdot | \mathbf{s}_j)) -  W_1(d_\pi^\prime)(\mathcal{P}^\pi(\cdot | \mathbf{s}_i), \mathcal{P}^\pi(\cdot | \mathbf{s}_j)) \right) \\
    &= c_T\norm{d_\pi-d_\pi^\prime}_\infty, ~\forall(\mathbf{s}_i, \mathbf{s}_j) \in \mathcal{S} \times \mathcal{S}.
\end{align*}
For $c_T \in [0, 1)$, there exists a unique fixed-point due to the Banach fixed-point theorem.

Next, we consider the case where both $\mathcal{P}$ and $\pi$ are deterministic, such that $\mathcal{P}^\pi$ is a delta distribution. Observe that for point masses, $W_p(d)(\delta(\mathbf{s}_i), \delta(\mathbf{s}_j)) = d(\mathbf{s}_i, \mathbf{s}_j)$, due to Definition \ref{def:wass-primal} of the Wasserstein metric. Then:
\begin{align*}
    \mathcal{F}(d_\pi)(\mathbf{s}_i, \mathbf{s}_j) - \mathcal{F}(d_\pi^\prime)(\mathbf{s}_i, \mathbf{s}_j) &= c_T \left(W_p(d_\pi)(\mathcal{P}^\pi(\cdot | \mathbf{s}_i), \mathcal{P}^\pi(\cdot | \mathbf{s}_j)) -  W_p(d_\pi^\prime)(\mathcal{P}^\pi(\cdot | \mathbf{s}_i), \mathcal{P}^\pi(\cdot | \mathbf{s}_j)) \right)\\
    & = c_T (d_\pi(\mathbf{s}_i^\prime, \mathbf{s}_j^\prime) - d_\pi^\prime(\mathbf{s}_i^\prime, \mathbf{s}_j^\prime)) \\
    & \leq c_T\norm{d_\pi-d_\pi^\prime}_\infty, ~\forall(\mathbf{s}_i, \mathbf{s}_j) \in \mathcal{S} \times \mathcal{S}.
\end{align*}
Then, fixed point iterations that update the metric as $d^{(n+1)}(\mathbf{s}_i, \mathbf{s}_j) \leftarrow \mathcal{F}(d^{(n)})(\mathbf{s}_i, \mathbf{s}_j)$ will converge for finite MDPs.
\end{proof}

\begin{restatable}[$p$-Wasserstein value difference bound]{lemma}{vfabound}
\label{lemma:p-wasserstein}
For an on-policy bisimulation metric given by Eq. \eqref{eq:p-Wass-bisim-metric},
for any $c_T \in [\gamma, 1)$ and $p \geq 1$, the bisimulation distance between a pair of states upper-bounds the difference in their values:
\begin{align}
\label{eq:vfa}
    c_R|V^\pi(\mathbf{s}_i) - V^\pi(\mathbf{s}_j)| ~\leq~ d_\pi(\mathbf{s}_i, \mathbf{s}_j), ~\forall(\mathbf{s}_i, \mathbf{s}_j) \in \mathcal{S} \times \mathcal{S}.
\end{align}
\end{restatable}
\begin{proof}
The proof follows similarly to the proofs of Thm. 5.1 of \cite{Ferns2004Metrics} and Thm. 3 of \cite{castro2020scalable}. We will prove by induction. Consider the following updates:
\begin{align*}
    V^{(n+1)}(\mathbf{s}_i) &= r_i^\pi + \gamma \int_{\mathbf{s}^\prime \in \mathcal{S}}\mathcal{P}^\pi(\mathbf{s}^\prime|\mathbf{s}_i)V^{(n)}(\mathbf{s}^\prime)d\mathbf{s}^\prime\\
    d^{(n+1)}_\pi(\mathbf{s}_i, \mathbf{s}_j) &= c_R|r_i^\pi - r_j^\pi| ~+~ c_T
W_{p}(d^{(n)}_\pi)(\mathcal{P}^\pi(\cdot| \mathbf{s}_i), \mathcal{P}^\pi(\cdot| \mathbf{s}_j)).
\end{align*}
We need to show that the following holds for all $n \in \mathbb{N}$:
\begin{align*}
        c_R\left|V^{(n)}(\mathbf{s}_i) - V^{(n)}(\mathbf{s}_j)\right| ~\leq~ d_\pi^{(n)}(\mathbf{s}_i, \mathbf{s}_j), ~\forall(\mathbf{s}_i, \mathbf{s}_j) \in \mathcal{S} \times \mathcal{S}.
\end{align*}
Then, Eq. \eqref{eq:vfa} holds by taking a limit $n \rightarrow \infty$.
The base case holds since:
\begin{align*}
    \left|V^{(0)}(\mathbf{s}_i) - V^{(0)}(\mathbf{s}_j)\right|=d^{(0)}_\pi(\mathbf{s}_i, \mathbf{s}_j)=0, ~\forall (\mathbf{s}_i, \mathbf{s}_j) \in \mathcal{S} \times \mathcal{S}.
\end{align*}
In the general case:
\begin{align*}
    c_R|V^{(n+1)}(\mathbf{s}_i) - V^{(n+1)}(\mathbf{s}_j)| &= c_R \left|r_i^\pi - r_j^\pi + \gamma \int_{\mathbf{s}^\prime \in \mathcal{S}}\left(\mathcal{P}^\pi(\mathbf{s}^\prime|\mathbf{s}_i)-\mathcal{P}^\pi(\mathbf{s}^\prime|\mathbf{s}_j)\right)V^{(n)}(\mathbf{s}^\prime)d\mathbf{s}^\prime\right|\\
    &\leq c_R\left|r_i^\pi - r_j^\pi\right| + c_R \gamma \left|\int_{\mathbf{s}^\prime \in \mathcal{S}}\left(\mathcal{P}^\pi(\mathbf{s}^\prime|\mathbf{s}_i)-\mathcal{P}^\pi(\mathbf{s}^\prime|\mathbf{s}_j)\right)V^{(n)}(\mathbf{s}^\prime)d\mathbf{s}^\prime\right| \\
    &= c_R\left|r_i^\pi - r_j^\pi\right| + c_T \left|\int_{\mathbf{s}^\prime \in \mathcal{S}}\left(\mathcal{P}^\pi(\mathbf{s}^\prime|\mathbf{s}_i)-\mathcal{P}^\pi(\mathbf{s}^\prime|\mathbf{s}_j)\right)\frac{c_R \gamma}{c_T}V^{(n)}(\mathbf{s}^\prime)d\mathbf{s}^\prime\right|.
\end{align*}
Notice that by the induction hypothesis, $c_R V^{(n)}(\mathbf{s})$ is a 1-Lipschitz function with respect to the distance function $d_\pi^{(n)}$, i.e., $c_R V^{(n)}(\mathbf{s}) \in \mathrm{Lip}_{1, d_\pi^{(n)}}$. Then, since $\gamma \leq c_T$ by assumption, $\frac{c_R \gamma}{c_T}V^{(n)}(\mathbf{s})$ is also 1-Lipschitz. Using the dual form of the $W_1$ metric in Eq. \eqref{eq:wass-dual}:
\begin{align*}
    c_R|V^{(n+1)}(\mathbf{s}_i) - V^{(n+1)}(\mathbf{s}_j)| &\leq c_R\left|r_i^\pi - r_j^\pi\right| + c_T W_1(d_\pi^{(n)})(\mathcal{P}^\pi(\cdot|\mathbf{s}_i), \mathcal{P}^\pi(\cdot|\mathbf{s}_j))\\
    &\leq c_R\left|r_i^\pi - r_j^\pi\right| + c_T W_p(d_\pi^{(n)})(\mathcal{P}^\pi(\cdot|\mathbf{s}_i), \mathcal{P}^\pi(\cdot|\mathbf{s}_j))\\
    &= d_\pi^{(n+1)},
\end{align*}
where the last inequality is due to Lemma \ref{lemma:wass-lemma}.
\end{proof}

\begin{lemma}[Value function difference bound for different discount factors \cite{petrik2008biasing}]
\label{lemma:diff-gamma}
Consider two otherwise identical MDPs with different discount factors $\gamma_1 \leq \gamma_2$, and a bounded reward function $R \in [0, 1]$. Let $V^\pi_{\gamma}$ denote the value function for policy $\pi$ given discount factor $\gamma$.
\begin{align}
    \left|V_{\gamma_1}^\pi(\mathbf{s}) - V_{\gamma_2}^\pi(\mathbf{s}) \right| \leq \frac{\gamma_2 - \gamma_1}{(1-\gamma_1)(1-\gamma_2)}, \forall \mathbf{s} \in \mathcal{S}.
\end{align}
\end{lemma}

\begin{proof}
See Thm. 2 of \cite{petrik2008biasing}.
\end{proof}

\generalizedvdbound*
\begin{proof}
A bisimulation metric with $c_T \leq \gamma$ can be viewed as approximating a value function for another MDP with $\gamma^\prime = c_T$. We will make use of this view and apply Lemma \ref{lemma:p-wasserstein} with Lemma \ref{lemma:diff-gamma} to derive the above relation.

First, note that by Lemma \ref{lemma:p-wasserstein}:
\begin{align*}
    c_R|V_{c_T}^\pi(\mathbf{s}_i) - V_{c_T}^\pi(\mathbf{s}_j)| ~\leq~ d_\pi(\mathbf{s}_i, \mathbf{s}_j), ~\forall(\mathbf{s}_i, \mathbf{s}_j) \in \mathcal{S} \times \mathcal{S}.
\end{align*}
Then;
\begin{align*}
    c_R|V^\pi(\mathbf{s}_i) - V^\pi(\mathbf{s}_j)| &= c_R\left|V^\pi(\mathbf{s}_i) - V_{c_T}^\pi(\mathbf{s}_i) + V_{c_T}^\pi(\mathbf{s}_i) - V^\pi(\mathbf{s}_j) + V_{c_T}^\pi(\mathbf{s}_j) - V_{c_T}^\pi(\mathbf{s}_j)\right| \\
    & \leq c_R \left(|V^\pi(\mathbf{s}_i) - V_{c_T}^\pi(\mathbf{s}_i)| + |V^\pi(\mathbf{s}_j) - V_{c_T}^\pi(\mathbf{s}_j)| + |V_{c_T}^\pi(\mathbf{s}_i) - V_{c_T}^\pi(\mathbf{s}_j)| \right) \\
    & \leq d_\pi(\mathbf{s}_i, \mathbf{s}_j) + c_R\left(|V^\pi(\mathbf{s}_i) - V_{c_T}^\pi(\mathbf{s}_i)| + |V^\pi(\mathbf{s}_j) - V_{c_T}^\pi(\mathbf{s}_j)|\right) \\
    &\leq d_\pi(\mathbf{s}_i, \mathbf{s}_j) + \frac{2c_R(\gamma - c_T)}{(1-\gamma)(1-c_T)},
\end{align*}
where the last inequality is due to Lemma \ref{lemma:diff-gamma}. Due to Lemma \ref{lemma:p-wasserstein}:
\begin{align}
    c_R|V^\pi(\mathbf{s}_i) - V^\pi(\mathbf{s}_j)| ~\leq~ d_\pi(\mathbf{s}_i, \mathbf{s}_j) +  \frac{2c_R(\gamma - \min(c_T, \gamma))}{(1-\gamma)(1-c_T)}.
\end{align}
\end{proof}

\begin{restatable}[On-policy VFA bound]{lemma}{onpolicyvfabound}
\label{lemma:on-policy-vfa}
Let the reward function be bounded as $R \in [0, 1]$ and $\Phi: \mathcal{S} \rightarrow \widetilde{\mathcal{S}}$ a function mapping states to a finite partitioning $\widetilde{\mathcal{S}}$ such that $\Phi(\mathbf{s}_i) = \Phi(\mathbf{s}_j) \Rightarrow d_\pi(\mathbf{s}_i, \mathbf{s}_j) \leq 2\epsilon$, which produces an aggregated MDP $\langle \widetilde{\mathcal{S}}, \mathcal{A}, \widetilde{\mathcal{P}}, \widetilde{R}, \widetilde{\rho_0} \rangle$. For $c_T \in [\gamma, 1)$:
\begin{align}
    |V^\pi(\mathbf{s}_i) - \widetilde{V}^\pi(\Phi(\mathbf{s}_i))| ~\leq~ \frac{2\epsilon}{c_R(1-\gamma)}, ~\forall\mathbf{s}_i \in \mathcal{S}.
\end{align}
\end{restatable}\begin{proof}
Let $\xi$ be a measure on $\mathcal{S}$. Given a partition $\Phi(\mathbf{s}) \in \widetilde{\mathcal{S}}$, i.e., a set of points in $\mathcal{S}$ clustered in an $\epsilon$-neighborhood such that $\xi(\Phi(\mathbf{s})) > 0$, we can define the reward function and transition probabilities of a $\xi$-average finite MDP as in Thm. 3.21 of \cite{Ferns2011Bisimulation}:
\begin{align*}
    \tilde{r}^\pi(\Phi(\mathbf{s})) &= \frac{1}{\xi(\Phi(\mathbf{s}))}\int\limits_{\mathbf{z} \in \Phi(\mathbf{s})} r^\pi(\mathbf{z})d\xi(\mathbf{z}), \\
    \widetilde{\mathcal{P}}^\pi(\Phi(\mathbf{s}^\prime)|\Phi(\mathbf{s})) &= \frac{1}{\xi(\Phi(\mathbf{s}))}\int\limits_{\mathbf{z} \in \Phi(\mathbf{s})} \mathcal{P}^\pi(\Phi(\mathbf{s}^\prime)| \mathbf{z})d\xi(\mathbf{z}).
\end{align*}
Then,
\begin{align*}
    \begin{split}
        &|V^\pi(\mathbf{s}) - \widetilde{V}^\pi(\Phi(\mathbf{s}))| \\
        &= \bigg| r^\pi(\mathbf{s}) - \tilde{r}^\pi(\Phi(\mathbf{s})) + \gamma  \int\limits_{\mathbf{s}^\prime \in \mathcal{S}} \mathcal{P}^\pi(\mathbf{s}^\prime|\mathbf{s})V^\pi(\mathbf{s}^\prime)d\mathbf{s}^\prime - \gamma \hspace{-8pt} \int\limits_{\Phi(\mathbf{s}^\prime) \in \widetilde{\mathcal{S}}} \hspace{-8pt} \widetilde{\mathcal{P}}^\pi(\Phi(\mathbf{s}^\prime)|\Phi(\mathbf{s}))\widetilde{V}^\pi(\Phi(\mathbf{s}^\prime))d\Phi(\mathbf{s}^\prime) \bigg| 
    \end{split} 
    \\ 
    \begin{split}
    & \leq \frac{1}{\xi(\Phi(\mathbf{s}))}\int\limits_{\mathbf{z} \in \Phi(\mathbf{s})} \left | r^\pi(\mathbf{s}) - r^\pi(\mathbf{z}) \right| + \gamma \left|~ \int\limits_{\mathbf{s}^\prime \in \mathcal{S}}  \mathcal{P}^\pi(\mathbf{s}^\prime|\mathbf{s})V^\pi(\mathbf{s}^\prime)d\mathbf{s}^\prime - \hspace{-10pt} \int\limits_{\Phi(\mathbf{s}^\prime) \in \widetilde{\mathcal{S}}} \hspace{-10pt} \widetilde{\mathcal{P}}^\pi(\Phi(\mathbf{s}^\prime)|\mathbf{z})\widetilde{V}^\pi(\Phi(\mathbf{s}^\prime))d\Phi(\mathbf{s}^\prime) \right| d\xi(\mathbf{z})
    \end{split}
    \\
    & \leq \frac{1}{\xi(\Phi(\mathbf{s}))}\int\limits_{\mathbf{z} \in \Phi(\mathbf{s})} \left | r^\pi(\mathbf{s}) - r^\pi(\mathbf{z}) \right| + \gamma \left|~ \int\limits_{\mathbf{s}^\prime \in \mathcal{S}}  \mathcal{P}^\pi(\mathbf{s}^\prime|\mathbf{s})V^\pi(\mathbf{s}^\prime) - \mathcal{P}^\pi(\mathbf{s}^\prime|\mathbf{z})\widetilde{V}^\pi(\Phi(\mathbf{s}^\prime)) d\mathbf{s}^\prime \right| d\xi(\mathbf{z}) 
    \\
    \begin{split}
    & \leq \frac{1}{\xi(\Phi(\mathbf{s}))}\int\limits_{\mathbf{z} \in \Phi(\mathbf{s})} \left | r^\pi(\mathbf{s}) - r^\pi(\mathbf{z}) \right| + \gamma \left|~ \int\limits_{\mathbf{s}^\prime \in \mathcal{S}}  \mathcal{P}^\pi(\mathbf{s}^\prime|\mathbf{s})V^\pi(\mathbf{s}^\prime) - \mathcal{P}^\pi(\mathbf{s}^\prime|\mathbf{z})V^\pi(\mathbf{s}^\prime) d\mathbf{s}^\prime \right| d\xi(\mathbf{z}) \\
    &\qquad + \gamma \frac{1}{\xi(\Phi(\mathbf{s}))}\int\limits_{\mathbf{z} \in \Phi(\mathbf{s})} \left| \int_{\mathbf{s}^\prime \in \mathcal{S}} \mathcal{P}^\pi(\mathbf{s}^\prime|\mathbf{z})\left(V^\pi(\mathbf{s}^\prime) - \widetilde{V}^\pi(\Phi(\mathbf{s}^\prime))\right)d\mathbf{s}^\prime \right| d\xi(\mathbf{z})
    \end{split}
    \\
    \begin{split}
    & \leq \frac{1}{\xi(\Phi(\mathbf{s}))}\int\limits_{\mathbf{z} \in \Phi(\mathbf{s})} \left | r^\pi(\mathbf{s}) - r^\pi(\mathbf{z}) \right| + \gamma \left|~ \int\limits_{\mathbf{s}^\prime \in \mathcal{S}}  \left(\mathcal{P}^\pi(\mathbf{s}^\prime|\mathbf{s}) - \mathcal{P}^\pi(\mathbf{s}^\prime|\mathbf{z})\right)V^\pi(\mathbf{s}^\prime) d\mathbf{s}^\prime \right| d\xi(\mathbf{z}) \\
    &\qquad + \gamma\norm{V^\pi - \widetilde{V}^\pi}_\infty
    \end{split}
    \\
    \begin{split}
    & \leq \frac{c_R^{-1}}{\xi(\Phi(\mathbf{s}))}\int\limits_{\mathbf{z} \in \Phi(\mathbf{s})} c_R \left | r^\pi(\mathbf{s}) - r^\pi(\mathbf{z}) \right| + c_T \left|~ \int\limits_{\mathbf{s}^\prime \in \mathcal{S}}  (\mathcal{P}^\pi(\mathbf{s}^\prime|\mathbf{s}) - \mathcal{P}^\pi(\mathbf{s}^\prime|\mathbf{z}))\frac{c_R\gamma}{c_T}V^\pi(\mathbf{s}^\prime) d\mathbf{s}^\prime \right| d\xi(\mathbf{z}) \\
    &\qquad + \gamma \norm{V^\pi - \widetilde{V}^\pi}_\infty,
    \end{split}
\end{align*}
where $\norm{\cdot}_\infty$ is the supremum norm over $\mathcal{S}$. Due to Lemma \ref{lemma:p-wasserstein}, $c_R V(\mathbf{s})$ is a 1-Lipschitz function with respect to the distance function $d_\pi$. Then, since $\gamma \leq c_T$ by assumption, $\frac{c_R \gamma}{c_T}V(\mathbf{s})$ is also 1-Lipschitz. Using the dual form of the $W_1$ metric:
\begin{align*}
\begin{split}
    |V^\pi(\mathbf{s}) - \widetilde{V}^\pi(\Phi(\mathbf{s}))| &\leq \frac{c_R^{-1}}{\xi(\Phi(\mathbf{s}))}\int\limits_{\mathbf{z} \in \Phi(\mathbf{s})} c_R \left | r^\pi(\mathbf{s}) - r^\pi(\mathbf{z}) \right| + c_T W_1(d_\pi) (\mathcal{P}^\pi(\mathbf{s}^\prime|\mathbf{s}), \mathcal{P}^\pi(\mathbf{s}^\prime|\mathbf{z}))d\xi(\mathbf{z}) \\
    &\qquad + \gamma \norm{V^\pi - \widetilde{V}^\pi}_\infty \\
    & \leq \frac{c_R^{-1}}{\xi(\Phi(\mathbf{s}))}\int\limits_{\mathbf{z} \in \Phi(\mathbf{s})}d_\pi(\mathbf{s}, \mathbf{z})d\xi(\mathbf{z}) + \gamma \norm{V^\pi - \widetilde{V}^\pi}_\infty \\
    & \leq c_R^{-1} 2\epsilon + \gamma\norm{V^\pi - \widetilde{V}^\pi}_\infty.
\end{split}
\end{align*}
Thus, taking the supremum on the LHS over the state space $\mathcal{S}$:
\begin{align*}
    |V^\pi(\mathbf{s}) &- \widetilde{V}^\pi(\Phi(\mathbf{s}))| \leq \frac{2\epsilon}{c_R(1-\gamma)}, ~\forall \mathbf{s} \in \mathcal{S}.
\end{align*}
\end{proof}

\generalizedvfabound*
\begin{proof}
Due to Lemma \ref{lemma:on-policy-vfa},
\begin{align}
\label{eq:we-know}
    |V_{\gamma^\prime}^\pi(\mathbf{s}) - \widetilde{V}_{\gamma^\prime}^\pi(\Phi(\mathbf{s}))| \leq \frac{2\epsilon}{c_R(1-\gamma^\prime)},
\end{align}
where $V_{\gamma^\prime}^\pi$ denotes the expected value under a policy $\pi$ for a discount factor $\gamma^\prime$. Then, for $c_T < \gamma$:
\begin{align*}
    |V^\pi(\mathbf{s}) - \widetilde{V}^\pi(\Phi(\mathbf{s}))| &= |V^\pi(\mathbf{s}) - \widetilde{V}^\pi(\Phi(\mathbf{s})) + \widetilde{V}_{c_T}^\pi(\Phi(\mathbf{s})) - \widetilde{V}_{c_T}^\pi(\Phi(\mathbf{s})) + V_{c_T}^\pi(\mathbf{s}) - V_{c_T}^\pi(\mathbf{s})| \\
    & \leq |V_{c_T}^\pi(\mathbf{s}) - \widetilde{V}_{c_T}^\pi(\Phi(\mathbf{s}))| ~+~  |V^\pi(\mathbf{s}) - V_{c_T}^\pi(\mathbf{s})| ~+~  |\widetilde{V}^\pi(\Phi(\mathbf{s})) - \widetilde{V}_{c_T}^\pi(\Phi(\mathbf{s}))| \\
    & \leq  \frac{2\epsilon}{c_R(1-c_T)} ~+~ |V^\pi(\mathbf{s}) - V_{c_T}^\pi(\mathbf{s})| ~+~  |\widetilde{V}^\pi(\Phi(\mathbf{s})) - \widetilde{V}_{c_T}^\pi(\Phi(\mathbf{s}))|\\
    & \leq \frac{2\epsilon}{c_R(1-c_T)} + \frac{2(\gamma - c_T)}{(1-\gamma)(1-c_T)},
\end{align*}
where the second and third inequalities are due to Eq. \eqref{eq:we-know} and Lemma \ref{lemma:diff-gamma} respectively. For $\gamma \leq c_T$, we recover Lemma \ref{lemma:on-policy-vfa}, hence, for all $c_T \in [0, 1)$:
\begin{align*}
    |V^\pi(\mathbf{s}) - \widetilde{V}^\pi(\Phi(\mathbf{s}))| ~\leq~ \frac{2\epsilon}{c_R(1-\min(c_T, \gamma))} + \frac{2(\gamma - \min(c_T, \gamma))}{(1-\gamma)(1-c_T)}, ~\forall\mathbf{s} \in \mathcal{S}.
\end{align*}
\end{proof}

\maxdiam*
\begin{proof}
This lemma is a slight generalization of the distance bounds given in Thm. 3.12 of \cite{Ferns2011Bisimulation}, and the proof follows similarly:
\begin{align*}
    {d}(\mathbf{s}_i, \mathbf{s}_j) &= \max_{\mathbf{a} \in \mathcal{A}}\left(c_R|R(\mathbf{s}_i, \mathbf{a}) - R(\mathbf{s}_j, \mathbf{a})| ~+~ c_T W_1({d})({\mathcal{P}}(\cdot | \mathbf{s}_i, \mathbf{a}), {\mathcal{P}}(\cdot | \mathbf{s}_j, \mathbf{a}))\right)\\
    & \leq c_R(R_{\mathrm{max}}-R_{\mathrm{min}}) + c_T \mathrm{diam}(\mathcal{S}; {d}), ~\forall(\mathbf{s}_i, \mathbf{s}_j) \in \mathcal{S} \times \mathcal{S},
\end{align*}
due to Lemma \ref{lemma:wass-diam} (upper bound as $p \rightarrow \infty$). Then,
\begin{align*}
    \mathrm{diam}(\mathcal{S}; {d}) & \leq c_R(R_{\mathrm{max}}-R_{\mathrm{min}}) + c_T \mathrm{diam}(\mathcal{S}; {d}) \\
    & \leq \frac{c_R}{1-c_T}(R_{\mathrm{max}}-R_{\mathrm{min}}).
\end{align*}
\end{proof}

\boundedness*
\begin{proof}
The existence proof is virtually identical to the proof of Remark \ref{remark:rem1}, except it replaces $\mathcal{P}$ with an approximate dynamics model $\widehat{\mathcal{P}}$. This is possible since $\mathcal{S}$ is compact by assumption such that $\mathrm{supp}(\widehat{\mathcal{P}}) \subseteq \mathcal{S}$ is also compact:
\begin{align*}
    &\mathcal{F}(d_\pi)(\mathbf{s}_i, \mathbf{s}_j) - \mathcal{F}(d_\pi^\prime)(\mathbf{s}_i, \mathbf{s}_j) \\
    &= c_T \left(W_1(d_\pi)(\widehat{\mathcal{P}}^\pi(\cdot | \mathbf{s}_i), \widehat{\mathcal{P}}^\pi(\cdot | \mathbf{s}_j)) -  W_1(d_\pi^\prime)(\widehat{\mathcal{P}}^\pi(\cdot | \mathbf{s}_i), \widehat{\mathcal{P}}^\pi(\cdot | \mathbf{s}_j)) \right)\\
    & = c_T \left(W_1(d_\pi-d_\pi^\prime+d_\pi^\prime)(\widehat{\mathcal{P}}^\pi(\cdot | \mathbf{s}_i), \widehat{\mathcal{P}}^\pi(\cdot | \mathbf{s}_j)) -  W_1(d_\pi^\prime)(\widehat{\mathcal{P}}^\pi(\cdot | \mathbf{s}_i), \widehat{\mathcal{P}}^\pi(\cdot | \mathbf{s}_j)) \right)\\
    & \leq c_T \left(W_1(\norm{d_\pi-d_\pi^\prime}_\infty+d_\pi^\prime)(\widehat{\mathcal{P}}^\pi(\cdot | \mathbf{s}_i), \widehat{\mathcal{P}}^\pi(\cdot | \mathbf{s}_j)) -  W_1(d_\pi^\prime)(\widehat{\mathcal{P}}^\pi(\cdot | \mathbf{s}_i), \widehat{\mathcal{P}}^\pi(\cdot | \mathbf{s}_j)) \right)\\
    & \leq c_T \left(\norm{d_\pi-d_\pi^\prime}_\infty + W_1(d_\pi^\prime)(\widehat{\mathcal{P}}^\pi(\cdot | \mathbf{s}_i), \widehat{\mathcal{P}}^\pi(\cdot | \mathbf{s}_j)) -  W_1(d_\pi^\prime)(\widehat{\mathcal{P}}^\pi(\cdot | \mathbf{s}_i), \widehat{\mathcal{P}}^\pi(\cdot | \mathbf{s}_j)) \right) \\
    &= c_T\norm{d_\pi-d_\pi^\prime}_\infty, ~\forall(\mathbf{s}_i, \mathbf{s}_j) \in \mathcal{S} \times \mathcal{S},
\end{align*}
which implies $\mathcal{F}$ is a $c_T$-contraction.
It remains to prove that the distance is bounded. First, note that due to Lemma \ref{lemma:wass-diam}:
\begin{align}
\mathrm{supp}(\widehat{\mathcal{P}}) \subseteq \mathcal{S} \Rightarrow \sup_{\mathbf{s}_i, \mathbf{s}_j \in \mathcal{S} \times \mathcal{S}} W_p(\widehat{d}_\pi)(\widehat{\mathcal{P}}^\pi(\cdot | \mathbf{s}_i), \widehat{\mathcal{P}}^\pi(\cdot | \mathbf{s}_j)) \leq \mathrm{diam}(\mathcal{S}; \widehat{d}_\pi), ~\forall p \geq 1.
\end{align}
Then, similarly to Lemma \ref{lemma:diam},
\begin{align*}
    \widehat{d}_\pi(\mathbf{s}_i, \mathbf{s}_j) &= c_R|r_i^\pi - r_j^\pi| ~+~ c_T W_p(\widehat{d}_\pi)(\widehat{\mathcal{P}}^\pi(\cdot | \mathbf{s}_i), \widehat{\mathcal{P}}^\pi(\cdot | \mathbf{s}_j))\\
    & \leq c_R(R_{\mathrm{max}}-R_{\mathrm{min}}) + c_T \mathrm{diam}(\mathcal{S}; \widehat{d}_\pi), ~\forall(\mathbf{s}_i, \mathbf{s}_j) \in \mathcal{S} \times \mathcal{S},
\end{align*}
which implies,
\begin{align*}
    \mathrm{diam}(\mathcal{S}; \widehat{d}_\pi) &\leq c_R(R_{\mathrm{max}}-R_{\mathrm{min}}) + c_T \mathrm{diam}(\mathcal{S}; \widehat{d}_\pi) \\
    & \leq \frac{c_R}{1-c_T}(R_{\mathrm{max}}-R_{\mathrm{min}}).
\end{align*}
\end{proof}

\onpolicydiam*

\begin{proof}
In the on-policy case, the bound in Lemma \ref{lemma:diam} can be much tighter depending on the policy:
\begin{align*}
    d_\pi(\mathbf{s}_i, \mathbf{s}_j) &= c_R|r_i^\pi - r_j^\pi| ~+~ c_T W_1(d_\pi)(\mathcal{P}^\pi(\cdot | \mathbf{s}_i), \mathcal{P}^\pi(\cdot | \mathbf{s}_j))\\
    & \leq c_R\sup|r_i^\pi - r_j^\pi| ~+~ c_T \mathrm{diam}(\mathcal{S}; d_\pi), ~\forall(\mathbf{s}_i, \mathbf{s}_j) \in \mathcal{S} \times \mathcal{S}.
\end{align*}
As before,
\begin{align*}
    \mathrm{diam}(\mathcal{S}; d_\pi) &\leq \frac{c_R}{1-c_T}\sup|r_i^\pi - r_j^\pi|.
\end{align*}
\end{proof}

\meanupperbound*
\begin{proof}
For point masses, $W_p(d)(\delta(\mathbf{s}_i), \delta(\mathbf{s}_j)) = d(\mathbf{s}_i, \mathbf{s}_j)$:
\begin{align*} 
% \label{eq:deterministic-bisim-metric}
d_\pi(\mathbf{s}_i, \mathbf{s}_j) = c_R|r^{\pi}_i - r^{\pi}_j| ~+~ c_T d_\pi(\mathbf{s}_i^\prime, \mathbf{s}_j^\prime).
\end{align*}
Simply taking an expectation under $\nu^\pi$, due to the stationarity assumption: 
\begin{align*} 
\label{eq:reward-dependence}
\mu^{\pi}_{bd} &= c_R\mu^{\pi}_{rd} + c_T \mu^{\pi}_{bd}\\
&= \frac{c_R}{1-c_T}\mu^{\pi}_{rd}.
\end{align*}
\end{proof} 

\section{Notes on Reward Scale, $c_R$ and $c_T$}\label{sec:reward_scale}
In recent work \cite{castro2020scalable, gelada2019deepmdp, zhang2021invariant}, various forms of bisimulation metrics have been presented with different scaling constants; $(c_R=1-c, c_T=c)$ as in Definition \ref{def:bisim2}, and $(c_R=1, c_T=\gamma)$ as in Definition \ref{def:on-policy}. Here, we aim to add clarity to the effect of these choices and how they relate to the reward scale. First, note that due to Lemma \ref{lemma:diam}, setting $c_R=1-c_T$ serves to ensure that $d \in [0, 1]$ when the reward range is specified as $(R_{\mathrm{max}}=1, R_{\mathrm{min}}=0)$ as in \cite{Ferns2004Metrics, Ferns2011Bisimulation}.
\begin{corollary}[Policy-independent mean and variance bounds]
\label{corollary:meanvar}
Due to Lemma \ref{lemma:diam}, the first two moments of the random variable given by bisimulation distance have the following bounds independently of any policy $\pi$:\\
\begin{minipage}{0.5\textwidth}
  \begin{align}
      \mu_{bd} \leq \frac{c_R(R_{\mathrm{max}} - R_{\mathrm{min}})}{2(1 - c_T)},
    \end{align}
\end{minipage}
\begin{minipage}{0.50\textwidth}
    \begin{align}
    \label{eq:policy-indepdendent-variance-bound}
      \sigma_{bd}^2 \leq \frac{c_R^2(R_{\mathrm{max}} - R_{\mathrm{min}})^2}{4(1 - c_T)^2}.
    \end{align}
\end{minipage}
\end{corollary}
Conversely, if $(c_R=1, c_T = \gamma)$ as in Zhang et al. \cite{zhang2021invariant}, the formulation allows large distances between embeddings since $\gamma$ is commonly set to a value close to $1$. Large pairwise distances imply large norms and high variance (see Corollary \ref{corollary:meanvar}), which may cause instabilities in optimization (especially in the absence of norm constraints), considering the compactness conditions discussed in Sec. \ref{sec:lipschitz-forward}. 
\begin{definition}[Variance of distances and reward differences] 
Given a stationary distribution $\rho^\pi$ over states, and $\nu^\pi$ the distribution over pairs of states, $(\mathbf{s}_i, \mathbf{s}_j)$ sampled independently from $\rho^\pi$:
\begin{align}
    (\sigma_{bd}^\pi)^{2} &\coloneqq \mathbb{V}_{(\mathbf{s}_i, \mathbf{s}_j) \sim \nu^\pi}[d^\pi(\mathbf{s}_i, \mathbf{s}_j)] &&
    (\sigma_{rd}^\pi)^{2} \coloneqq \mathbb{V}_{(\mathbf{s}_i, \mathbf{s}_j) \sim \nu^\pi}[|r_i^\pi - r_j^\pi|].
\end{align}
\end{definition}
\begin{restatable}[On-policy variance bound]{proposition}{varupperbound}
\label{prop:prop4}
Assume a deterministic environment. Given $c_T \in [0, \sqrt{0.5})$, the variance of the optimal on-policy bisimulation distance for an objective of the form Eq. \eqref{eq:general-bisim-loss} can be upper-bounded as follows:
\begin{align}
(\sigma_{bd}^\pi)^{2} \leq \frac{2c_R^2}{1 - 2c_T^2} (\sigma_{rd}^\pi)^{2} + \frac{c_R^2\left(1 - 2c_T\right)^2}{\left(1 - 2c_T^2\right)\left(1-c_T\right)^2}(\mu_{rd}^\pi)^2,
\end{align}
while for all $c_T \in [0, 1)$, the bound in Eq. \eqref{eq:policy-indepdendent-variance-bound} applies.
\end{restatable}
\begin{proof}
Given a deterministic environment:
\begin{align*} 
d_\pi(\mathbf{s}_i, \mathbf{s}_j) = c_R|r^{\pi}_i - r^{\pi}_j| ~+~ c_T d_\pi(\mathbf{s}_i^\prime, \mathbf{s}_j^\prime).
\end{align*}
Then,
\begin{align*} 
\label{eq:variance-inequality2}
(\sigma_{bd}^\pi)^{2} &= \mathbb{E}_{(\mathbf{s}_i, \mathbf{s}_j) \sim \nu^\pi}[d_\pi(\mathbf{s}_i, \mathbf{s}_j)^2] - (\mu_{bd}^\pi)^{2}\\
&\leq 2c_R^2 ((\sigma_{rd}^\pi)^{2}+(\mu_{rd}^\pi)^{2}) +2 c_T^2 ((\sigma_{bd}^\pi)^{2} + (\mu_{bd}^\pi)^{2}) - (\mu_{bd}^\pi)^{2}.
\end{align*}
When $c_T \geq \sqrt{0.5}$ (as in Zhang et al. \cite{zhang2021invariant}), the above bound is loose. However, $c_T < \sqrt{0.5}$ provides a convenient upper bound:
\begin{align*}
(\sigma_{bd}^\pi)^{2} &\leq \frac{2c_R^2}{1 - 2c_T^2} ((\sigma_{rd}^\pi)^{2}+(\mu_{rd}^\pi)^{2}) - (\mu_{bd}^\pi)^{2} \\
& =  \frac{2c_R^2}{1 - 2c_T^2} (\sigma_{rd}^\pi)^{2} + \frac{c_R^2\left(1 - 2c_T\right)^2}{\left(1 - 2c_T^2\right)\left(1-c_T\right)^2}(\mu_{rd}^\pi)^{2},
\end{align*}
where the equality is due to Lemma \ref{lemma:mean-lemma}.
\end{proof}

\begin{figure}[h]
    \centering
     \includegraphics[width=0.49\textwidth]{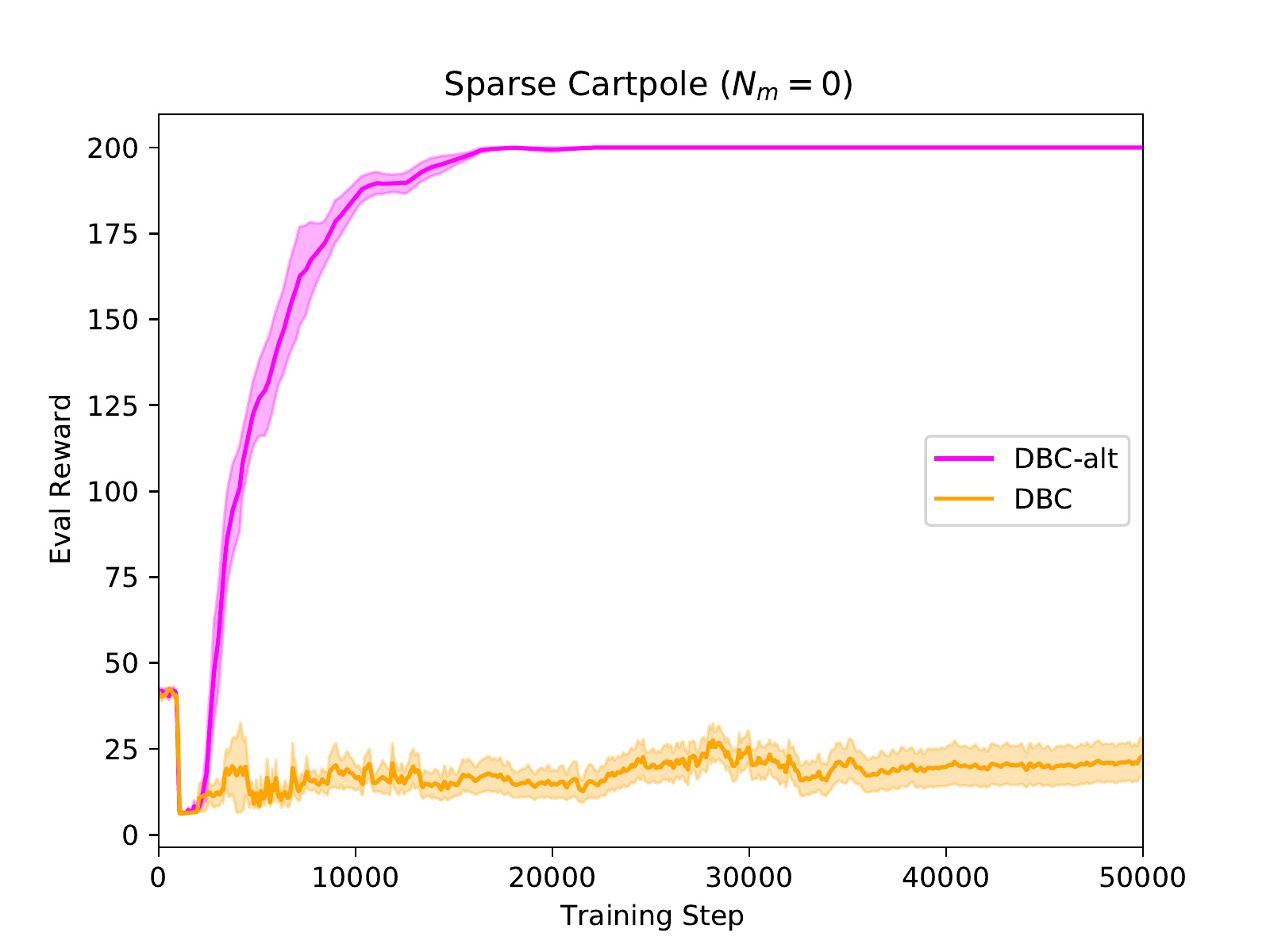}
    \hspace{\fill}
  \includegraphics[width=0.49\linewidth]{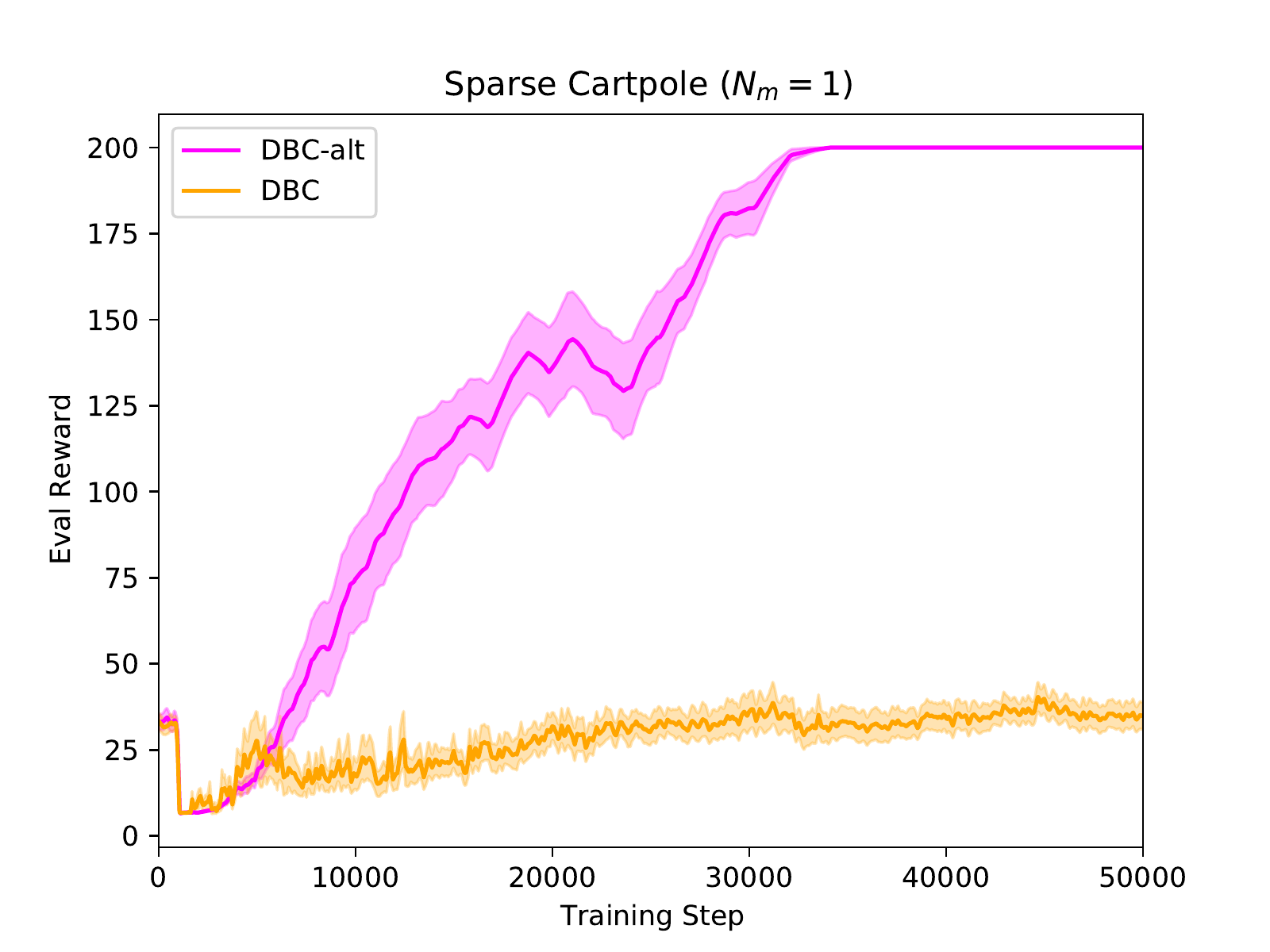} 
    \caption{Performance on the Sparse Cartpole task, comparing the standard DBC algorithm (using $c_R = 1.0$ and $c_T = \gamma$) with an alternative weighting formulation (denoted ``DBC-alt''), where $c_R = 0.5$ and $c_T = 0.5$. 
    For this task, the alternative weighting formulation is much more robust to reward sparsity.
    Shaded bars are standard errors over 10 seeds.}
    \label{fig:dbc-alt-results}
\end{figure}

Tighter bounds on the variance of the on-policy bisimulation metric may be important, since the statistics of $r^\pi$ undergo change throughout training between policy updates. Hence, it is desirable to remove the dependence of the bound from the $\mu_{rd}^\pi$ term with a choice of $c_T=0.5$, such that tighter bounds can be obtained. This choice renders the formulation more robust to changes in the scale of the expected rewards. The resulting bound for $c_T=0.5$ has a simpler form:
\begin{align}
    (\sigma_{bd}^\pi)^{2} \leq 4c_R^2(\sigma_{rd}^\pi)^{2} \leq c_R^2(R_{\mathrm{max}} - R_{\mathrm{min}})^2.
\end{align}
Indeed, in Figure \ref{fig:dbc-alt-results}, we show that such a choice can stabilize the DBC \cite{zhang2021invariant} algorithm significantly, resulting in higher overall performance.

\section{Value Bounds with Model Error}
\label{sec:vfa-model-error}

Our goal in this section is to characterize the errors induced by using approximate dynamics and an imperfect encoder, with respect to estimating both the ground-truth on-policy bisimulation metric, as well as preserving the value function with the encoded state. 

For this section, we first remark on the difference between three forms of the PBSM, for some fixed policy $\pi$:
\begin{itemize}
    \item 
    $d_\pi(\mathbf{s}_i, \mathbf{s}_j)$ is the \textit{ground-truth} bisimulation metric, defined as the fixed point of the operator defined in Assumption \ref{assum:p-wass} (when it exists uniquely).
    
    \item 
    $\widehat{d}_\pi(\mathbf{s}_i, \mathbf{s}_j;  \widehat{\mathcal{P}})$ 
    is the fixed point of the same operator but with \textit{approximate dynamics} 
    (i.e., using $ \widehat{\mathcal{P}} $ and $\widehat{r}^\pi$ 
    instead of the true $ \mathcal{P}$ and $r^\pi$; we leave out conditioning on $\widehat{r}^\pi$ for brevity). 
    
    \item 
    $\widehat{d}_{\pi,\phi}(\mathbf{s}_i, \mathbf{s}_j) := || \phi(\mathbf{s}_i) - \phi(\mathbf{s}_j) ||_q $ is a non-negative function of states, dependent on an encoder function $\phi$. For example, $\phi$ may be produced by a metric learning process, such as by stochastic minimization of the objective in Eq.\ \ref{eq:general-bisim-loss}.
    
\end{itemize}

Notice that, if we attempt to learn $\phi$ with a metric learning process based on approximate dynamics, then the best we can do is obtain 
$ \widehat{d}_{\pi,\phi} \rightarrow \widehat{d}_\pi $, 
which can still have some irreducible error.
On the other hand, even with perfect approximate dynamics, the metric learning process may be incomplete, meaning $ \widehat{d}_{\pi,\phi} \ne \widehat{d}_\pi $.
We therefore hope to characterize the error into two types: dynamics approximation error and metric learning error.

Next, we define three types of model errors, relating to the quality of encoding and forward dynamics prediction (in terms of state and reward).

\begin{restatable}[Model and Encoder Errors]{definition}{model-error}
\label{prop:modelerror}
The bisimulation distance approximation error $\mathcal{E}_{\phi}$,
transition probability error $\mathcal{E}_{\mathcal{P}}$,
and
reward prediction error $\mathcal{E}_{r}$
are given by
\begin{align}
    \mathcal{E}_{\phi} &:= \norm{\widehat{d}_{\pi,\phi} - \widehat{d}_\pi}_\infty \label{eqdef:ephi} \\
    \mathcal{E}_{\mathcal{P}} &:= \sup_{\mathbf{s} \in\mathcal{S}} W_p(d_\pi)(\mathcal{P}^\pi(\cdot|\mathbf{s}), \widehat{\mathcal{P}}^\pi(\cdot|\mathbf{s}))  \\
    \mathcal{E}_{r} &:= \norm{\widehat{r}^\pi - {r}^\pi}_\infty 
\end{align}
where $||\cdot||_\infty$ is the supremum (or uniform) norm over states and
$\widehat{d}_\pi$ is the fixed-point bisimulation metric with the $W_p$ distance defined by using $\widehat{\mathcal{P}}^\pi$ and $\widehat{r}^\pi$, instead of the true dynamics $\mathcal{P}^\pi$ and ${r}^\pi$.%
\footnote{Firstly, note that for stochastic reward signals and/or policies, this is a difference between expectations, meaning that there could be sampling noise. Otherwise, for the DBC use case, the observed reward collected by the agent is used for training (meaning there will be zero modelling error for the reward term when computing $\widehat{d}_\pi$ between observed states). 
Nevertheless, given a reward model, $\widehat{d}_{\pi,\phi}$ and $ \widehat{d}_{\pi}$ can still be queried for states where the ground truth reward (or reward distribution) may not be known.  } 
\end{restatable}

Note that we have defined our forward dynamics model errors here irrespective of $\phi$ (i.e., it may be used by $\widehat{\mathcal{P}}$ and/or $\widehat{r}^\pi$, or not).
Our goal is to use these errors to bound the difference in the MDP value function induced by the approximate nature of the environmental model and state encoder.
First, we consider how the model errors affect the optimal policy-dependent bisimulation distance that we can obtain given our approximate forward dynamics.
\begin{restatable}[Bisimulation distance error]{lemma}{distance-error}
\label{lemma:distanceerror}
Let $c_T \in [0,1)$ and $c_R \geq 0$. Assume $\mathrm{supp}(\widehat{\mathcal{P}}) \subseteq \mathcal{S}$
and $1 - c_Ta_p > 0$. 
Then
\begin{equation}
\norm{d_\pi  -  \widehat{d}_\pi}_\infty \leq 
\frac{2c_R}{1 - c_T a_p} \mathcal{E}_{r} + \frac{2c_T}{1 - c_T a_p} \mathcal{E}_{\mathcal{P}} + 
\frac{ c_T[a_p - 1] }{1 - c_Ta_p}
\mathrm{diam}(\mathcal{S};{d}_\pi)
\end{equation}
where $a_p = 2^{(p-1)/p}$ and 
$\mathrm{diam}(\mathcal{S};{d}_\pi) \leq  \frac{c_R}{1-c_T}(R_{\mathrm{max}}-R_{\mathrm{min}})$ by Theorem \ref{thm:boundedness}.
\end{restatable}
\begin{proof}
We can first use the triangle inequality to bound the difference between reward distances:
\begin{align*}
|r_i^\pi - r_j^\pi| 
&\leq |{r}_i^\pi - \widehat{r}_i^\pi| + |{r}_j^\pi - \widehat{r}_i^\pi| 
\leq \mathcal{E}_{r} + \mathcal{E}_{r} + |\widehat{r}_i^\pi - \widehat{r}_j^\pi|,
\end{align*}
so that
$|r_i^\pi - r_j^\pi|  - |\widehat{r}_i^\pi - \widehat{r}_j^\pi|
\leq \mathcal{E}_{r} + \mathcal{E}_{r}$. 
Symmetrically, we can show that 
$|\widehat{r}_i^\pi - \widehat{r}_j^\pi| - |r_i^\pi - r_j^\pi| 
\leq 2\mathcal{E}_{r} $ as well.
For notational clarity, 
let $a_p = 2^{(p-1)/p}$ and
$ W_p(d_\pi, \mathcal{P}^\pi) :=  W_p(d_\pi)(\mathcal{P}^\pi(\cdot|\mathbf{s}_i), {\mathcal{P}}^\pi(\cdot|\mathbf{s}_j)) $, as well as similarly for 
$ W_p(d_\pi, \widehat{\mathcal{P}}^\pi)$ 
and 
$ W_p( \widehat{d}_\pi, \widehat{\mathcal{P}}^\pi) $.

First, by the Wasserstein triangle inequality  \cite{clement2008elementary}, as for the reward difference:
\begin{equation}
|W_p(d_\pi,\widehat{\mathcal{P}}^\pi)
 - W_p(d_\pi,\mathcal{P}^\pi)| 
\leq 2\mathcal{E}_{\mathcal{P}}.
\label{eq:supp:9:triineq}
\end{equation} 

Second, the convexity of $d^p$ implies that,
\begin{align*}
    W_p( || d_\pi - \widehat{d}_\pi ||_\infty + d_\pi,\widehat{\mathcal{P}}^\pi )
    &= \left(\inf_{\omega \in \Omega}\mathbb{E}_{(\mathbf{s}_i, \mathbf{s}_j) \sim \omega}[( || d_\pi - \widehat{d}_\pi ||_\infty + d_\pi(\mathbf{s}_i, \mathbf{s}_j))^p]  \right)^{\frac{1}{p}} \\
    &\leq \left(\inf_{\omega \in \Omega} 2^{p-1} \mathbb{E}_{(\mathbf{s}_i, \mathbf{s}_j) \sim \omega}[( || d_\pi - \widehat{d}_\pi ||_\infty^p + d_\pi(\mathbf{s}_i, \mathbf{s}_j)^p]  \right)^{\frac{1}{p}} \\
    &\leq a_p \left( || d_\pi - \widehat{d}_\pi ||_\infty^p +  
    W_p^p( d_\pi, \widehat{\mathcal{P}}^\pi) \right)^{\frac{1}{p}} \\
    &\leq
    a_p \left( \left[|| d_\pi - \widehat{d}_\pi ||_\infty + 
    W_p( d_\pi , \widehat{\mathcal{P}}^\pi) \right]^p \right)^{1/p} \\
    &= a_p \left(|| d_\pi - \widehat{d}_\pi ||_\infty + W_p( d_\pi , \widehat{\mathcal{P}}^\pi) \right).
    \numberthis \label{eq:supp:9:convexity}
\end{align*}
Third, 
recall that when $\mathrm{supp}(\widehat{\mathcal{P}}) \subseteq \mathcal{S}$, due to Lemma \ref{lemma:wass-diam}, we have:
\begin{equation}
    W_p(d_\pi,\widehat{\mathcal{P}}^\pi)
    \leq 
    \mathrm{diam}(\mathcal{S};{d}_\pi). 
    \numberthis \label{eq:supp:9:Wdiam}
\end{equation}

Then, the difference in distances can be bounded by:
\begin{align*}
    \bigl| 
        &W_p(d_\pi, \mathcal{P}^\pi)
        -
        W_p( \widehat{d}_\pi, \widehat{\mathcal{P}}^\pi)
    \bigr|\\
    &\leq 
        \bigl| 
            W_p( \widehat{d}_\pi, \widehat{\mathcal{P}}^\pi)
            - 
            W_p( {d}_\pi, \widehat{\mathcal{P}}^\pi)
            \bigr|
        + 
        \bigl| W_p( {d}_\pi, {\mathcal{P}}^\pi)
            - W_p( {d}_\pi, \widehat{\mathcal{P}}^\pi)
            \bigr| \\
    &\leq \bigl| 
        W_p( \widehat{d}_\pi, \widehat{\mathcal{P}}^\pi)
                - 
                W_p( {d}_\pi, \widehat{\mathcal{P}}^\pi)
                \bigr| +
            2 \mathcal{E}_{\mathcal{P}}  && \text{By Eq.\ \ref{eq:supp:9:triineq}} \\
&=
\bigl| 
        W_p( \widehat{d}_\pi - d_\pi + d_\pi, \widehat{\mathcal{P}}^\pi)
                - 
                W_p( {d}_\pi, \widehat{\mathcal{P}}^\pi)
                \bigr| +
            2 \mathcal{E}_{\mathcal{P}} \\
&\leq
\bigl| 
        W_p( ||\widehat{d}_\pi - d_\pi||_\infty + d_\pi, \widehat{\mathcal{P}}^\pi)
                - 
                W_p( {d}_\pi, \widehat{\mathcal{P}}^\pi)
                \bigr| +
            2 \mathcal{E}_{\mathcal{P}} \\
&=
\bigl| 
        W_p( || d_\pi - \widehat{d}_\pi ||_\infty + d_\pi, \widehat{\mathcal{P}}^\pi)
                - 
                W_p( {d}_\pi, \widehat{\mathcal{P}}^\pi)
                \bigr| +
            2 \mathcal{E}_{\mathcal{P}} \\
&\leq 
\bigl| 
        a_p || d_\pi - \widehat{d}_\pi ||_\infty + 
        a_p W_p( {d}_\pi, \widehat{\mathcal{P}}^\pi)
                - 
                W_p( {d}_\pi, \widehat{\mathcal{P}}^\pi)
                \bigr| +
            2 \mathcal{E}_{\mathcal{P}} && \text{By Eq.\ \ref{eq:supp:9:convexity}} 
\\
&\leq %
 a_p || d_\pi - \widehat{d}_\pi ||_\infty  + [a_p - 1] \mathrm{diam}(\mathcal{S};{d}_\pi) 
+ 2 \mathcal{E}_{\mathcal{P}}. && \text{By Eq.\ \ref{eq:supp:9:Wdiam}}   
\end{align*}

We can then plug these into the difference between the true and approximate policy-dependent bisimulation distances:
\begin{align*}
    |d_\pi(\mathbf{s}_i, \mathbf{s}_j)  -  \widehat{d}_\pi(\mathbf{s}_i, \mathbf{s}_j)|
    &\leq c_R \left| |r_i^\pi - r_j^\pi| - |\widehat{r}_i^\pi - \widehat{r}_j^\pi| \right|
         + c_T\left| W_p(d_\pi,\mathcal{P}^\pi) - W_p(\widehat{d}_\pi,\widehat{\mathcal{P}}^\pi) \right| \\
    &\leq 2c_R \mathcal{E}_{r} 
        + c_T\left| a_p || d_\pi - \widehat{d}_\pi ||_\infty  + [a_p - 1] \mathrm{diam}(\mathcal{S};{d}_\pi)  + 2 \mathcal{E}_{\mathcal{P}} \right| \\
|| d_\pi - \widehat{d}_\pi ||_\infty    
&\leq 
2c_R \mathcal{E}_{r} + {2c_T} \mathcal{E}_{\mathcal{P}} + 
c_T a_p || d_\pi - \widehat{d}_\pi ||_\infty + c_T[a_p - 1] \mathrm{diam}(\mathcal{S};{d}_\pi) \\
|| d_\pi - \widehat{d}_\pi ||_\infty    
&\leq 
\frac{2c_R}{1 - c_T a_p} \mathcal{E}_{r} + \frac{2c_T}{1 - c_T a_p} \mathcal{E}_{\mathcal{P}} + 
\frac{ c_T[a_p - 1] }{1 - c_Ta_p} \mathrm{diam}(\mathcal{S};{d}_\pi)
\end{align*}
where the second-last inequality follows by taking the supremum over states for both sides. 
\end{proof}

For the remainder of this section, we assume $p=1$. 

\begin{restatable}[Bisimulation distance error with $p=1$]{corollary}{distance-error-peq1}
\label{corollary:distanceerrorpequal1}
Let $p=1$, with the remaining conditions as in Lemma \ref{lemma:distanceerror}. Then
\begin{equation}
\norm{d_\pi - \widehat{d}_\pi}_\infty
\leq 
\frac{2c_R}{1-c_T} \mathcal{E}_{r} + \frac{2c_T}{1-c_T} \mathcal{E}_{\mathcal{P}}.
\end{equation}
\end{restatable}
\begin{proof}
When $p=1$, we have $a_p = a_1 = 1$, giving the expression above.
\end{proof}

This bounds the error between the true on-policy bisimulation distance and the optimal \textit{approximate} bisimulation distance (i.e., the best distance function we could hope to achieve with our encoder, given the error in our forward dynamics model).
However, ultimately, we wish to bound the error in the value function in terms of $\widehat{d}_{\pi,\phi}$, not just $\widehat{d}_\pi$ (to take the error of the encoder $\phi$ into account, as well as that of the dynamics model).
First, we can bound the true bisimulation distance in terms of the encoder and model error as follows:

\begin{restatable}[Bound on bisimulation distance with encoder error]{lemma}{bisim-bound-enc-err}
\label{lemma:bisimboundencerr}
Consider the same conditions as Corollary \ref{corollary:distanceerrorpequal1}. Then
\begin{equation}
\norm{d_\pi - \widehat{d}_{\pi,\phi}}_\infty
\leq 
\mathcal{E}_\phi + \frac{2c_R}{1-c_T} \mathcal{E}_{r} + \frac{2c_T}{1-c_T} \mathcal{E}_{\mathcal{P}}.
\end{equation}
\end{restatable}
\begin{proof}
\begin{align*}
    \norm{d_\pi - \widehat{d}_{\pi,\phi}}_\infty
    &= 
    \norm{d_\pi - \widehat{d}_{\pi,\phi} - \widehat{d}_\pi + \widehat{d}_\pi}_\infty \\
    &\leq 
      \norm{d_\pi - 
        \widehat{d}_\pi}_\infty
    + \norm{\widehat{d}_{\pi,\phi} - 
        \widehat{d}_\pi}_\infty \\
    &\leq \frac{2c_R}{1-c_T} \mathcal{E}_{r} + \frac{2c_T}{1-c_T} \mathcal{E}_{\mathcal{P}}
        + \mathcal{E}_\phi 
\end{align*}
using Corollary \ref{corollary:distanceerrorpequal1} and Equation \ref{eqdef:ephi}.
\end{proof}
Thus, if we can relate $d_\pi$ to the value function, we can also do so for $\widehat{d}_{\pi,\phi}$, as a function of model error.

Finally, we look at bounding the difference in the state value function, using the \textit{approximate} bisimulation distance defined through the learned encoder (i.e., our partitioning $Z$ is defined via $ \widehat{d}_{\pi,\phi} $).
Let $\widehat{\epsilon}$ be the aggregation radius in $\phi$-space 
(meaning the maximum diameter with respect to $\widehat{d}_{\pi,\phi}$ per partition subset, or equivalence class, is at most $2\,\widehat{\epsilon}\,$):
\begin{equation*}
    \sup_{\mathbf{z}\in Z} \sup_{\mathbf{s}_i, \mathbf{s}_j \in \mathbf{z}}
    || \phi(\mathbf{s}_i) - \phi(\mathbf{s}_j) ||_q ~\leq~ 2 \widehat{\epsilon}.
\end{equation*}
Notice that $\widehat{\epsilon}$ bounds the maximal diameter of the partition cells with respect to the \textit{learned} metric, using $\phi$, rather than the ground truth bisimulation distance.

\valuemodelerror*
\begin{proof}

Following the proof of Lemma \ref{lemma:on-policy-vfa}, we have that
\footnote{Notice that using the Equation $(\dag)$ allows us to recover the value bound with model error from \cite{zhang2021invariant}: $ |V^\pi(\mathbf{s}) - \widetilde{V}^\pi(\Phi(\mathbf{s}))| \leq \frac{2\widehat{\varepsilon} + \mathcal{L}}{(1-\gamma)(1-c)} $, where $c_R = 1-c$. } %
\begin{align*}
    (1-\gamma)|V^\pi(\mathbf{s}) - \widetilde{V}^\pi(\Phi(\mathbf{s}))| 
    & \leq \frac{c_R^{-1}}{\xi(\Phi(\mathbf{s}))}\int\limits_{\mathbf{z} \in \Phi(\mathbf{s})}d_\pi(\mathbf{s}, \mathbf{z})d\xi(\mathbf{z}) \\
    & \leq \frac{c_R^{-1}}{\xi(\Phi(\mathbf{s}))}\int\limits_{\mathbf{z} \in \Phi(\mathbf{s})} 
    \widehat{d}_{\pi,\phi}(\mathbf{s}, \mathbf{z}) 
    + \underbrace{| d_\pi(\mathbf{s},\mathbf{z}) - 
            \widehat{d}_{\pi,\phi}(\mathbf{s}, \mathbf{z}) |_\infty}_{\mathcal{L}}
    d\xi(\mathbf{z}) \\
    & \leq \frac{c_R^{-1}}{\xi(\Phi(\mathbf{s}))}\int\limits_{\mathbf{z} \in \Phi(\mathbf{s})} 
    2\widehat{\epsilon} 
    + {\mathcal{L}}\,
    d\xi(\mathbf{z}) \\
    &= c_R^{-1}( 2\widehat{\epsilon} + \mathcal{L} ) \tag{\dag} \\
    &\leq 
    \frac{1}{c_R}\left( 2\widehat{\epsilon} + \mathcal{E}_\phi + \frac{2c_R}{1-c_T} \mathcal{E}_{r} + \frac{2c_T}{1-c_T} \mathcal{E}_{\mathcal{P}} \right)
\end{align*}
where the last line used Lemma \ref{lemma:bisimboundencerr}.

\end{proof}

Rather than bound the ground-truth on-policy bisimulation distance $d_\pi$, we have instead bound the \textit{estimated} distance $\widehat{d}_{\pi,\phi}$, using both approximate predictive dynamics and error in the metric learning process.

Notice that as we shrink the size of the equivalence classes in the partition, 
    so $\widehat{\epsilon} \rightarrow 0$, 
    we get $\Phi \rightarrow \phi$.
This tells us that information about the value of a state is preserved by the encoder $\phi$, as long as the error in the forward dynamics model and metric learning process is small.
Further, in the low error case, if the vectors $\phi(\mathbf{s}_i)$ and $\phi(\mathbf{s}_j)$ are close, we are guaranteed they have similar value under $\pi$, with the difference growing only linearly with the error.

\begin{restatable}[VFA bound in terms of model error for arbitrary $c_T$]{corollary}{vfamodelerrorct}
\label{crl:vfamodelerrorct}
Consider the same conditions and definitions as Thm. \ref{theorem:valboundmodelerror}, except $c_T \in [0, 1)$. Let $\overline{\gamma} = \min(c_T, \gamma)$:
\begin{equation}
| V^\pi(\mathbf{s}) - \widetilde{V}^\pi( \Phi( \mathbf{s} ) ) |
\leq 
\frac{1}{c_R(1 - \overline{\gamma})} \left( 
    2 \,\widehat{\epsilon} + 
    \mathcal{E}_\phi + 
    \frac{2c_R}{1 - c_T}\mathcal{E}_r + 
    \frac{2c_T}{1 - c_T}\mathcal{E}_{\mathcal{P}} \right) + \frac{2(\gamma - \overline{\gamma})}{(1-\gamma)(1-c_T)}, ~\forall\mathbf{s} \in \mathcal{S}.
\end{equation}
\end{restatable}
\begin{proof}
The proof follows similarly to the proof of Thm. \ref{thm:generalized-ct-vfa}. Suppose $c_T < \gamma$:
\begin{align*}
    |V^\pi(\mathbf{s}) - \widetilde{V}^\pi(\Phi(\mathbf{s}))| &= |V^\pi(\mathbf{s}) - \widetilde{V}^\pi(\Phi(\mathbf{s})) + \widetilde{V}_{c_T}^\pi(\Phi(\mathbf{s})) - \widetilde{V}_{c_T}^\pi(\Phi(\mathbf{s})) + V_{c_T}^\pi(\mathbf{s}) - V_{c_T}^\pi(\mathbf{s})| \\
    & \leq |V_{c_T}^\pi(\mathbf{s}) - \widetilde{V}_{c_T}^\pi(\Phi(\mathbf{s}))| ~+~  |V^\pi(\mathbf{s}) - V_{c_T}^\pi(\mathbf{s})| ~+~  |\widetilde{V}^\pi(\Phi(\mathbf{s})) - \widetilde{V}_{c_T}^\pi(\Phi(\mathbf{s}))| \\
    & \leq  \frac{1}{c_R(1 - c_T)} \left(
    2 \,\widehat{\epsilon} + 
    \mathcal{E}_\phi + 
    \frac{2c_R}{1 - c_T}\mathcal{E}_r + 
    \frac{2c_T}{1 - c_T}\mathcal{E}_{\mathcal{P}} \right) \\ &~~~+~ |V^\pi(\mathbf{s}) - V_{c_T}^\pi(\mathbf{s})| ~+~  |\widetilde{V}^\pi(\Phi(\mathbf{s})) - \widetilde{V}_{c_T}^\pi(\Phi(\mathbf{s}))|\\
    & \leq \frac{1}{c_R(1 - c_T)} \left(
    2 \,\widehat{\epsilon} + 
    \mathcal{E}_\phi + 
    \frac{2c_R}{1 - c_T}\mathcal{E}_r + 
    \frac{2c_T}{1 - c_T}\mathcal{E}_{\mathcal{P}} \right) + \frac{2(\gamma - c_T)}{(1-\gamma)(1-c_T)},
\end{align*}
where the second and third inequalities are due to Thm. \ref{theorem:valboundmodelerror} and Lemma \ref{lemma:diff-gamma} respectively. For $\gamma \leq c_T$, we recover Thm. \ref{theorem:valboundmodelerror}, hence, for all $c_T \in [0, 1)$ and $\mathbf{s} \in \mathcal{S}$:
\begin{align*}
    |V^\pi(\mathbf{s}) - \widetilde{V}^\pi(\Phi(\mathbf{s}))| ~\leq~ \frac{1}{c_R(1 - \min(c_T, \gamma))} \left( 
    2 \,\widehat{\epsilon} + 
    \mathcal{E}_\phi + 
    \frac{2c_R}{1 - c_T}\mathcal{E}_r + 
    \frac{2c_T}{1 - c_T}\mathcal{E}_{\mathcal{P}} \right) + \frac{2(\gamma - \min(c_T, \gamma))}{(1-\gamma)(1-c_T)}.
\end{align*}
\end{proof}

\section{Experimental Details and Additional Results}
\label{sec:additional-results}

\subsection{OpenAI Gym}
\label{sec:additional-results:oagym}

\subsubsection{Sparse Noisy Cartpole Additional Details}
\label{sec:additional-results:cp}

First, we consider modifications to the Cartpole-v0 task. In the standard version, the agent receives a constant reward of $+1$ at each time-step for as long as the cart-pole system is upright between $[-\theta_{\mathrm{term}}, \theta_{\mathrm{term}}]$ degrees. If the pole falls below this range, the episode terminates. To make the task more challenging, we introduce a second parameter $\theta_{\mathrm{rew}} \ll \theta_{\mathrm{term}}$, such that the agent receives a reward only if the pole is between $[-\theta_{\mathrm{rew}}, \theta_{\mathrm{rew}}]$ degrees. We refer to this task as Sparse Cartpole and set $ \theta_{\mathrm{term}} = 12^{\circ} $ with $\theta_{\mathrm{rew}} = 0.01 \,  \theta_{\mathrm{term}}$. Additionally, we consider a noisy version of Sparse Cartpole to mimic distractors. In particular, we concatenate an $N_m \mathrm{dim}(\mathcal{S})$-dimensional vector sampled from an isotropic Gaussian to the state vector. The resulting task is referred to as Noisy Sparse Cartpole.
Thus, the encoder $\phi$ must learn to embed functionally similar states in close proximity, despite the distractions, and maintain a well-behaved embedding, despite reward sparsity. 
While we provide sparse reward signals at training time to make the learning problem more difficult, we evaluate the resulting models in the standard environment based on $\theta_{\mathrm{term}}$, since this provides a lower variance return signal and still allows us determine whether the task has been solved.

\subsubsection{Noisy Mountain Car Additional Details}
\label{sec:additional-results:mc}

We next tested on the Noisy Mountain Car task
\cite{moore1990efficient}, implemented as MountainCarContinuous-v0
in the OpenAI Gym \cite{brockman2016openai},
and modified to concatenate $N_m \mathrm{dim}(\mathcal{S})$ noise dimensions to the observed state to simulate distraction.
Briefly, the agent controls a car that should reach the top of a mountain, but has an engine of insufficient power to attain it directly.
It must therefore learn a sequence of actions that build enough momentum to complete the task. 
The reward signal is highly uninformative: a small negative reward is given at every step, unless the task is solved, in which case a large positive reward is given.
For $N_m > 0$, this task has both noisy distractors and high sparsity, making it rather challenging.

Note that only methods with intrinsic reward were able to solve the task (see Fig.\ \ref{fig:gym} and Fig.\ \ref{fig:gym-additional-plots}), and that
DBC without normalization was also unable to complete it.
We remark that all methods rely on the maximum policy entropy RL formulation, which is known to improve exploration \cite{haarnoja2018soft,haarnoja2018soft2,ziebart2010modeling};
nevertheless,
our results suggest that curiosity-driven exploration, 
induced by intrinsic rewards based on predictive error,
is at least complementary to such techniques.

\subsubsection{Sparse Pendulum Task Details}
\label{sec:additional-results:pd}
For the Pendulum-v0 task, we implement similar modifications to those in SparseCartpole. The standard Pendulum task starts with a pole in downright position, and provides negative rewards proportional to $|\theta_{\mathrm{pend}}|$, degrees away from upright position. Our SparsePendulum instead provides a reward of $+1$ only when the pendulum is between $[-\theta_{\mathrm{rew}}, \theta_{\mathrm{rew}}]$ degrees, where $\theta_{\mathrm{rew}}$, and does not provide a reward otherwise. NoisySparsePendulum similarly concatenates to state vectors a noise vector of $N_m$ times the original dimensionality. 
Note that we evaluate with an environment with 
$\theta_{\mathrm{rew}} = 1^{\circ}$, 
to reduce variance in the evaluation reward.

Results (see Fig.\ \ref{fig:pendulum}) show that our method performs comparably to DBC. In the presence of high distraction ($N_m=10$), 
DBC-based approaches can do much better than SAC, which is no longer able to solve the task.

\subsubsection{Plots with Additional Distraction}
\label{sec:additional-results:nm3}

In Fig.\ \ref{fig:gym-additional-plots}, we show results for Sparse Cartpole (left) and Mountain Car (right), with an even higher level of distraction ($N_m = 3$).
On Sparse Cartpole, only DBC-normed-IR and DBC-normed-IR-ID were able to solve the task, with the latter being slightly more stable, while DBC-normed performed significantly better than any methods without normalization.
On Mountain Car, all methods struggle to solve the task at this level of distraction; however, DBC-normed-IR-ID performs significantly better than the others, showing the utility of the inverse dynamics regularization.

\begin{figure}
    \centering
    \includegraphics[width=0.49\textwidth]{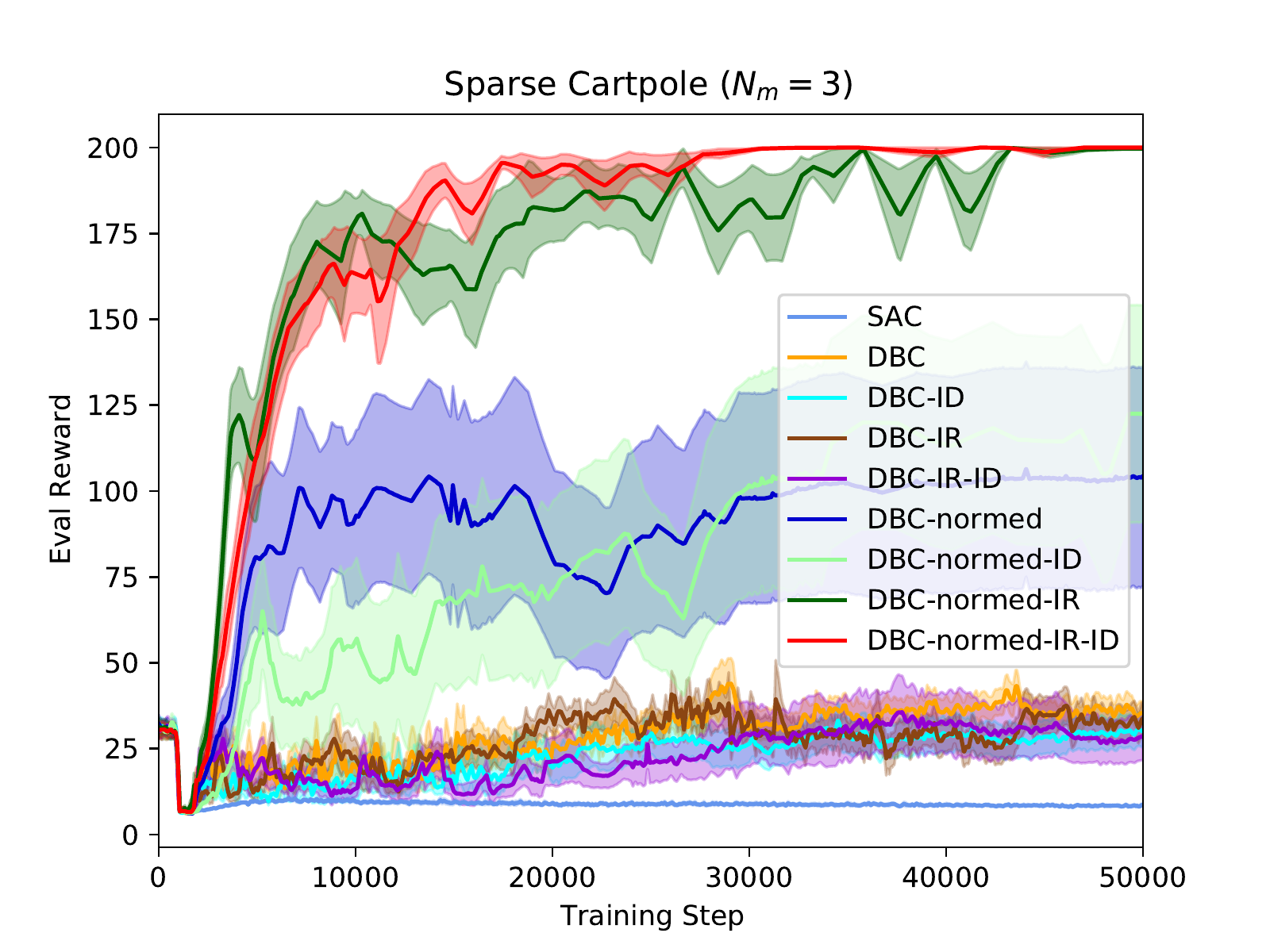}
    \hspace{\fill}
  \includegraphics[width=0.49\linewidth]{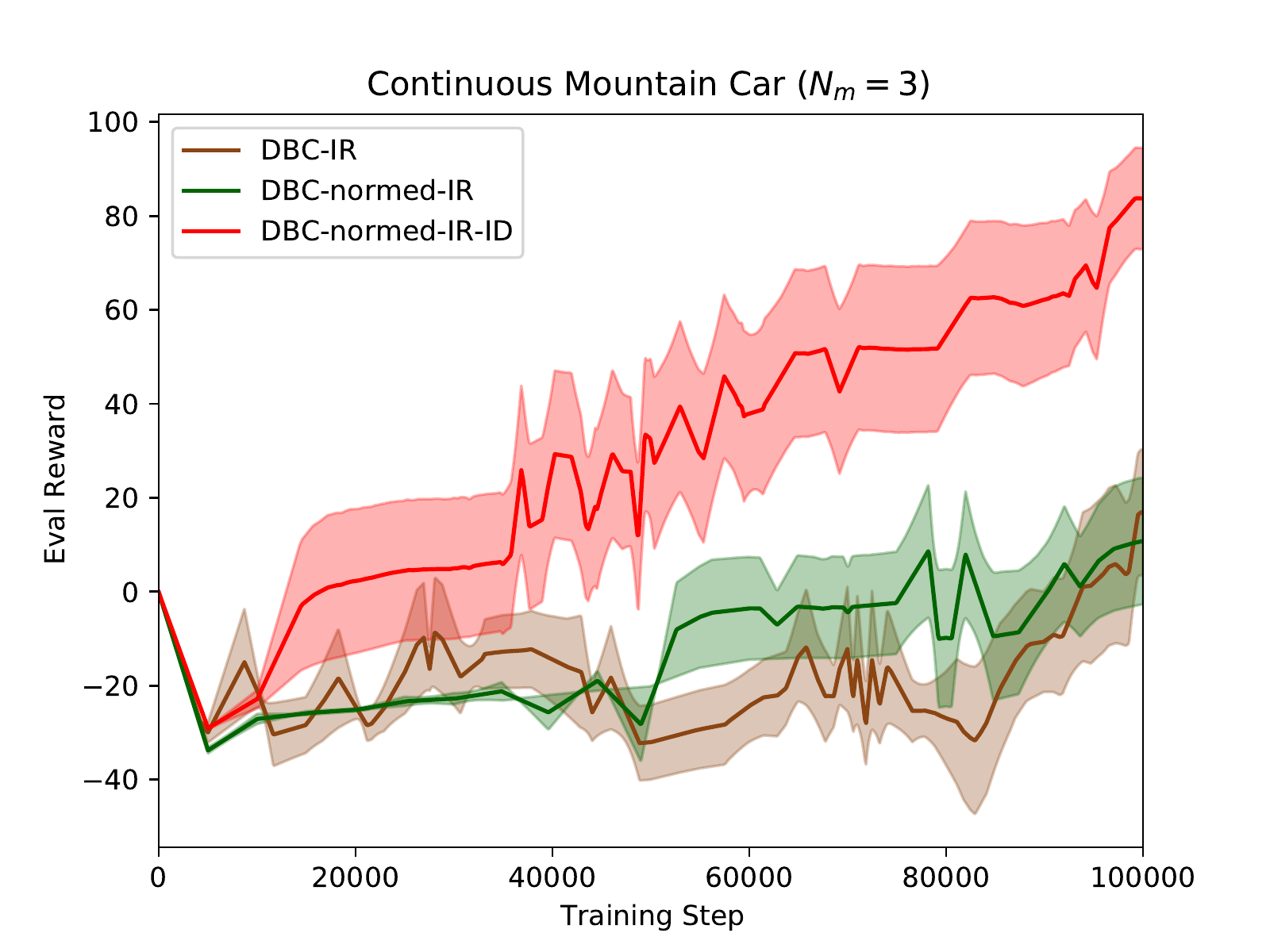} 
    \caption{Additional plots on Sparse Noisy Cartpole (left) and Mountain Car (right) at high distraction ($N_m=3$). Shaded bars are standard errors over 10 seeds.}
    \label{fig:gym-additional-plots}
\end{figure}

\begin{figure}
    \centering
     \includegraphics[width=0.49\textwidth]{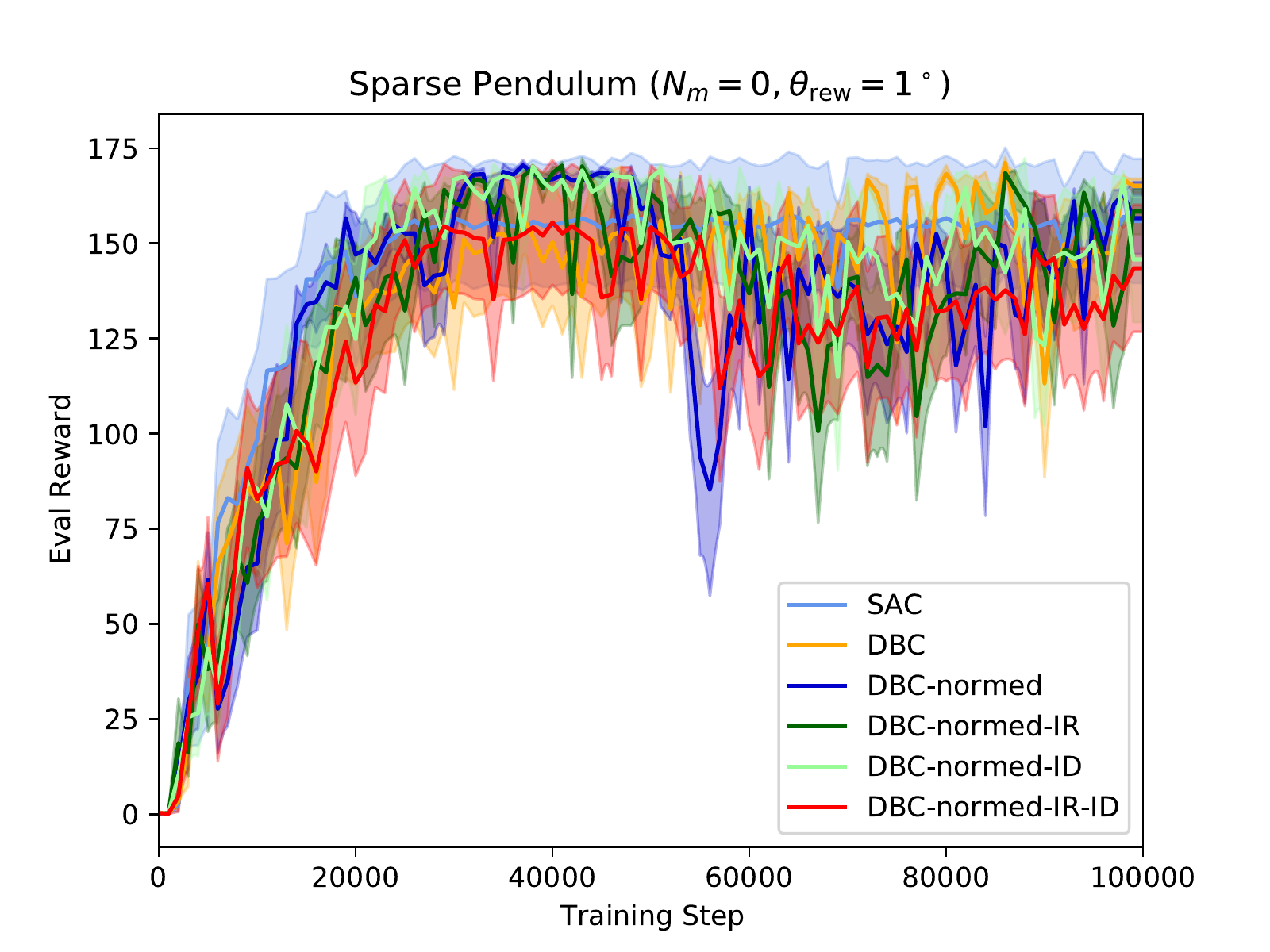}
    \hspace{\fill}
  \includegraphics[width=0.49\linewidth]{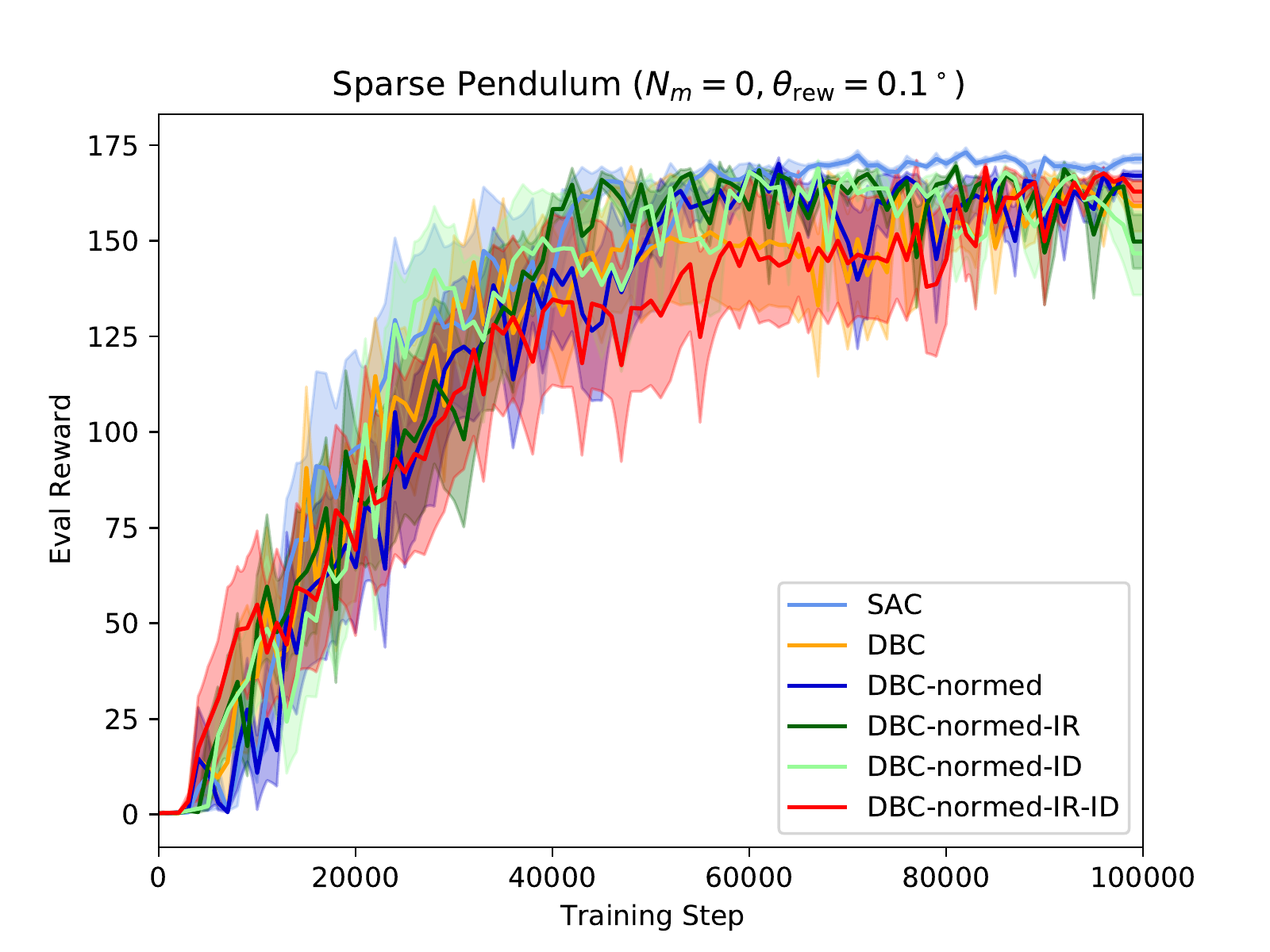} \\
  \includegraphics[width=0.49\linewidth]{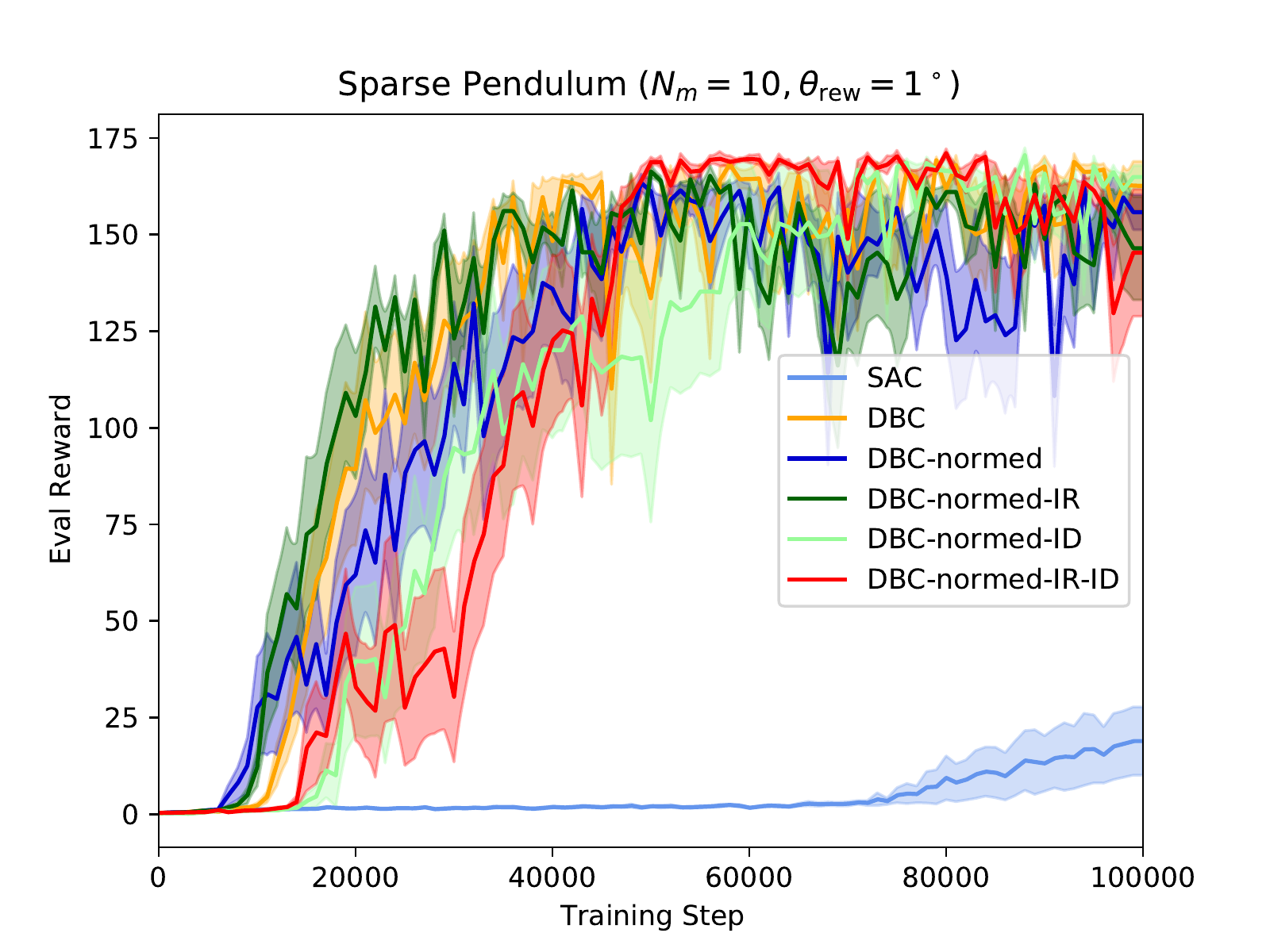}
    \hspace{\fill}
  \includegraphics[width=0.49\linewidth]{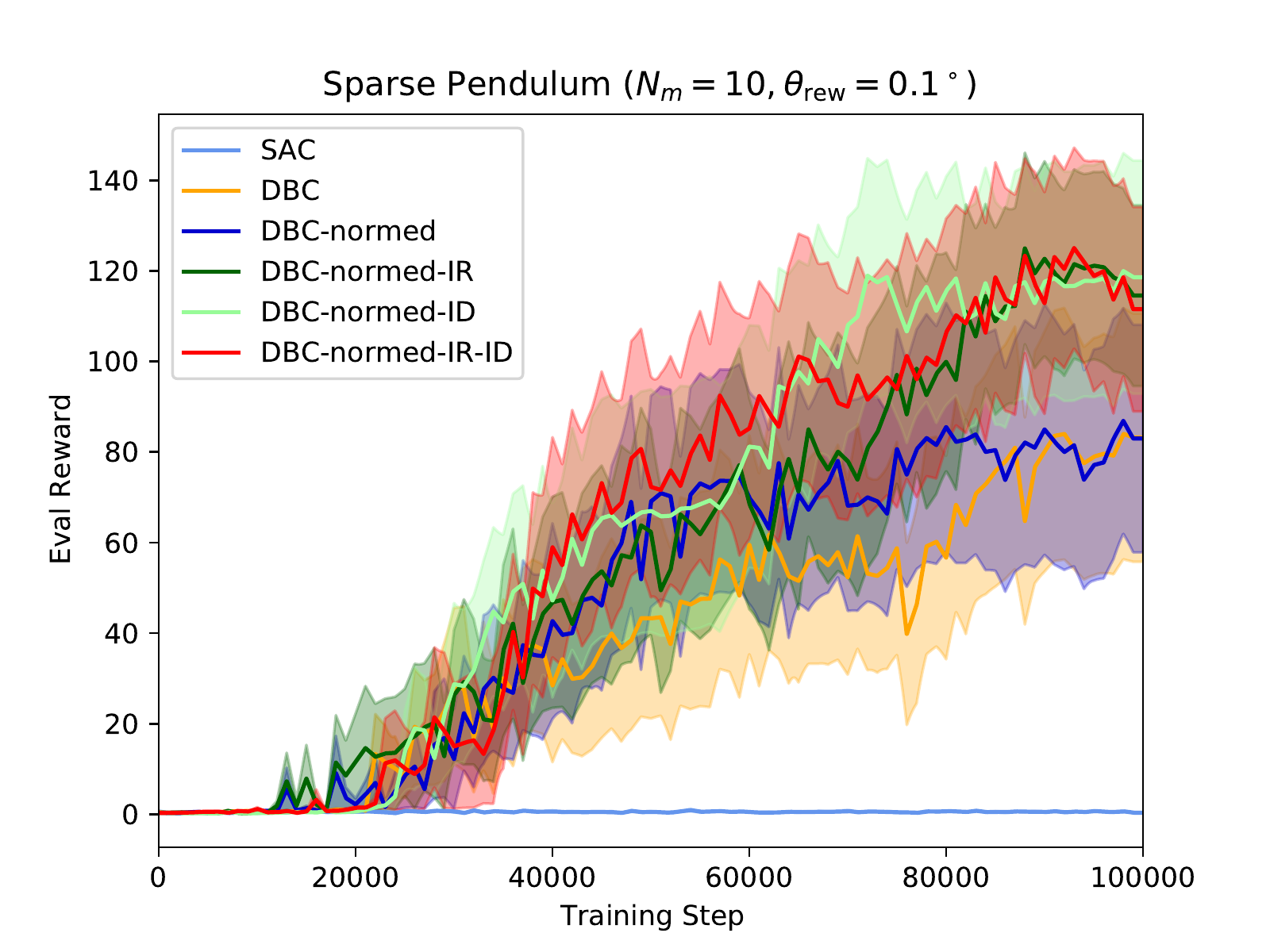}
    \caption{Results on the Sparse Pendulum task with differing levels of distraction and sparsity (shaded bars are standard errors over 10 seeds). }
    \label{fig:pendulum}
\end{figure}

\subsubsection{Hyper-parameters and Architectures for OpenAI Gym Tasks}
\label{sec:additional-results:hpa}

All models use the code generously released alongside the DBC paper \cite{zhang2021invariant} (CC-BY-NC 4.0 licensed), with default hyper-parameters and architectures according to the released code, unless otherwise specified.
DBC and our modifications of it are all built on top of the Soft Actor-Critic implementation therein.
Actor, critic, and encoder learning rates were all 0.001, using the Adam optimizer. 
A replay buffer of 50K was used, 
    with a batch size for all training steps of 512.
All encoders were implemented as multilayer perceptrons (MLPs) with four layers, except for the SparsePendulum task where two layers were used for all normalized approaches (unnormalized DBC still performed better with four). 
Encoder feature dimensionality was set to 
    $\mathrm{dim}(\phi(\mathbf{s})) = 50$. 
When inverse dynamics prediction was needed,
    we implemented $g_I$ as an MLP with two hidden layers (of size 256 and 128), with ELU activations.
Note that for the predictive approximate transition model $\widehat{\mathcal{P}}$, 
    we use a deterministic predictor.
Following the implementation of \cite{zhang2021invariant}, distances in reward and predicted encoded state, for the bisimulation metric, were computed with a Huber loss, defining the value of $q$ as a function of the distance.
However, embedding normalization was computed with the $L_2$ norm.
A discount of $\gamma = 0.99$ was always used.
In all tasks, rewards are bounded as $R \in [0, 1]$.

We remark that all evaluation rewards in plots (per seed) are computed as averages over 10 episodes.

We used a maximum intrinsic reward clamping value of $R_{\mathrm{max}, I} = 0.1$ for all Gym tasks,
and set $\eta_r$ and $\eta_d$ per task as in Table \ref{tab:hprms} below.
Hyper-parameters not left at default ($R_{\mathrm{max}, I}$, $\eta_r$, $\eta_d$, and number of encoder layers and latent dimensionality) were set by searching over a small, manually defined set of values.

\subsection{DeepMind Control Suite}
\label{sec:additional-results:dmc}
For the DeepMind Control Suite \cite{tassa2018deepmind}, our encoder model architecture is identical to the open-source code repository released by \cite{zhang2021invariant}. Namely, a 3x3 convolutional layer with stride 2 is followed by another 3x3 convolution with stride 1 (both with 32 channels), before a fully-connected layer with 50-dimensional output and layer normalization. \texttt{ReLU} activations are used between neural layers. When an inverse dynamics model is used, we use the same architecture as those used for the OpenAI experiments (described in Appendix \ref{sec:additional-results:hpa}), namely, a two-layer MLP. Differently from \cite{zhang2021invariant}, to speed up training, we run 16 environments in parallel, all of which add experience tuples to a shared replay buffer. After every 16 environment steps (i.e., each parallel step), we apply 2 gradient updates. Our hyperparameters for inverse dynamics and intrinsic motivation are given in the last column of Table \ref{tab:hprms}, and were selected by searching over a small, manually defined set of values.

\begin{table}[h]
    \begin{center}
        \begin{tabular}{c|cccc}
                ~    & Cartpole & Pendulum & Mountain Car & DMC \\\hline 
            $\eta_r$ & 2        & 0.1      & 20     & 1      \\
            $\eta_d$ & 1        & 0.1      & 20     & 10
        \end{tabular}
    \end{center}
    \caption{Hyper-parameters used per task for intrinsic reward and inverse dynamics.}
    \label{tab:hprms}
\end{table}

\subsection{Computational Resources and Timing}
\label{sec:additional-results:resources}

All training and evaluation was done on a small set of NVIDIA GPUs (GTX 1080 TI, Titan X, or RTX 2080 TI), less than 10 in total and shared with other users.

OpenAI Gym tasks were run with multiple seeds per GPU (up to the GPU memory limit) during training. 
In this parallel training context, 
which allowed us to complete a Gym task for all methods and seeds (for a single distraction level) within roughly a day on 2-4 GPUs,
we obtain the following approximate timings (in seconds per training iteration).
Cartpole: ${\sim}0.1$ for DBC-based methods and ${\sim}0.07$ for SAC.
Pendulum: 
${\sim}0.09$ for DBC-based methods 
(${\sim}0.1$ with IR+ID present) and
${\sim}0.06$ for SAC.
Mountain Car: ${\sim}0.07$ for DBC-based methods 
and ${\sim}0.04$ for SAC.

For the Deepmind Control (DMC) tasks, experiments were performed on 4 GPUs over the course of a week. Our 16-process parallelization for running MuJoCo \cite{todorov2012mujoco} simulations greatly sped up training, resulting in approximately 0.03, 0.05 and 0.07 seconds per environment step (with 2 gradient updates for every 16 environment step) for SAC, DBC and DBC+IR+ID respectively.

\end{document}